\documentclass{article}

\PassOptionsToPackage{numbers,sort&compress}{natbib}

\usepackage[preprint]{neurips_2026}

\usepackage[utf8]{inputenc} % allow utf-8 input
\usepackage[T1]{fontenc}    % use 8-bit T1 fonts
\usepackage{hyperref}       % hyperlinks
\hypersetup{colorlinks=true,linkcolor=blue,citecolor=blue,urlcolor=blue}
\usepackage{url}            % simple URL typesetting
\usepackage{booktabs}       % professional-quality tables
\usepackage{amsfonts}       % blackboard math symbols
\usepackage{amssymb}        % \blacktriangleright etc. (case-study macros)
\usepackage{nicefrac}       % compact symbols for 1/2, etc.
\usepackage{microtype}      % microtypography
\usepackage{xcolor}         % colors
\usepackage{amsmath}
\usepackage{amsthm}
\usepackage{algpseudocode}

\usepackage{booktabs}
\usepackage{multirow}
\usepackage{colortbl}
\usepackage{graphicx}
\usepackage{float}
\usepackage{enumitem}        % tighter list spacing (collaborator's whitespace-saving trick)

\usepackage{mathtools}
\usepackage{siunitx}
\usepackage{tikz}
\usetikzlibrary{positioning,arrows.meta,fit,backgrounds,decorations.pathreplacing,calligraphy,calc,matrix}

\usepackage{wrapfig}
\graphicspath{{figures/}{../latex/figures/}}

\usepackage{tcolorbox}
\tcbuselibrary{breakable,skins}
\usepackage{pifont}  % \ding{51} = check, \ding{55} = cross

\definecolor{caseWrongBd}{HTML}{DC2626}
\definecolor{caseRightBd}{HTML}{16A34A}
\definecolor{caseWrongBg}{HTML}{FEF2F2}
\definecolor{caseRightBg}{HTML}{F0FDF4}
\definecolor{caseBannerBg}{HTML}{FEF3C7}
\definecolor{caseBannerFg}{HTML}{92400E}
\definecolor{caseQBg}{HTML}{F8FAFC}
\definecolor{casePillBg}{HTML}{F1F5F9}
\definecolor{casePillBd}{HTML}{E2E8F0}
\definecolor{casePillFg}{HTML}{334155}
\definecolor{casePillMarkBg}{HTML}{FEF3C7}
\definecolor{casePillMarkBd}{HTML}{F59E0B}
\definecolor{casePillMarkFg}{HTML}{92400E}
\definecolor{caseDim}{HTML}{9CA3AF}

\newcommand{\codepill}[1]{%
  \tikz[baseline=(c.base)]\node[draw=casePillBd, fill=casePillBg,
    text=casePillFg, rounded corners=1pt, inner xsep=2pt, inner ysep=0.5pt,
    line width=0.3pt, font=\ttfamily\fontsize{6}{7}\selectfont,
    minimum width=8pt, minimum height=8pt, anchor=base] (c) {#1};%
  \hspace{1pt plus 1pt}\allowbreak}
\newcommand{\codepillmark}[1]{%
  \tikz[baseline=(c.base)]\node[draw=casePillMarkBd, fill=casePillMarkBg,
    text=casePillMarkFg, rounded corners=1pt, inner xsep=2pt, inner ysep=0.5pt,
    line width=0.5pt, font=\ttfamily\bfseries\fontsize{6}{7}\selectfont,
    minimum width=8pt, minimum height=8pt, anchor=base] (c) {#1};%
  \hspace{1pt plus 1pt}\allowbreak}

\newenvironment{codestrip}{%
  \par\nopagebreak\vspace{2pt}%
  \noindent\sloppy\hangindent=42pt\hangafter=1\relax
  \leavevmode\ignorespaces}{\par\vspace{1pt}}

\usepackage{pgffor}
\newcommand{\pillrow}[1]{\foreach \pn in {#1}{\codepill{\pn}}}
\newcommand{\pillrowmark}[2]{%
  \foreach \pn in {#2}{\ifnum\pn=#1\codepillmark{\pn}\else\codepill{\pn}\fi}}

\newcommand{\casediv}[6]{%
  \par\noindent\textsf{\footnotesize diverges @
    \textbf{#1\%} of NL \;\textbullet\; \textbf{#2\%} of code seq.\;%
    \textcolor{caseDim}{(NL: #3/#4 chars \textbullet\ code: #5/#6 chunks)}}\par}

\newcommand{\caseheader}[4]{%
  \begin{casebanner}\textsc{sample #1}\quad
    \textnormal{\texttt{#2}}\hfill
    \textnormal{\textcolor{caseWrongBd}{\textbf{#3}}}
    $\rightarrow$
    \textnormal{\textcolor{caseRightBd}{\textbf{#4}}}\end{casebanner}%
  \nopagebreak\par\nopagebreak}

\newcommand{\panehead}[4]{%
  {\sffamily\fontsize{7}{8}\selectfont\bfseries\color{#1}#2 \hfill
    \colorbox{#1}{\textcolor{white}{\,\texttt{#3}\,}}\ #4}\par\medskip}

\newcommand{\twostripmark}[3]{%
  \begin{codestrip}%
    {\sffamily\fontsize{6}{7}\selectfont\color{caseDim}\textsc{prefill}}\hspace{4pt}%
    \pillrowmark{#1}{#2}%
    \hspace{6pt}{\sffamily\fontsize{6}{7}\selectfont\color{caseDim}%
      $\blacktriangleright$\ \textsc{decode}}\hspace{4pt}%
    \pillrowmark{#1}{#3}%
  \end{codestrip}}

\newcommand{\twostripplain}[2]{%
  \begin{codestrip}%
    {\sffamily\fontsize{6}{7}\selectfont\color{caseDim}\textsc{prefill}}\hspace{4pt}%
    \pillrow{#1}%
    \hspace{6pt}{\sffamily\fontsize{6}{7}\selectfont\color{caseDim}%
      $\blacktriangleright$\ \textsc{decode}}\hspace{4pt}%
    \pillrow{#2}%
  \end{codestrip}}

\newtcolorbox{caseboxwrong}[1][]{enhanced, breakable,
  colback=caseWrongBg, colframe=caseWrongBd, boxrule=0.5pt, arc=2pt,
  left=4pt, right=4pt, top=3pt, bottom=3pt,
  fontupper=\ttfamily\footnotesize, #1}
\newtcolorbox{caseboxright}[1][]{enhanced, breakable,
  colback=caseRightBg, colframe=caseRightBd, boxrule=0.5pt, arc=2pt,
  left=4pt, right=4pt, top=3pt, bottom=3pt,
  fontupper=\ttfamily\footnotesize, #1}

\newtcolorbox{casebanner}{enhanced, sharp corners=south,
  colback=caseBannerBg, colframe=caseBannerBg, boxrule=0pt,
  left=6pt, right=6pt, top=2pt, bottom=2pt,
  fontupper=\sffamily\bfseries\footnotesize\color{caseBannerFg}}

\newtcolorbox{caseqbox}{enhanced, breakable,
  colback=caseQBg, colframe=caseQBg, boxrule=0pt,
  borderline west={2pt}{0pt}{caseDim},
  left=8pt, right=4pt, top=3pt, bottom=3pt,
  fontupper=\small}

\newcommand{\casepre}[1]{{\color{caseDim}#1}}
\newcommand{\casemark}{{\color{orange!80!black}\,$\blacktriangleright$\,}}

\definecolor{qwenLight}{HTML}{E8F1FA}
\definecolor{qwenMid}{HTML}{D2E3F2}
\definecolor{qwenDark}{HTML}{B9D3E8}
\definecolor{llamaLight}{HTML}{FDE3C8}
\definecolor{llamaDark}{HTML}{F5C8A3}

\definecolor{dlrgold}{HTML}{FFF6D8}
\definecolor{famband}{HTML}{F7F7F4}

\definecolor{red1}{HTML}{FDF0EC}
\definecolor{red2}{HTML}{FAD8CF}
\definecolor{red3}{HTML}{F4B9AA}
\definecolor{red4}{HTML}{EC9785}
\definecolor{red5}{HTML}{DE6B58}
\definecolor{redink}{HTML}{8A2615}

\newcommand{\dcell}[2]{%
  \ifnum#1=5\relax
    \cellcolor{red5}\textcolor{white}{\textbf{#2}}%
  \else
    \cellcolor{red\the\numexpr#1\relax}\textcolor{redink}{\textbf{#2}}%
  \fi
}

\newcommand{\cihalf}[1]{{\color{gray!65}\scriptsize$_{\pm#1}$}}
\newcommand{\val}[2]{$#1$\cihalf{#2}}
\newcommand{\bval}[2]{$\mathbf{#1}$\cihalf{#2}}

\providecommand{\sft}{\textsc{SFT}}

\newtheorem{theorem}{Theorem}
\newtheorem*{theorem*}{Theorem}
\newenvironment{restatedthm}[2]%
  {\medskip\par\noindent\textbf{Theorem~\ref{#1}} (#2)\textbf{.}\itshape\quad}%
  {\par\upshape\medskip}
\newtheorem{lemma}{Lemma}

\newtheorem{definition}{Definition}

\newtheorem{assumption}{Assumption}

\usepackage[ruled,vlined]{algorithm2e}

\title{Dynamic Latent Routing}

\author{%
  Fangyuan Yu\thanks{Corresponding author: \texttt{fangyuan.yu18@gmail.com}.}
  \hspace{3em} Xin Su \hspace{3em} Amir Abdullah \\[4pt]
  Thoughtworks AI Labs (TAILS)
}

\newcommand{\dlt}[1]{%
  \ifdim#1pt<0pt {\tiny\,\textcolor{red}{$(#1)$}}%
  \else\ifdim#1pt>0pt {\tiny\,\textcolor{green!55!black}{$(+#1)$}}%
  \else {\tiny\,$(\pm0.0)$}%
  \fi\fi}

\begin{document}

% Allow pages to end short rather than vertically stretching content;
% prevents the large mid-page gaps seen in the appendix when tables/floats
% don't fully fill a page.
\raggedbottom

\maketitle
% FY: the "how policy are composed ..." is not accurate, as 
% compositional RL does policy composition, our contribution is 
% temporal concatenation", which justifies temporal abstraction in HRL. 

\begin{abstract}
We investigate the temporal concatenation of sub-policies in Markov Decision Processes (MDP) with time-varying reward functions. We introduce General Dijkstra Search (GDS), and prove that globally optimal goal-reaching policies can be recovered through temporal composition of intermediate optimal sub-policies.

Motivated by the ``search, select, update'' principle underlying GDS, we propose Dynamic Latent Routing (DLR), a language-model post-training method that jointly learns discrete latent codes, routing policies, and model parameters through dynamic search in a single training stage. In low-data fine-tuning settings, DLR matches or outperforms supervised fine-tuning across four datasets and six models, achieving a mean gain of $+6.6$ percentage points, while prior discrete-latent baselines consistently underperform SFT. Mechanistic analyses and targeted code ablations show that DLR learns structured routing behaviors with distinct causal roles. 
%In six-digit arithmetic, a canonical mechanistic interpretability benchmark, DLR externalizes into discrete routing tokens algorithmic subtasks previously studied through activation and circuit analysis.
\end{abstract}
\section{Introduction}

The sensory world we live in is continuous; the language we describe it with is discrete~\citep{harnad1990symbol}. Neural networks, increasingly capable across a wide range of tasks~\citep{radford2019gpt2,openai2025gpt55,anthropic2025claude47,deepseek2025v4,qwen2025qwen35,moonshot2025kimi26}, operate almost entirely in continuous representations even when their inputs are discrete. This raises a question: are discrete codes beneficial? Consider what a discrete code makes possible. The sentence ``stand up and walk away'' composes two policies along the temporal dimension: ``stand up'' and ``walk away'', each encoding a chain of low-level motor actions. Whilst hierarchical reinforcement learning~\citep{sutton1999options,dayan1993feudal,vezhnevets2017feudal,dietterich2000maxq,machado2023temporal} extends the agent's action space to include such sub-policies, a theoretical understanding of \emph{why} this helps is still lacking.

We fill this gap by making the reward in the MDP explicitly time-varying. In this setting, no static policy is optimal in general (Thm.~\ref{thm:static-gap}). We propose the General Dijkstra Search (GDS) algorithm, which searches over concatenations of sub-policies, and prove that it finds an optimal goal-reaching policy (Thm.~\ref{thm:gds}). GDS departs from the Bellman iteration paradigm that underpins most of RL: rather than refining a policy through per-state updates, it shows that searching over \emph{concatenations} of sub-policies has a guarantee of reaching the optimal goal-reaching policy.

Motivated by GDS, we study how an LLM can create and concatenate its own discrete latent codes, each encoding a sub-policy. Given a natural language sequence $x$, an LLM maximizes its predictive likelihood $p_{\theta}(x)$. Injecting latent codes $a$ shifts the objective to maximizing the conditional likelihood $p_{\theta}(x \mid a)$. Unlike $x$, however, we have no supervision target for $a$; the codes become an action space the model must explore.

Prior work injecting discrete codes during post-training~\citep{goyal2023pause,pfau2024filler,su2025tokenassorted,zelikman2024metatokens,ramji2026abstractcot} shares two drawbacks. First, they inject codes as extra tokens into the natural language sequence, disrupting the structure seen during pre-training. This necessitates large training budgets to match baseline SFT and consistently lags behind in low-data regimes. We find that injecting codes directly as steering vectors into the residual stream avoids this disruption, matching or outperforming SFT under low-data constraints.

Second, these methods require multiple training stages: codes are fixed in advance~\citep{goyal2023pause,pfau2024filler,zelikman2024metatokens}, pre-labeled by a separate model~\citep{su2025tokenassorted}, or warmed up in a separate phase before search can begin~\citep{ramji2026abstractcot}. What GDS prescribes - jointly learning which codes to use, what they do, and how to compose them - is not achieved in a single stage by any existing method.

We address both issues with Dynamic Latent Routing (DLR). DLR can be viewed as a neural relaxation of GDS that preserves its core structure: \textbf{search} over candidate code sequences via guided rollouts from a policy head, \textbf{select} the sequence that maximizes the conditional likelihood $p_{\theta}(x \mid a)$, and \textbf{update} the policy head, codebook, and base model jointly from a single objective. The explicit priority queue is replaced by a learned policy head, but the ``search, select, update'' loop remains intact. By steering the residual stream rather than injecting tokens, DLR avoids disrupting the pre-trained sequence structure - matching or outperforming SFT in the low-data regime - while unifying code search and learning into a single stage. The learned codes are diverse, input-dependent, and carry causal effects on downstream behavior. Because routing decisions are explicit discrete symbols rather than implicit activation patterns, DLR also exposes a directly observable interface to intermediate computation, enabling causal interventions and mechanistic analysis without post-hoc activation probing.
In summary, our main contributions are:
\begin{itemize}
    \item We introduce General Dijkstra Search (GDS), and prove that under MDPs with time-varying rewards, optimal goal-reaching policies can be recovered through temporal composition of sub-policies.
    
    \item We propose Dynamic Latent Routing (DLR), a single-stage post-training method where a learned policy head dynamically searches for discrete latent codes that steer the model through the residual stream.
    
    \item DLR outperforms all baselines across 24 model-dataset settings spanning four QA benchmarks and six model sizes, with the largest gains on reasoning-heavy tasks, reaching $+10.2$ on GSM8K and $+18.8$ on ScienceQA.
    
    \item The codes learned by DLR specialize to distinct subtasks and are causally load-bearing, emerging without any subtask labels provided during training. On ScienceQA, n-grams become increasingly topic-specialized as length increases; in six-digit arithmetic, individual codes specialize to distinct algorithmic subtasks. In both settings, ablating specific codes produces subtask-specific effects on model behavior.
    
\end{itemize}

% DLR outperforms all baselines in test accuracy across all 20 model-dataset settings, covering four QA benchmarks and five model sizes. The largest absolute gains are observed on reasoning-heavy tasks, reaching 10.2 points on GSM8K and 12.9 points on ScienceQA. We also show that the learned codes expose an interpretable routing structure: they are diverse, input-dependent, and causally involved in downstream predictions.

\section{Related Work}

\paragraph{Discrete latents in LLMs.} Recent work introduces discrete codes into LLMs as extra tokens. Filler Tokens~\citep{pfau2024filler} show that adding a single, repeated token can improve performance when training transformers on synthetic tasks. Pause Tokens~\citep{goyal2023pause} and Meta-Tokens~\citep{zelikman2024metatokens} extend this to post-training, but require extensive warm-up to match or exceed SFT. TokenAssorted~\citep{su2025tokenassorted} pre-labels fine-tuning sequences with diverse codes learned via a separate vector-quantized variational encoder before post-training the LLM on the interleaved sequence. Abstract-CoT~\citep{ramji2026abstractcot} uses a warm-up phase that iterates between two sub-phases to teach the model to use new tokens, then applies group relative preference optimization to refine code selection. Compared to DLR, none of these methods unifies code search, policy update, and LM update into a single training stage, and they all require substantial compute in a warm-up phase, limiting their application in the low-data fine-tuning regime; TokenAssorted and Abstract-CoT are additionally designed to compress chain-of-thought sequences, a direction not aligned with ours.

\paragraph{Representation engineering and steering.} Representation engineering~\citep{zou2023representation} manipulates hidden states to control model behavior, with applications in safety~\citep{yousefpour2025repbend,siu2025steeringsafety}, reasoning~\citep{tang2025glore,seal2025}, and truthfulness~\citep{whyrepeworks2025}. Existing methods either apply a single fixed steering vector; DLR learns dynamic steering codes jointly with the model through search.

\paragraph{Hierarchical and compositional reinforcement learning.} Bellman iteration~\citep{sutton2018reinforcement} has been the cornerstone of optimal policy search in reinforcement learning, refining a policy through per-state updates. Hierarchical RL introduces high-level actions through frameworks such as options~\citep{sutton1999options,bacon2017option}, feudal networks~\citep{dayan1993feudal,vezhnevets2017feudal}, and MAXQ~\citep{dietterich2000maxq}, but lacks a theoretical justification for why composing sub-policies along the temporal axis helps. Compositional and goal-conditioned RL~\citep{schaul2015universal,andrychowicz2017hindsight,barreto2017successor,pmlr-v97-van-niekerk19a,cao2020zero,todorov2009compositionality,hunt2019composing} parameterizes goals into value functions or defines algebraic operations over sub-policies, enabling multi-task learning and zero-shot transfer; however, their composition is not along the temporal dimension and assumes stationary rewards. In contrast, we explicitly define temporal policy concatenation in a dynamic MDP without introducing high-level actions, and prove that GDS discovers optimal goal-reaching policies by concatenating sub-policies optimal for intermediate goals. This provides the missing theoretical justification for composing sub-policies along the temporal axis.

\paragraph{Mechanistic interpretability and probing in LLMs.} Most existing LLM interpretability work recovers abstract structure from a model's internals after training. The mechanistic line identifies specific computational circuits \citep{olsson_2022_induction, wang_2023_ioi, nanda_2023_grokking, quirke_2024_addition, quirke_2024_addsub_preprint, zhang_2024_arithmetic}; probing reads semantic content out of hidden states \citep{li_2023_worldrepresentations, nanda_2023_linearworldmodels, belinkov_2022_probes, sun_2025_arithmeticerrors}; sparse autoencoder approaches train an extra decoder on top of the frozen model to extract a large set of sparse features and then run a large automated pipeline to label them \citep{huben_2024_sae, paulo_2025_autointerp}. All of these methods assume the abstraction lies hidden inside the continuous representation, waiting to be recovered post hoc. DLR inverts this: the abstraction is baked in during training as $C$ abstraction codes produced by hard-argmax, each a directly observable, input-dependent routing decision; the same probing and auto-interpretation pipelines then apply to these codes at orders-of-magnitude lower cost.

% \section{Theory}
% \section{Theory}
\section{Theoretical Foundation}
\label{sec:theoretical-foundation}

A Markov Decision Process is defined by the tuple $(\mathcal{S}, \mathcal{A}, P, r, \gamma)$ with state space $\mathcal{S}$, action space $\mathcal{A}$, transition kernel $P(s; | s, a)$ and reward $r(s, a)$. Since the reward is not explicitly time-dependent, optimal policies are also time-invariant. To define policy concatenation along the temporal axis, we make the reward function explicitly time-dependent, yielding a dynamic MDP. We formalize policy concatenation, show how the value function decomposes under concatenation (Thm.~\ref{thm:concat-value}), and prove that optimal policies in this setting are in general time-dependent (Thm.~\ref{thm:static-gap}). Moreover, we show they can be obtained by dynamic Bellman iteration under mild assumptions (Thm.~\ref{thm:policy-iteration}), setting the stage for the General Dijkstra Search algorithm. Full theory and proof are provided in Appendix~\ref{app:dmdp-theory}.

\begin{definition}[Dynamic Markov Decision Process (DMDP)]
A \emph{dynamic MDP} is a finite-horizon MDP $(\mathcal{S}, \mathcal{A}, P, \{r_t\}_{t=0}^{T-1}, \gamma)$ with state space $\mathcal{S}$, action space $\mathcal{A}$, time-homogeneous transition kernel $P(s'\mid s,a)$, time-varying reward $r_t:\mathcal{S}\times\mathcal{A}\to\mathbb{R}_{\le 0}$, discount $\gamma\in[0,1)$, and horizon $T\in\mathbb{N}$. WLOG we assume $r_t\le 0$: any $r_{\max}$-bounded reward shifts to $r_t-r_{\max}\le 0$ without changing the argmax, so reward can be read as \emph{negated cost}.
\end{definition}

\begin{definition}[Time-varying policy]
A \emph{time-varying} (or \emph{dynamic}) policy is a sequence 
$\pi = \{\pi_t\}_{t=0}^{T-1}$, where each
\[
\pi_t : \mathcal{S} \to \mathcal{A}
\]
is a (possibly deterministic) decision rule at time $t$.
Equivalently, we can write $\pi : \mathcal{S} \times \mathbb{N} \to \mathcal{A}$ and interpret
$\pi_t(\cdot) := \pi(\cdot, t)$.
We denote by
\[
\Pi := \bigl\{ \pi \,\big|\, \pi : \mathcal{S} \times \mathbb{N} \to \mathcal{A} \bigr\}
\]
the collection of all time-varying policies.
\end{definition}

\begin{definition}[Concatenated policy]\label{def:concat-policy}
Let $\pi^{1} \in \Pi_{T_{1}}$ and $\pi^{2} \in \Pi_{T_{2}}$ be time-varying
policies of horizons $T_{1}$ and $T_{2}$, respectively. Their concatenation
$\pi = \pi^{1}_{0:T_{1}} \circ \pi^{2}_{0:T_{2}} \in \Pi_{T_{1}+T_{2}}$ is the
policy defined by
\[
\pi_{t}(s)
=
\begin{cases}
\pi^{1}_{t}(s), & t = 0,\dots,T_{1}-1,\\
\pi^{2}_{t-T_{1}}(s), & t = T_{1},\dots,T_{1}+T_{2}-1.
\end{cases}
\]
\end{definition}

\begin{restatedthm}{thm:concat-value}{Value of a concatenated policy}
Under Assumptions~\hyperref[asm:A1]{A1}--\hyperref[asm:A3]{A3},
let $\pi = \pi^{1}_{0:T_{1}} \circ \pi^{2}_{0:T_{2}}$ be as in
Definition~\ref{def:concat-policy} and let $T := T_{1} + T_{2}$. Then for all
$t = 0,\dots,T-1$ and $s \in \mathcal{S}$,
\[
V_{t}^{\pi}(s)
=
V^{\pi^{1}}_{t}(s)
+ \gamma^{\max(T_{1}-t,0)}
  \,\mathbb{E}\!\Big[
      V^{\pi^{2}}_{\max(0,\,t-T_{1})}\bigl(s_{\max(t,T_{1})}\bigr)
      \,\Big|\, s_{t} = s
    \Big].
\]
The detailed proof (Appendix~\ref{thm:concat-value}) proceeds by splitting the discounted return at time $T_{1}$. 
\end{restatedthm}

\subsection{Policy Composition and Goal-Oriented Optimality}
\label{sec:gds}

We formally define optimal goal reaching policies to be the policy with the largest value that reach a set of goal states. We then propose General Dijkstra Search, along with Thm.~\ref{thm:gds-reach} proving it discovers optimal goal reaching policies. Full theory and proof are provided in Appendices~\ref{app:policy-composition}--\ref{app:gds-coverage}.

\begin{definition}[Goal states]
For a policy $\pi \in \Pi_{t}$ starting from $s \in \mathcal{S}$, the goal
states $\mathcal{G}^{\pi}(s)$ are defined by
\[
\mathcal{G}^{\pi}(s) := \Big\{ g \in \mathcal{S} \,\big|\, p_{\pi}(s_{t}=g \mid s_{0}=s) > 0 \Big\}.
\]
\end{definition}

\begin{definition}[Reaching policies]
Given a start state $s \in \mathcal{S}$ and a set of goal states
$\mathcal{G} \subset \mathcal{S}$, the set of \emph{reaching policies}
$\Pi^{\mathcal{G}\supset}_{1:T}(s)$ is
\[
\Pi^{\mathcal{G}\supset}_{1:T}(s)
:= \Bigl\{ \pi \in \bigcup_{t=1}^{T} \Pi_{t} \,\Big|\, \mathcal{G}^{\pi}(s) \subset \mathcal{G} \Bigr\}.
\]
\end{definition}

\begin{definition}[Optimal goal-reaching policy]
Given a start state $s \in \mathcal{S}$ and a goal set
$\mathcal{G} \subset \mathcal{S}$, an \emph{optimal goal-reaching policy} is a
policy $\pi^{*} \in \Pi^{\mathcal{G}\supset}_{1:T}(s)$ such that, for all
$\pi \in \Pi^{\mathcal{G}\supset}_{1:T}(s)$,
\[
V_{0}^{\pi^{*}}(s) \geq V_{0}^{\pi}(s).
\]
\end{definition}

\begin{algorithm}[H]
\caption{General Dijkstra Search (Optimal Reach)}
\KwIn{Error tolerance $\epsilon_{t}:= \frac{r_{\max}}{1-\gamma} \cdot \sum_{i=t}^{\infty} \gamma^{i}$, start state $s \in \mathcal{S}$, goal states $\mathcal{G}^{*} \subset \mathcal{S}$}
Initialize $\mathcal{Q} = \{(\emptyset, 0, \{s\}, 0)\}$, $v = \emptyset$, $\mathcal{R} = \emptyset$\;
\While{$\mathcal{Q} \neq \emptyset$}{
Pop $(\pi_{1:t}, v_{t}, \mathcal{G}^{\pi_{1:t}}(s), t)$ from the priority queue with maximal value; add $\mathcal{G}^{\pi_{1:t}}(s)$ into $\mathcal{R}$ if it is not already inside\;
\If{$\mathcal{G}^{\pi_{1:t}}(s) \subset \mathcal{G}^{*}$}{
break\;
}
\If{$(\exists (s, \mathcal{G}) \in v \text{ s.t. } \mathcal{G} \subset \mathcal{G}^{\pi_{1:t}}(s) \text{ and } v_{t} \leq v(s,\mathcal{G}) - \epsilon_{t})$ or $t = T$}{
continue\;
}
\For{$\pi \in \Pi_{1}$}{
Concatenate policy $\pi_{1:t+1} = \pi_{1:t} \circ \pi$ and compute $v_{t+1} = v_{t} + \gamma^{t} \cdot \mathbb{E}_{s_{0:t}}\big[ V_{0}^{\pi}(s_{t}) \,\big|\, s_{0} = s\big]$\;
Push $(\pi_{1:t+1}, v_{t+1}, \mathcal{G}^{\pi_{1:t+1}}(s), t+1)$ into $\mathcal{Q}$\;
\ForEach{$(s, \mathcal{G}) \in v$ such that $\mathcal{G}^{\pi_{1:t+1}}(s) \subset \mathcal{G}$ and $\mathcal{G} \notin \mathcal{R}$}{
$v(s, \mathcal{G}) \gets \max(v(s, \mathcal{G}), v_{t+1})$\;
}
\If{$(s, \mathcal{G}^{\pi_{1:t}}) \notin v$}{
$v(s, \mathcal{G}^{\pi_{1:t}}) \gets v_{t+1}$\;
}
}
}
\end{algorithm}

\begin{restatedthm}{thm:gds-reach}{GDS optimality for optimal reach}
Under Assumptions~\hyperref[asm:A1]{A1}--\hyperref[asm:A3]{A3}, for any reachable goal set, General Dijkstra Search for Optimal Reach finds an optimal goal-reaching policy. Specifically, for every $\mathcal{G} \in \mathcal{G}_{T}^{\supset}(s)$, there exists $\pi^{*} \in \Pi_{1:T}^{\mathcal{G}\supset}(s)$ such that $V_{0}^{\pi^{*}}(s) \ge V_{0}^{\pi}(s)$ for every $\pi \in \Pi_{1:T}^{\mathcal{G}\supset}(s)$.
\end{restatedthm}

\noindent The proof (Appendix~\ref{thm:gds-reach}) follows from the queue invariant, a pruning guarantee, and the optimality of the first popped feasible policy.

\section{Methodology}
\label{sec:method}

\begin{figure}[t]
  \centering
  \includegraphics[width=\linewidth]{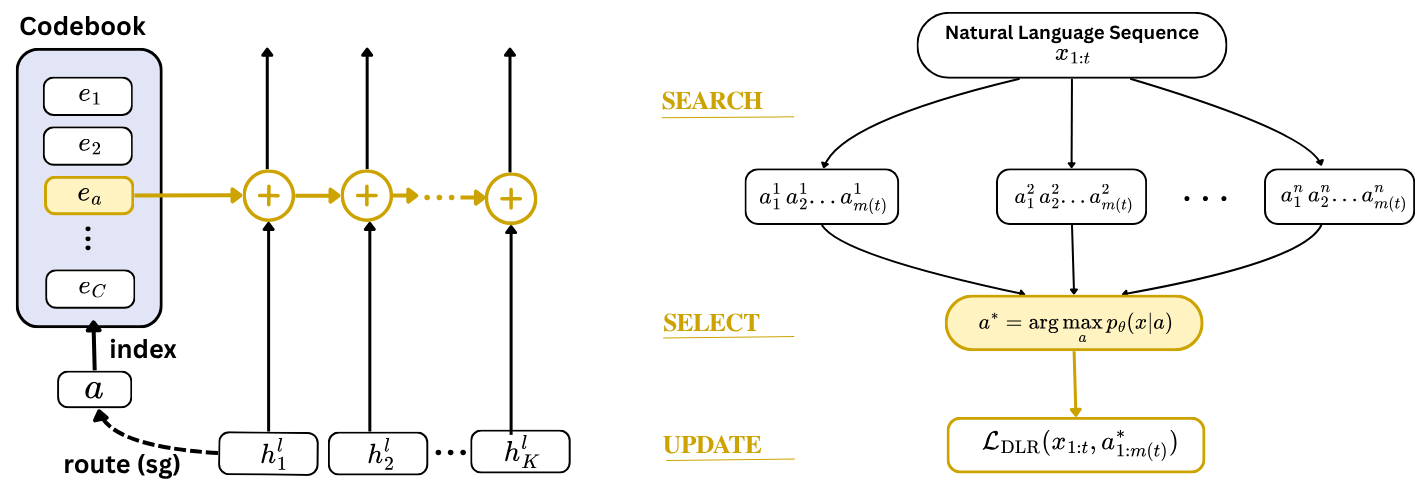}
  \caption{DLR method overview. \emph{Left:} chunk-level steering: each chunk $m$ is routed to a discrete code $a_m$, which indexes a steering vector $\alpha\, e_{a_m}$ from the codebook and is added in place to every hidden state $h_1,\dots,h_K$ in the chunk. \emph{Right:} per-step search: $N$ code sequences are sampled from the routing head (\textsc{search}), the one maximizing $p_\theta(x\mid a)$ is selected (\textsc{select}), and $\mathcal{L}_{\mathrm{DLR}}$ is used to jointly update the codes, routing head, and LM (\textsc{update}).}
  \label{fig:dlr-method-overview}
\end{figure}

GDS proves that optimal policies emerge from concatenating sub-policies rather than per-state refinement, using a priority queue in a search--select--update loop. We apply this principle to language models by treating discrete latent codes as sub-policies that are concatenated to steer the residual stream. DLR replaces the priority queue with a learned policy head, unifying routing, codebook, and model optimization in a single search--select--update phase. This serves as a neural relaxation of GDS that preserves its core structure. We illustrate DLR in Figure~\ref{fig:dlr-method-overview} and describe the latent codes and routing policy in \S\ref{sec:method:codes}, the objective in \S\ref{sec:method:loss}, and the search--select--update procedure in \S\ref{sec:method:search}.

\subsection{Discrete codes and chunk-level steering}
\label{sec:method:codes}

Let $\pi_{\theta}$ be a causal language model with hidden states
$h_t^{(l)}\in\mathbb{R}^{H}$ at block $l$ and token position $t$, and
let $l^{*}$ denote a chosen \emph{injection layer}. Given a
natural-language token sequence $x_{1:T}$, DLR maintains an
abstract-code sequence $a_{0:M-1}$ on a coarser temporal dimension:
each code steers a chunk of $K$ consecutive NL tokens at layer
$l^{*}$, where $K\in\mathbb{N}$ is the \emph{abstraction ratio}. The
chunk index of token~$t$ is
\begin{equation}\label{eq:chunkmap}
m(t) \;=\; \bigl\lfloor (t-1)/K \bigr\rfloor,
\qquad t\in\{1,\dots,T\},
\end{equation}
giving $M=\lceil T/K\rceil$ codes in total. Each chunk is first
\emph{routed} to a code and then \emph{steered} by the continuous
vector indexed by it (Fig.~\ref{fig:dlr-method-overview}, left). 
A linear head $W_{\mathrm{rt}}\in\mathbb{R}^{C\times H}$, with $C$
the codebook size, is applied at the \emph{first} token of each chunk
to produce routing logits:
\begin{equation}\label{eq:code}
z_m \;=\; W_{\mathrm{rt}}\,\mathrm{sg}\!\bigl(\,h_{mK+1}^{(l^{*})}\,\bigr)
\;\in\;\mathbb{R}^{C}.
\end{equation}
The role of the stop-gradient $\mathrm{sg}(\cdot)$ is discussed
alongside the loss in \S\ref{sec:method:loss}.
Each codebook entry $k\in\{0,\dots,C{-}1\}$ owns a learned vector
$e_k\in\mathbb{R}^{H}$, initialised to zero. At the same injection
layer, every token whose chunk index is $m$ has $\alpha\,e_{a_m}$
added in place:
\begin{equation}\label{eq:steer}
h_t^{(l^{*})}\;\leftarrow\;h_t^{(l^{*})} + \alpha\,e_{a_{m(t)}},
\qquad t\in\{1,\dots,T\},
\end{equation}
with $\alpha>0$ a fixed steering scale.

\subsection{Objective}
\label{sec:method:loss}

Write $x\equiv x_{1:T}$ for the NL sequence and $a\equiv a_{0:M-1}$
for the chunk-level code sequence. We use
$p_{\theta}(x)$ for the unsteered likelihood,
$p_{\theta}(x\mid a)$ for the conditional likelihood given codes $a$, 
substituted via Eq.~\eqref{eq:steer}, and
$p_{\theta}(a\mid x)$ for the routing head's distribution over
chunk codes read off from Eq.~\eqref{eq:code};
$\mathrm{sg}(\cdot)$ denotes stop-gradient.
The per-step objective is
\begin{equation}\label{eq:total-loss}
\begin{aligned}
\mathcal{L}_{\mathrm{DLR}}(x,a)
&\;=\;
\underbrace{-\log p_{\theta}(x)}_{\text{Generalist}}
\;-\;
\underbrace{\log\frac{p_{\theta}(x\mid a)}{\mathrm{sg}\bigl(p_{\theta}(x)\bigr)}}_{\text{Information Gain}}
\;-\;\alpha_{\mathrm{policy}}\,
\underbrace{\log p_{\theta}(a\mid x)}_{\text{Policy Optimization}} \\
&\;+\;\alpha_{\mathrm{reg}}\,
\underbrace{D_{\mathrm{KL}}\!\bigl(p_{\theta}(a)\,\big\|\,p_{\text{bi-zipf}}(a)\bigr)}_{\text{Marginal Entropy Regularization}}.
\end{aligned}
\end{equation}
\textbf{Generalist.} $-\log p_{\theta}(x)$ is the standard LM loss, preserving the unsteered model's capacity. 

\textbf{Information Gain.} The log-ratio
$\log\bigl[p_{\theta}(x\mid a)/\mathrm{sg}(p_{\theta}(x))\bigr]$
is the loss-side information-gain term ~\citep{hatamizadeh2025rlp}: it rewards code sequences that improve
the steered likelihood over the unsteered baseline, analogous to the
skill-conditioned intrinsic reward ~\citep{sharma2020dads}. 

\textbf{Policy and Marginal Entropy.} The last two terms are the mutual-information loss
$-\log p_{\theta}(a\mid x)+D_{\mathrm{KL}}(p_{\theta}(a)\|p(a))$
for unsupervised skill discovery~\citep{eysenbach2019diayn}, encouraging global diversity whilst reducing local uncertainty.
$p(a)$ is instantiated as bigram-Zipfian 
$p_{\text{bi-zipf}}$ (Appendix~\ref{app:zipf}). The stop-gradient in
Eq.~\eqref{eq:code} keeps routing-side and LM-side gradients separated,
as in standard skill-discovery objectives~\citep{eysenbach2019diayn,sharma2020dads}.

\subsection{Search-Select-Update Procedure}
\label{sec:method:search}

The reward of a code sequence $a$ given an NL sequence $x$ is 
$R_{\theta}(x,a)=\log p_{\theta}(x\mid a)$, this changes as model $p_{\theta}$ are updated, coupling model $p_{\theta}(x | a)$ and policy $p_{\theta}(a | s)$. A separately optimized policy quickly degrades as the model changes, and an optimized model may not perform well under an updated policy due to degrading $\mathbb{E}_{a \sim p_{\theta}(a | s)}[\log p_{\theta}(s | a)]$. We therefore propose Dynamic Latent Routing (DLR) to jointly update policy $p_{\theta}(a | s)$ and model $p_{\theta}(s | a)$ via a search, select update process, serving as a neural relaxation of General Dijkstra Search \S\ref{sec:gds}, where we replace priority query with a policy module to "record" and "pick" the best codes. 

% addresses both with a per-step search (Algorithm~\ref{alg:sorl-step}),
% a single-step relaxation of the General Dijkstra Search of
% \S\ref{sec:gds} that returns the best-scoring rollout $a^{*}$.

\begin{algorithm}[t]
\caption{Dynamic Latent Routing (DLR).}
\label{alg:sorl-step}
Initialize $\theta$ (base LM, codebook $\{e_k\}$, routing head $W_{\mathrm{rt}}$)\;
Given a temperature schedule $t_{1:N}$ and total training steps $S$\;
\For{$s = 1$ \KwTo $S$}{
  Sample $a^{(i)} \sim p_{\theta}^{\,t_i}(a\mid x)$ for $i\in\{1,\dots,N\}$\;
  Select $a^{*} = \arg\max_{a^{(i)}}\, p_{\theta}(x\mid a^{(i)})$\;
  Optimize $\mathcal{L}_{\mathrm{DLR}}(x,a^{*})$\;
}
\end{algorithm}

\section{Experiments}
\label{sec:experiments}

\paragraph{Setup.} We evaluate DLR against SFT, Pause Tokens~\citep{goyal2023pause},
and TokenAssorted~\citep{su2025tokenassorted} on five base models---Qwen3-\{0.6B,
1.7B, 4B\}~\citep{qwen3technical} and Llama3.2-\{1B, 3B\}~\citep{grattafiori2024llama3}---across
four QA benchmarks: GSM8K~\citep{cobbe2021gsm8k}, ScienceQA~\citep{lu2022learn},
StrategyQA~\citep{geva2021strategyqa}, and CSQA~\citep{talmor2019commonsenseqa}.
We also include a $C{=}1$ DLR variant that learns a single steering vector jointly
with the model, providing a fair comparison to representation
engineering~\citep{zou2023representation} which freezes the model and updates only
the steering vector. \textbf{All methods are matched in data and optimizer steps}:
one epoch on each dataset's standard training split, lr $10^{-5}$, effective batch
size 8, identical optimizer steps. DLR uses codebook size $C{=}32$, abstraction
ratio $K{=}4$, number of rollouts $N{=}8$. Experiments are run on $2{\times}$H100.

\paragraph{A low-data regime.} Inserting new tokens disrupts the natural-language
structure that pre-trained models rely on, requiring substantial compute to
overcome. Pause Tokens~\citep{goyal2023pause} use a dedicated pre-finetuning
stage on a large corpus to make the model tolerate this disruption, and
TokenAssorted~\citep{su2025tokenassorted} requires multi-epoch training on a
merged multi-domain corpus to be effective. Our setting is the opposite: a
single epoch on each dataset's native split, often an order of magnitude
smaller. Methods that alter the sequence structure need this adaptation budget
before the new tokens acquire meaning; DLR avoids this by adding steering
vectors to the hidden states, leaving the input sequence intact and well-defined
from the first step.

\paragraph{Results.} Table~\ref{tab:main} reports test accuracy with 95\%
bootstrap confidence intervals. DLR wins every (model, dataset) cell, with a
$+6.6$\,pp mean gain over SFT across all 24 cells ($+7.8$\,pp on the three
reasoning benchmarks). Pause Tokens roughly tracks SFT on GSM8K/SciQA
and trails on StrategyQA; TokenAssorted collapses on math/reasoning at small
scale ($15.7$ vs.\ SFT $46.0$ on GSM8K-Qwen3-0.6B; $17.3$ vs.\ SFT $56.4$ on
SciQA-Qwen3-1.7B), consistent with the regime note above. DLR's largest gains
land on the same reasoning cells where TokenAssorted collapses
(SciQA: $+18.8$\,pp on Qwen3-8B, $+12.9$\,pp on Qwen3-4B;
GSM8K: $+10.2$\,pp on Llama-1B, $+6.7$\,pp on Qwen3-4B)---evidence that
\emph{data-dependent routing} unlocks new capability without disturbing the
language model.

% --- Legacy commentary kept for reference; superseded by paragraphs above. ---
\iffalse
The two competing latent-augmentation baselines do not uniformly improve
on SFT in this regime: pause tokens roughly match SFT on GSM8K/SciQA
and add a small margin on CSQA, while TokenAssorted catastrophically
degrades math/reasoning at smaller scales (e.g.\ $15.7$ vs.\ SFT
$46.0$ on GSM8K-Qwen3-0.6B; $17.3$ vs.\ SFT $56.4$ on
SciQA-Qwen3-1.7B). The TokenAssorted failure is informative rather
than incidental: replacing whole NL-token chunks with newly
introduced abstract tokens breaks the input sequence structure, and
\citet{su2025tokenassorted} recovers from this perturbation only by
training for many more epochs on far larger corpora; under
single-epoch low-data finetuning the abstract tokens never acquire
useful meaning. By contrast, DLR adds steering \emph{into} an
unaltered NL sequence and is therefore well-defined from step one.
DLR's gains are largest on the same math/reasoning benchmarks where
TokenAssorted collapses (GSM8K: $+5.5$\,pp on Qwen3-1.7B,
$+10.2$\,pp on Llama-1B, $+6.7$\,pp on Qwen3-4B), indicating that
data-dependent routing unlocks genuinely new capability without
disturbing the underlying language model.
\fi

\paragraph{Ablations on objective and hyperparameters.} We first validate the necessity of
each component of the DLR objective (Eq.~\eqref{eq:total-loss}) with an
extensive ablation grid over 5 models $\times$ 4 datasets (excluding
Qwen3-8B): removing any term degrades accuracy or triggers codebook
collapse; detailed results and analysis are deferred to
App.~\ref{app:ablations}. We further sweep the codebook size $C$, the
abstraction ratio $K$, the number of rollouts $N$, the rollout temperature
$\tau$, and the loss weights $\alpha_{\mathrm{reg}},\alpha_{\mathrm{policy}}$
(App.~\ref{app:hp-sweeps}). Accuracy degrades when the codebook is
non-diverse ($C{=}1$), when rollouts are non-diverse ($\tau{=}0$), when
search is removed ($N{=}1$), or when the policy head is untrained
($\alpha_{\mathrm{policy}}{=}0$), validating the search-select-update
design of DLR.

\begin{table}[t]
  \centering
  \footnotesize
  \setlength{\tabcolsep}{5pt}
  \caption{Main results: test accuracy (\%) on four QA benchmarks.
    \textbf{DLR} is compared against SFT and two
    latent-augmentation baselines (pause tokens, TokenAssorted)
    matched in data and optimiser. \textbf{Bold}: best
    per (model, dataset). Subscripts are 95\%\,CI half-widths
    ($\pm$pp) from bootstrap resampling on the test sets.
    Split sizes are shown beneath each column header.}
  \label{tab:main}
  \begin{tabular}{l l cccc}
    \toprule
    Model & Method & GSM8K & SciQA & StrategyQA & CSQA \\
          &        & \scriptsize{(/1319)} & \scriptsize{(/2224)}
          & \scriptsize{(/687)} & \scriptsize{(/1221)} \\
    \midrule
    % ---- Qwen3-0.6B ----
    \multirow{5}{*}{Qwen3-0.6B}
      & SFT           & \val{46.0}{2.7} & \val{48.0}{2.1} & \val{47.0}{3.7} & \val{64.1}{2.7} \\
      & PauseToken    & \val{46.2}{2.7} & \val{47.9}{2.1} & \val{24.6}{3.2} & \val{64.7}{2.7} \\
      & TokenAssorted & \val{15.7}{2.0} & \val{13.1}{1.4} & \val{29.4}{3.4} & \val{64.2}{2.7} \\
      & DLR ($C{=}1$) & \val{45.8}{2.7} & \val{48.9}{2.1} & \val{50.4}{3.7} & \val{64.0}{2.7} \\
      \rowcolor{dlrgold}
      & \textbf{DLR}  & \bval{49.4}{2.7} & \bval{55.3}{2.1} & \bval{54.6}{3.7} & \bval{66.4}{2.6} \\
    \midrule
    % ---- Qwen3-1.7B ----
    \multirow{5}{*}{Qwen3-1.7B}
      & SFT           & \val{60.2}{2.6} & \val{56.4}{2.1} & \val{51.7}{3.7} & \val{74.3}{2.5} \\
      & PauseToken    & \val{60.9}{2.6} & \val{56.2}{2.1} & \val{32.0}{3.5} & \val{76.4}{2.4} \\
      & TokenAssorted & \val{28.5}{2.4} & \val{17.3}{1.6} & \val{33.0}{3.5} & \val{76.3}{2.4} \\
      & DLR ($C{=}1$) & \val{63.5}{2.6} & \val{60.0}{2.0} & \val{55.3}{3.7} & \val{77.4}{2.3} \\
      \rowcolor{dlrgold}
      & \textbf{DLR}  & \bval{65.7}{2.6} & \bval{64.1}{2.0} & \bval{60.3}{3.7} & \bval{78.4}{2.3} \\
    \midrule
    % ---- Qwen3-4B ----
    \multirow{5}{*}{Qwen3-4B}
      & SFT           & \val{75.4}{2.3} & \val{59.4}{2.0} & \val{65.8}{3.5} & \val{82.5}{2.1} \\
      & PauseToken    & \val{78.6}{2.2} & \val{63.8}{2.0} & \val{37.6}{3.6} & \val{80.2}{2.2} \\
      & TokenAssorted & \val{60.9}{2.6} & \val{68.0}{1.9} & \val{39.3}{3.7} & \val{80.6}{2.2} \\
      & DLR ($C{=}1$) & \val{78.6}{2.2} & \val{64.1}{2.0} & \val{66.2}{3.5} & \val{81.1}{2.2} \\
      \rowcolor{dlrgold}
      & \textbf{DLR}  & \bval{82.1}{2.1} & \bval{72.3}{1.9} & \bval{68.7}{3.5} & \bval{84.0}{2.1} \\
    \midrule
    % ---- Qwen3-8B ----
    \multirow{5}{*}{Qwen3-8B}
      & SFT           & \val{78.9}{2.1} & \val{61.8}{2.0} & \val{68.2}{3.5} & \val{85.6}{2.1} \\
      & PauseToken    & \val{79.4}{2.1} & \val{65.2}{2.0} & \val{38.8}{3.6} & \val{82.9}{2.2} \\
      & TokenAssorted & \val{63.0}{2.5} & \val{69.7}{1.9} & \val{41.2}{3.7} & \val{83.1}{2.2} \\
      & DLR ($C{=}1$) & \val{78.9}{2.2} & \val{66.4}{2.0} & \val{70.3}{3.5} & \val{83.8}{2.1} \\
      \rowcolor{dlrgold}
      & \textbf{DLR}  & \bval{84.3}{2.0} & \bval{80.6}{1.7} & \bval{72.9}{3.4} & \bval{87.2}{1.9} \\
    \midrule
    % ---- Llama3.2-1B ----
    \multirow{5}{*}{Llama3.2-1B}
      & SFT           & \val{30.9}{2.5} & \val{38.0}{2.0} & \val{48.0}{3.7} & \val{65.2}{2.7} \\
      & PauseToken    & \val{31.3}{2.5} & \val{34.0}{2.0} & \val{25.5}{3.3} & \val{64.9}{2.7} \\
      & TokenAssorted & \val{14.2}{1.9} & \val{33.6}{2.0} & \val{29.3}{3.4} & \val{65.6}{2.7} \\
      & DLR ($C{=}1$) & \val{38.4}{2.6} & \val{43.7}{2.1} & \val{48.6}{3.7} & \val{68.3}{2.6} \\
      \rowcolor{dlrgold}
      & \textbf{DLR}  & \bval{41.1}{2.7} & \bval{49.1}{2.1} & \bval{51.4}{3.7} & \bval{71.4}{2.5} \\
    \midrule
    % ---- Llama3.2-3B ----
    \multirow{5}{*}{Llama3.2-3B}
      & SFT           & \val{41.3}{2.7} & \val{56.7}{2.1} & \val{52.5}{3.7} & \val{79.8}{2.3} \\
      & PauseToken    & \val{41.5}{2.7} & \val{57.6}{2.1} & \val{29.0}{3.4} & \val{77.1}{2.4} \\
      & TokenAssorted & \val{21.8}{2.2} & \val{34.6}{2.0} & \val{34.7}{3.6} & \val{77.3}{2.4} \\
      & DLR ($C{=}1$) & \val{46.3}{2.7} & \val{59.2}{2.0} & \val{56.1}{3.7} & \val{79.6}{2.3} \\
      \rowcolor{dlrgold}
      & \textbf{DLR}  & \bval{49.1}{2.7} & \bval{63.3}{2.0} & \bval{62.3}{3.6} & \bval{81.0}{2.2} \\
    \bottomrule
  \end{tabular}
\end{table}

% =================================================================
\section{Analysis}
\label{sec:analysis}

Having established downstream gains, we now examine the routing structure DLR
learns, focusing on ScienceQA dataset. We present four analyses supporting two claims: routing is dynamic and
context-dependent, and individual
codes carry distinct, topic-specific causal effects. We remark that topic label is not accessible during 
training, the emerged structure comes solely from the DLR objective.

% -------------------------------------------------------------
% \subsection{Codebook diversity}

\begin{table}[t]
  \centering
  \footnotesize
  \caption{Codebook diversity on SciQA: mean cosine similarity (cos) between distinct vectors and fraction of codes used at least once (util.).}
  \label{tab:codebook-diversity}
  \begin{tabular}{lcccccc}
    \toprule
    Model & Qwen3-0.6B & Qwen3-1.7B & Qwen3-4B & Qwen3-8B & Llama3.2-1B & Llama3.2-3B \\
    \midrule
    cos   & $0.24$ & $0.28$ & $0.16$ & $0.07$ & $0.01$ & $0.00$  \\
    util. & $41\%$ & $31\%$ & $56\%$ & $76\%$ & $78\%$ & $100\%$ \\
    \bottomrule
  \end{tabular}
\end{table}

\textbf{Diverse routing policy.} We first measure mean cosine similarity between distinct
vectors in the codebook, as well as the fraction of codes selected at least
once by the router. Table~\ref{tab:codebook-diversity} reports the results:
steering vectors are distinct across all models (cosine ${\leq}0.28$), and code
utilization ranges from $31\%$ to $100\%$, indicating that DLR learns a
diverse, non-collapsed codebook.

% -------------------------------------------------------------
% \subsection{N-gram topic specialization}
\textbf{Structured, context-dependent routing.} We next measure topic purity for code $n$-grams (appearing more than $30$ times) using the topic labels provided
in ScienceQA (e.g., physics, biology, chemistry). We define topic purity of a
code $n$-gram as the proportion of its occurrences assigned to its top topic,
capturing whether the $n$-gram specializes to a particular context. A random
router would achieve a purity of $0.17$ due to unbalanced topic distribution.
Figure~\ref{fig:ngram-purity} shows that code $n$-grams are consistently more
topic-pure than this baseline, and purity increases with $n$-gram length,
indicating that longer code compositions are increasingly topic-specialized.
This supports the claim that DLR learns structured, input-dependent routing
patterns rather than assigning codes randomly.

\begin{figure}[!t]
  \centering
  \includegraphics[width=0.85\linewidth]{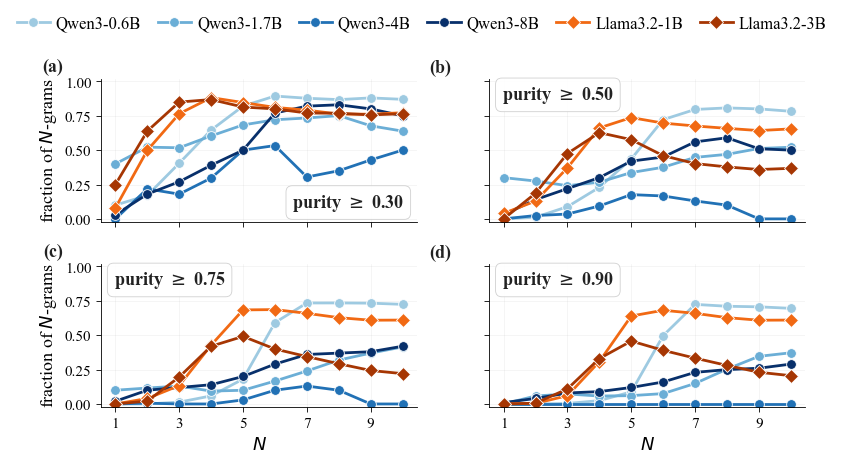}
  \caption{Fraction of $N$-grams with above threshold topic purity vs.\ $N$.}
  \label{fig:ngram-purity}
\end{figure}

% -------------------------------------------------------------
% \subsection{Necessity: global perturbation}
\textbf{Codes are causally load-bearing.}
\label{sec:necessity}
Two interventions confirm that the learned routes are necessary for
DLR's gains. Globally zeroing the steering scale or replacing routed codes with
random alternatives reduces SciQA accuracy by $6.2$--$17.4$\,pp and
$4.8$--$11.7$\,pp respectively across all Qwen models
(Table~\ref{tab:global-ablation}). Locally ablating a single code at a time
also degrades accuracy on aggregate, with mean drops of
$0.71$--$2.74$\,pp over four tested codes per model
(Table~\ref{tab:percode-acc-summary}).

\begin{table}[!t]
  \centering
  \small
  \begin{minipage}[t]{0.52\linewidth}
    \centering
    \setlength{\tabcolsep}{5pt}
    \begin{tabular}{lrrr}
      \toprule
      Model      & steered & scale$\to$0       & rand.\ replace    \\
      \midrule
      Qwen3-0.6B & $55.3\%$ & \dcell{2}{$-6.2$}  & \dcell{1}{$-4.8$}  \\
      Qwen3-1.7B & $64.1\%$ & \dcell{2}{$-6.4$}  & \dcell{2}{$-5.7$}  \\
      Qwen3-4B   & $72.3\%$ & \dcell{4}{$-12.2$} & \dcell{3}{$-8.1$}  \\
      Qwen3-8B   & $80.6\%$ & \dcell{5}{$-17.4$} & \dcell{4}{$-11.7$} \\
      \bottomrule
    \end{tabular}
    \caption{Global perturbation: SciQA $\Delta$acc (pp); cell intensity scales with drop magnitude.}
    \label{tab:global-ablation}
  \end{minipage}\hfill
  \begin{minipage}[t]{0.45\linewidth}
    \centering
    \setlength{\tabcolsep}{6pt}
    \begin{tabular}{lrrr}
      \toprule
      Model      & mean $\Delta$       & min                & max                \\
      \midrule
      Qwen3-0.6B & \dcell{2}{$-2.74$}  & \dcell{2}{$-3.90$} & \dcell{1}{$-1.43$} \\
      Qwen3-1.7B & \dcell{1}{$-0.71$}  & \dcell{1}{$-1.19$} & \dcell{1}{$-0.19$} \\
      Qwen3-4B   & \dcell{1}{$-1.06$}  & \dcell{1}{$-1.48$} & \dcell{1}{$-0.57$} \\
      Qwen3-8B   & \dcell{1}{$-1.31$}  & \dcell{2}{$-1.72$} & \dcell{1}{$-0.84$} \\
      \bottomrule
    \end{tabular}
    \caption{Single-code ablation: SciQA $\Delta$acc (pp) over four codes per model.}
    \label{tab:percode-acc-summary}
  \end{minipage}
\end{table}

% Per-topic, the aggregate masks sign-flipping effects (Fig.~\ref{fig:ablation-topic-delta}): ablating code~0 lifts \emph{biology} $+3.6$\,pp on Qwen3-0.6B (while dropping \emph{physics} $-4.0$\,pp), \emph{writing-strategies} $+4.8$\,pp on Qwen3-1.7B, and \emph{chemistry} $+9.8$\,pp on Qwen3-4B. Each code is a \emph{topic-conditional} handle.
\textbf{Codes are topic-specific.}
The aggregate single-code drop hides sign-flipping per-topic effects
(Fig.~\ref{fig:ablation-topic-delta}): ablating code~$0$ lifts \emph{biology}
by $+3.6$\,pp on Qwen3-0.6B while costing \emph{physics} $-4.0$\,pp, and
similarly improves \emph{writing-strategies} ($+4.8$\,pp, Qwen3-1.7B) and
\emph{chemistry} ($+9.8$\,pp, Qwen3-4B). Individual codes thus act as
topic-conditional handles rather than uniform performance knobs.

\begin{figure}[!t]
  \centering
  \includegraphics[width=0.88\linewidth]{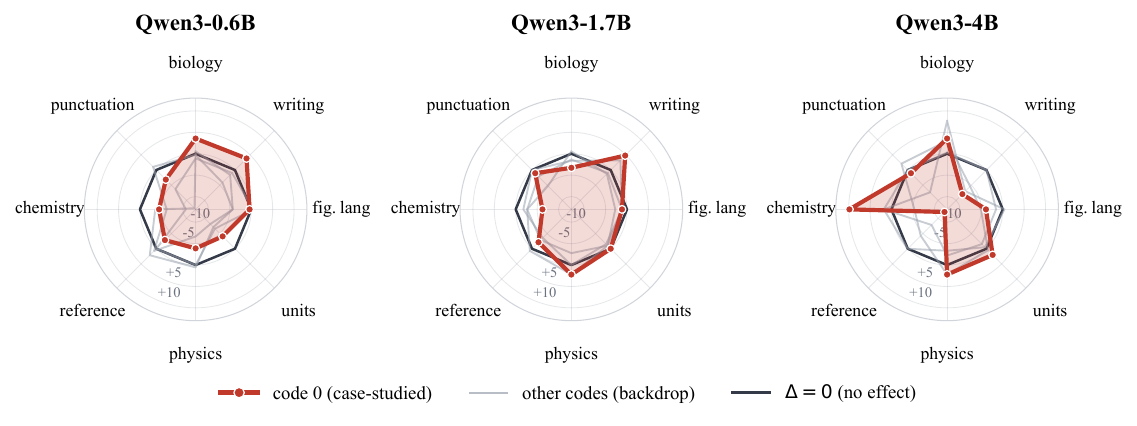}
  \caption{Per-topic $\Delta$acc (pp) under single-code ablation.}
  \label{fig:ablation-topic-delta}
\end{figure}

\subsection{Case study: six-digit arithmetic}
\label{sec:arithmetic}
While ScienceQA reveals structured and causally meaningful routing behavior, the underlying computation remains latent. We therefore turn to six-digit arithmetic, where the internal algorithmic structure is explicitly known: \citet{quirke_2024_addsub_preprint} identify several mutually exclusive subtask types at each answer-digit position, including carry generation, carry use, sum-9 boundary detection, and borrow cascades, that transformers must implement to solve the task.
(See Table~\ref{tab:quirke-subtasks} for details).
This makes arithmetic an ideal testbed for one application of DLR: that abstraction codes may \emph{externalize} reasoning steps, making them directly observable and intervenable without any activation-level tooling. We train a 2 layer, 1 Head, 128 hidden dimensional (\texttt{2L/1H/128d}) transformer on a synthetically generated dataset.

DLR inserts one abstraction code per answer-digit position ($K{=}1$,
codebook size $|\mathcal{A}|{=}30$), so each routing decision is a named, discrete symbol emitted at generation time.
Figure~\ref{fig:arithmetic-example} shows this concretely:
for $959{,}271 + 040{,}756 = 1{,}000{,}027$ (a four-deep carry cascade),
DLR assigns \texttt{t2} at every cascade position (UC/US),
\texttt{t6} at the sum-9 boundary (SS), and distinct codes at the trivial positions (SC/SA) - the carry structure is readable off the code sequence with no probing or patching required.
Full training and architecture details are in Appendix~\ref{app:training}.

\textbf{Abstraction codes recover known circuits without supervision.}
Analysis of \texttt{2L/1H/128d} shows that DLR's codebook
spontaneously partitions into subtask-specialist codes:
each of the 23 active codes concentrates on a narrow slice of the Quirke taxonomy - dominant subtask accounts for ${\geq}70\%$ of occurrences for the majority of codes - and is locked to one or two answer positions.
This structure emerges from the info-gain loss alone, with no access to ground-truth subtask labels or carry state.
Codes are also causally necessary: knocking out all codes collapses accuracy from 95.5\% to 0.1\%, confirming they carry the computation rather than merely annotating it.

\textbf{Named codes enable stronger fine-tuning and causal interventions.}
Because the routing codes are discrete and named, surgical model edits are possible that have no analog in standard transformers:
swapping a single code at one answer position fixes wrong predictions
at a 27-31\% rate on carry-heavy examples (cross-operation transplant:
93.5\% vs.\ 75.5\% random baseline).
Interpretability here is not merely post-hoc - it translates directly into the ability to correct the model.

Correspondingly, finetuned DLR outperforms \sft{} strongly on 12 of 13 tested
(architecture, data-size) configurations, and on \emph{all 13} on
the hardest 6-deep carry cascades, with gains as large as $+50$\,pp
(Table~\ref{tab:undersized-wins}).
The margin grows with cascade depth, consistent with explicit carry/borrow routing being the mechanism behind the gain. See Appendix \ref{app:arithmetic} for further details on DLR interpretability, including a demonstration of auto-interp ~\citep{bills2023language_models_explain_neurons}, code specializations and polysemantic codes.

\section{Conclusion}

We proved that under dynamic MDPs, optimal goal-reaching policies can be obtained by concatenating sub-policies via General Dijkstra Search. Applying the principle of GDS, we proposed DLR to train an LLM to create, search, and concatenate its own latent codes within a single training phase. Compared to other discrete-latent baselines, DLR uniquely outperforms SFT by $+6.6$\,pp under the low-data fine-tuning regime. Analyses show that codes learned by DLR specialize to distinct subtasks and are causally load-bearing without any subtask labels provided during training. These results suggest that DLR can serve as a practical mechanism for learning compositional internal control in language models. We discuss limitations in Appendix~\ref{sec:limitations}.

\bibliographystyle{plainnat}
\bibliography{reference}
% \end{bibliography}

% \section*{References}

% References follow the acknowledgments in the camera-ready paper. Use unnumbered first-level heading for
% the references. Any choice of citation style is acceptable as long as you are
% consistent. It is permissible to reduce the font size to \verb+small+ (9 point)
% when listing the references.
% Note that the Reference section does not count towards the page limit.
% \medskip

% {
% \small

% [1] Alexander, J.A.\ \& Mozer, M.C.\ (1995) Template-based algorithms for
% connectionist rule extraction. In G.\ Tesauro, D.S.\ Touretzky and T.K.\ Leen
% (eds.), {\it Advances in Neural Information Processing Systems 7},
% pp.\ 609--616. Cambridge, MA: MIT Press.

% [2] Bower, J.M.\ \& Beeman, D.\ (1995) {\it The Book of GENESIS: Exploring
%   Realistic Neural Models with the GEneral NEural SImulation System.}  New York:
% TELOS/Springer--Verlag.

% [3] Hasselmo, M.E., Schnell, E.\ \& Barkai, E.\ (1995) Dynamics of learning and
% recall at excitatory recurrent synapses and cholinergic modulation in rat
% hippocampal region CA3. {\it Journal of Neuroscience} {\bf 15}(7):5249-5262.
% }

%%%%%%%%%%%%%%%%%%%%%%%%%%%%%%%%%%%%%%%%%%%%%%%%%%%%%%%%%%%%

\appendix

% ------------------------------------------------------------------
% Hand-typed Appendix TOC (avoids etoc / minitoc package fragility).
% Page numbers resolve via \pageref to each \label below; requires
% one extra LaTeX pass after first compile.
% ------------------------------------------------------------------
\newcommand{\apptocentry}[3]{%
  \noindent\makebox[1.6em][l]{\textbf{#1.}}%
  \textbf{\hyperref[#3]{#2}}%
  \nobreak\leaders\hbox to 0.6em{\hss.\hss}\hfill%
  \nobreak\hbox to 1.6em{\hfil\pageref{#3}}\par
  \vspace{2pt}}
\newcommand{\apptocsub}[1]{%
  \noindent\hspace*{2.6em}{\small\itshape #1}\par
  \vspace{1pt}}

\begin{center}
  {\Large\bfseries Appendix Contents}
\end{center}
\vspace{0.5em}

\noindent\hrulefill\par
\vspace{0.6em}

{\noindent\small Part~I: Theory.}\par
\vspace{0.4em}

\apptocentry{A}{Theory for Dynamic MDP and Existence of Optimal Policy}{app:dmdp-theory}
\apptocsub{Time-indexed value/Q-functions; Bellman optimality for DMDPs}
\apptocsub{Existence and concatenation theorems with full proofs}

\apptocentry{B}{Policy Composition}{app:policy-composition}
\apptocsub{Policy dominance, dominating/dominated goal sets}
\apptocsub{Lemmas~1--8 and Optimality Theorem for GDS (full proofs)}

\apptocentry{C}{Symmetric goal-covering result}{app:gds-coverage}
\apptocsub{Dual algorithm and optimality for goal-covering policies}

\vspace{0.6em}
{\noindent\small Part~II: Method.}\par
\vspace{0.4em}

\apptocentry{D}{Bigram-Zipfian prior}{app:zipf}
\apptocsub{KL-divergence regulariser on consecutive code pairs}

\vspace{0.6em}
{\noindent\small Part~III: Ablations and sweeps.}\par
\vspace{0.4em}

\apptocentry{E}{Loss-term ablations}{app:ablations}
\apptocsub{$\alpha_{\mathrm{policy}}{=}0$, $\alpha_{\mathrm{base}}{=}0$, $\alpha_{\mathrm{reg}}{=}0$ across 5 models $\times$ 4 datasets}

\apptocentry{F}{Hyperparameter sweeps}{app:hp-sweeps}
\apptocsub{Codebook size $C$, abstraction ratio $K$, marginal-entropy weight $\alpha_{\mathrm{reg}}$, policy weight $\alpha_{\mathrm{policy}}$}
\apptocsub{Search-budget $N$ ($N{=}1$ vs.\ $N{=}8$) and sampling temperature $\tau$ ($\tau{=}0,2$)}

\vspace{0.6em}
{\noindent\small Part~IV: Empirical results.}\par
\vspace{0.4em}

\apptocentry{G}{Global ablation: full results}{app:global-ablation-full}
\apptocsub{All five models (excluding Qwen3-8B) under \emph{steered}, \emph{scale}\,$\to$\,$0$, and \emph{random-replace}}
\apptocsub{Llama checkpoints extending the Qwen-only main-text Table~3}

\apptocentry{H}{Single-code ablation case studies (full traces)}{app:case-studies}
\apptocsub{Qwen3-0.6B \textbullet\ biology \textbullet\ ablating code~0}
\apptocsub{Qwen3-1.7B \textbullet\ writing-strategies \textbullet\ ablating code~0}
\apptocsub{Qwen3-4B \textbullet\ chemistry \textbullet\ ablating code~0}

\apptocentry{I}{Arithmetic case study: interpretability analysis}{app:arithmetic}
\apptocsub{Performance of DLR vs.\ SFT on undersized architectures and causal ablations}
\apptocsub{Code--subtask heatmap, polysemantic codes, and automated code interpretation}

\vspace{0.6em}
\noindent\hrulefill\par
\vspace{0.4em}

\noindent\small Page numbers shown at right are clickable hyperlinks in the
PDF version. Cross-references in the main text use the labels listed above
(e.g., App.~\ref{app:ablations} for loss-term ablations).
\normalsize

\clearpage

\section{Theory for Dynamic MDP and Existence of Optimal Policy}
\label{app:dmdp-theory}
\begin{definition}[Dynamic value function and Q-function]
Given a dynamic policy $\pi \in \Pi$, its (time-indexed) value function
$V^{\pi} : \mathcal{S} \times \{0,\dots,T\} \to \mathbb{R}$ is defined by
\[
V^{\pi}_{t}(s)
:= \mathbb{E}_{\pi}\Bigl[ \sum_{i=t}^{T} \gamma^{\,i-t} \,
    r_{i}(s_{i}, a_{i}) \,\Big|\, s_{t} = s \Bigr],
\]
where $a_i = \pi_i(s_i)$ for $i \ge t$, and the expectation is taken over
trajectories $(s_t,a_t,\dots,s_T,a_T)$ induced by $\pi$ and $P$.

Its (time-indexed) Q-function $Q^{\pi} : \mathcal{S} \times \mathcal{A}
\times \{0,\dots,T\} \to \mathbb{R}$ is defined by
\[
Q_{t}^{\pi}(s,a)
:= \mathbb{E}_{\pi}\Bigl[ \sum_{i=t}^{T} \gamma^{\,i-t} \,
    r_{i}(s_{i}, a_{i}) \,\Big|\, s_{t} = s,\, a_{t} = a \Bigr],
\]
where $a_t$ is fixed to be $a$ and $a_i = \pi_i(s_i)$ for $i > t$.
\end{definition}

\begin{definition}[Finite-horizon truncation]
Let $\pi = (\pi_{0},\pi_{1},\dots,\pi_{T-1})$ be a time-varying policy with
horizon $T$. For any $n \le T$, define the truncated policy by
\[
\pi_{0:n} := (\pi_{0},\pi_{1},\dots,\pi_{n-1}).
\]
\end{definition}

\begin{definition}[Finite-horizon value function and Q-function]
Let $\pi = (\pi_{0},\pi_{1},\dots,\pi_{T-1})$ be a time-varying policy. Its
finite-horizon value function and Q-function are defined, for $t=0,\dots,T-1$,
by
\[
V_{t}^{\pi}(s)
:= \mathbb{E}\Bigg[
    \sum_{i=t}^{T-1} \gamma^{\,i-t}\, r_{i}(s_{i}, a_{i})
    \,\Big|\, s_{t}=s,\ a_{i}=\pi_{i}(s_{i})
\Bigg],
\]
\[
Q_{t}^{\pi}(s,a)
:= r_{t}(s,a)
 + \mathbb{E}\Bigg[
    \sum_{i=t+1}^{T-1} \gamma^{\,i-t}\, r_{i}(s_{i}, a_{i})
    \,\Big|\, s_{t}=s,\ a_{t}=a,\ a_{i}=\pi_{i}(s_{i})
\Bigg].
\]
For a truncated policy $\pi_{0:n}$, the same formulas apply with terminal time
$n-1$ in place of $T-1$.
\end{definition}

\begin{theorem}[Bellman equations for finite-horizon time-varying policies]
For any finite-horizon time-varying policy $\pi$,
\[
V^{\pi}_{t}(s)
= r_{t}\bigl(s,\pi_{t}(s)\bigr)
  + \gamma \sum_{s'} p\bigl(s' \mid s,\pi_{t}(s)\bigr) V^{\pi}_{t+1}(s'),
\quad t=0,\dots,T-1,
\]
and
\[
Q_{t}^{\pi}(s,a)
= r_{t}(s,a)
  + \gamma \sum_{s'} p(s' \mid s,a) V^{\pi}_{t+1}(s'),
\quad t=0,\dots,T-1.
\]
The same identities hold for truncated policies on their corresponding finite
horizon.
\end{theorem}

\begin{theorem}[Dynamic Bellman equations]

Given any (possibly time-dependent) policy $\pi = \{\pi_t\}_{t \ge 0}$ within a dynamic MDP environment $(\mathcal{S}, \mathcal{A}, P, \{r_t\}_{t \ge 0}, \gamma)$, its value function $V^{\pi}$ and Q-function $Q^{\pi}$ satisfy the (dynamic)
Bellman equations:
\begin{align}
Q_{t}^{\pi}(s,a)
&= r_{t}(s,a) + \gamma \sum_{s'} P(s' \mid s,a)\, V_{t+1}^{\pi}(s'), \\
V^{\pi}_{t}(s)
&= \sum_{a} \pi_{t}(a \mid s)\, Q^{\pi}_{t}(s,a), \\
V^{\pi}_{t}(s)
&= \sum_{a} \pi_{t}(a \mid s)\Bigl[
    r_{t}(s,a)
    + \gamma \sum_{s'} P(s' \mid s,a)\, V^{\pi}_{t+1}(s')
\Bigr].
\end{align}
\end{theorem}

\begin{proof}
We start from the definition of $Q_t^{\pi}(s,a)$:
\begin{align*}
Q_t^{\pi}(s,a)
&= \mathbb{E}_{\pi}\Bigl[\sum_{i=t}^{\infty} \gamma^{i-t} \, r_i(s_i, a_i)
\,\Big|\, s_t = s, a_t = a \Bigr] \\
&= \mathbb{E}_{\pi}\Bigl[
r_t(s_t,a_t)
+ \sum_{i=t+1}^{\infty} \gamma^{i-t} \, r_i(s_i, a_i)
\,\Big|\, s_t = s, a_t = a \Bigr] \\
&= r_t(s,a)
+ \underbrace{\mathbb{E}_{\pi}\Bigl[
\sum_{i=t+1}^{\infty} \gamma^{i-t} \, r_i(s_i, a_i)
\,\Big|\, s_t = s, a_t = a \Bigr]}_{(I)}.
\end{align*}

We now simplify term $(I)$. First factor out $\gamma$ and re-index the powers:
\begin{align*}
(I)
&= \mathbb{E}_{\pi}\Bigl[
\sum_{i=t+1}^{\infty} \gamma^{i-t} \, r_i(s_i, a_i)
\,\Big|\, s_t = s, a_t = a \Bigr] \\
&= \mathbb{E}_{\pi}\Bigl[
\sum_{i=t+1}^{\infty} \gamma \, \gamma^{i-(t+1)} \, r_i(s_i, a_i)
\,\Big|\, s_t = s, a_t = a \Bigr] \\
&= \gamma \, \mathbb{E}_{\pi}\Bigl[
\sum_{i=t+1}^{\infty} \gamma^{i-(t+1)} \, r_i(s_i, a_i)
\,\Big|\, s_t = s, a_t = a \Bigr].
\end{align*}

Apply the tower property of conditional expectation, conditioning on $s_{t+1}$:
\begin{align*}
(I)
&= \gamma \, \mathbb{E}_{\pi}\Bigl[
\mathbb{E}_{\pi}\Bigl[
\sum_{i=t+1}^{\infty} \gamma^{i-(t+1)} \, r_i(s_i, a_i)
\,\Big|\, s_{t+1} \Bigr]
\,\Big|\, s_t = s, a_t = a \Bigr].
\end{align*}

By the definition of the time-varying value function $V_{t+1}^{\pi}$, we have
\[
V_{t+1}^{\pi}(s')
= \mathbb{E}_{\pi}\Bigl[
\sum_{i=t+1}^{\infty} \gamma^{i-(t+1)} \, r_i(s_i, a_i)
\,\Big|\, s_{t+1} = s' \Bigr].
\]
Thus,
\begin{align*}
(I)
&= \gamma \, \mathbb{E}_{\pi}\Bigl[
V_{t+1}^{\pi}(s_{t+1})
\,\Big|\, s_t = s, a_t = a \Bigr] \\
&= \gamma \sum_{s'} P(s' \mid s,a) \, V_{t+1}^{\pi}(s').
\end{align*}

Substituting this back into the expression for $Q_t^{\pi}(s,a)$ yields
\[
Q_t^{\pi}(s,a)
= r_t(s,a) + \gamma \sum_{s'} P(s' \mid s,a) \, V_{t+1}^{\pi}(s'),
\]
which proves the first Bellman equation.

Next, consider the value function $V_t^{\pi}(s)$:
\begin{align*}
V_t^{\pi}(s)
&= \mathbb{E}_{\pi}\Bigl[
\sum_{i=t}^{\infty} \gamma^{i-t} \, r_i(s_i, a_i)
\,\Big|\, s_t = s \Bigr].
\end{align*}
Condition on the first action $a_t$ chosen according to $\pi_t(\cdot \mid s)$ and
use the tower property:
\begin{align*}
V_t^{\pi}(s)
&= \mathbb{E}_{\pi}\Bigl[
\mathbb{E}_{\pi}\Bigl[
\sum_{i=t}^{\infty} \gamma^{i-t} \, r_i(s_i, a_i)
\,\Big|\, s_t = s, a_t \Bigr]
\,\Big|\, s_t = s \Bigr] \\
&= \sum_{a} \pi_t(a \mid s) \,
\mathbb{E}_{\pi}\Bigl[
\sum_{i=t}^{\infty} \gamma^{i-t} \, r_i(s_i, a_i)
\,\Big|\, s_t = s, a_t = a \Bigr] \\
&= \sum_{a} \pi_t(a \mid s) \, Q_t^{\pi}(s,a),
\end{align*}
which proves the second equation.

Finally, substituting the first equation into the second, we obtain
\begin{align*}
V_t^{\pi}(s)
&= \sum_{a} \pi_t(a \mid s) \, Q_t^{\pi}(s,a) \\
&= \sum_{a} \pi_t(a \mid s)
\Bigl[ r_t(s,a) + \gamma \sum_{s'} P(s' \mid s,a) \, V_{t+1}^{\pi}(s') \Bigr],
\end{align*}
which is the third equation. This completes the proof.
\end{proof}

\begin{definition}[Time-varying Bellman policy operator]
The \emph{time-varying Bellman policy operator}
\[
\mathit{TP} : \Pi \to \Pi
\]
is defined for any policy $\pi \in \Pi$, time index $t \in \mathbb{N}$, and state
$s \in \mathcal{S}$ by
\begin{align*}
(\mathit{TP} \cdot \pi)_{t}(s)
&:= \arg\max_{a \in \mathcal{A}} Q_{t}^{\pi}(s,a) \\
&= \arg\max_{a \in \mathcal{A}}
\Bigl[
    r_{t}(s,a)
    + \gamma \sum_{s'} p(s' \mid s,a)\, V_{t+1}^{\pi}(s')
\Bigr].
\end{align*}
\end{definition}
Intuitively, Bellman policy operator asks a policy to do what it believed to be its best choice at any point in time. 

\begin{lemma}
Let $\pi = \{\pi_{t}\}_{t=0}^{T}$ be a (time-varying) policy.
For any integers $n \in \mathbb{N}$ and $m \ge 1$, and any state $s \in \mathcal{S}$,
the value $V_{n}^{\pi}(s)$ can be written as a linear combination of 
$\{V_{n+m}^{\pi}(s')\}_{s' \in \mathcal{S}}$:
\[
V_{n}^{\pi}(s)
= \sum_{s' \in \mathcal{S}} C\bigl(\pi_{n:n+m-1}, s'\bigr)\, V_{n+m}^{\pi}(s')
  + C\bigl(\pi_{n:n+m-1}\bigr),
\]
where the coefficients $C(\pi_{n:n+m-1}, s') > 0$ and $C(\pi_{n:n+m-1})$ depend
only on the policy slice $\{\pi^{t}\}_{t=n}^{n+m-1}$ and the dynamics, but not on 
$V_{n+m}^{\pi}$.
\end{lemma}

\begin{proof}
Denote by $\pi^{n}$ the decision rule at time $n$, i.e.\ $a = \pi^{n}(s)$ for $s \in \mathcal{S}$.
By the (dynamic) Bellman equation, we have
\begin{align*}
V_{n}^{\pi}(s)
&= Q_{n}^{\pi}\bigl(s, \pi^{n}(s)\bigr) \\
&= r_{n}\bigl(s, \pi^{n}(s)\bigr)
   + \gamma(n+1) \sum_{s'} p\bigl(s' \mid s, \pi^{n}(s)\bigr)\,
     V_{n+1}^{\pi}(s').
\end{align*}
This shows that $V_{n}^{\pi}(s)$ is a linear combination of the values
$\{V_{n+1}^{\pi}(s')\}_{s' \in \mathcal{S}}$, where the coefficients (including the
constant term) depend only on $\pi^{n}$ and the transition kernel $p$.

By repeatedly unrolling this relation from time $n$ up to time $n+m-1$, we obtain
an expression of the form
\[
V_{n}^{\pi}(s)
= \sum_{s' \in \mathcal{S}} C\bigl(\pi_{n:n+m-1}, s'\bigr)\, V_{n+m}^{\pi}(s')
  + C\bigl(\pi_{n:n+m-1}\bigr),
\]
where the coefficients $C(\pi_{n:n+m-1}, s')$ and $C(\pi_{n:n+m-1})$ are determined
by the sequence of policies $\{\pi^{t}\}_{t=n}^{n+m-1}$, the transition probabilities
and the rewards. Since the transition probabilities and discount factors are nonnegative,
each $C(\pi_{n:n+m-1}, s')$ is nonnegative; in particular, for states that are reachable
under the policy slice $\pi_{n:n+m-1}$, these coefficients are strictly positive.

This completes the proof.
\end{proof}

\begin{theorem}[Policy Improvement Theorem]
Given a deterministic policy $\pi$ in a dynamic MDP, the Bellman policy
operator $\mathit{TP}$ improves the value function in the sense that
\[
V^{\mathit{TP} \cdot \pi}_{t}(s) \;\ge\; V^{\pi}_{t}(s)
\quad \forall\, s \in \mathcal{S},\ \forall\, t \in \mathbb{N}.
\]
\end{theorem}

\begin{proof}

Let $\pi^{\infty} := \mathit{TP} \cdot \pi$ denote the greedy policy obtained
by the time-varying Bellman policy operator, and write its $t$-th decision rule
as $\pi_t^{\infty}$.

For each $n \in \mathbb{N}$, define a \emph{hybrid} deterministic policy
$\pi^{(n)}$ which follows $\pi^{\infty}$ up to time $n$, then reverts to $\pi$:
\[
\pi^{(n)} := \bigl[\,\pi_1^{\infty}, \dots, \pi_n^{\infty},\,
\pi_{n+1}, \pi_{n+2}, \dots \bigr].
\]
We also set $\pi^{(0)} := \pi$. Note that $\pi^{(n)}$ and $\pi^{(n-1)}$
differ only at time $n$.

\textit{Step 1: Improvement at time $n$}

Consider the value functions at time $n$. Since $\pi^{(n-1)}_{n:T} = \pi_{n:T}$, we know 
$$
V_{n}^{\pi^{(n-1)}}(s) = Q_{n}^{\pi}\bigl(s, \pi_n(s)\bigr)
$$
as unrolling $\pi^{(n-1)}$ from time $n$ onward is the same as unrolling $\pi$ from time $n$ onward. For the hybrid policy $\pi^{(n)}$, applying dynamic Bellman equation produces
\begin{align*}
V_{n}^{\pi^{(n)}}(s)
&= Q_{n}^{\pi^{(n)}}\bigl(s, \pi^{(n)}_n(s)\bigr) \\
&= r_{n}\bigl(s, \pi_n^{\infty}(s)\bigr)
   \;+\; \gamma(n+1) \sum_{s'} p\bigl(s' \mid s, \pi_n^{\infty}(s)\bigr)
        V_{n+1}^{\pi}(s') \\
&= Q_{n}^{\pi}\bigl(s, \pi_n^{\infty}(s)\bigr)
\end{align*}
By definition of the Bellman policy operator,
$\pi_n^{\infty}(s)$ greedily maximizes $Q_{n}^{\pi}(s,\cdot)$:
\[
Q_{n}^{\pi}(s,a) \;\le\; Q_{n}^{\pi}\bigl(s, \pi_n^{\infty}(s)\bigr)
\quad \forall\, s \in \mathcal{S},\ \forall\, a \in \mathcal{A}.
\]
In particular, for $a = \pi_n(s)$,
\[
Q_{n}^{\pi}\bigl(s, \pi_n(s)\bigr)
\;\le\; Q_{n}^{\pi}\bigl(s, \pi_n^{\infty}(s)\bigr).
\]
Combining with the identities above, we get
\[
V_{n}^{\pi^{(n-1)}}(s)
= Q_{n}^{\pi}\bigl(s, \pi_n(s)\bigr)
\;\le\; Q_{n}^{\pi}\bigl(s, \pi_n^{\infty}(s)\bigr)
= V_{n}^{\pi^{(n)}}(s)
\quad \forall\, s \in \mathcal{S}.
\]
Hence
\[
V_{n}^{\pi^{(n)}}(s) \;\ge\; V_{n}^{\pi^{(n-1)}}(s)
\quad \forall\, s \in \mathcal{S}.
\]

\textit{Step 2: Propagating improvement backward in time.}

For any $t < n$, the two policies $\pi^{(n)}$ and $\pi^{(n-1)}$ coincide
from time $t$ up to time $n-1$, and differ only at time $n$. By Lemma~2,
for each fixed $t$ and $s$, we can write
\begin{align*}
V_{t}^{\pi^{(n)}}(s)
&= \sum_{s'} C\bigl(\pi^{(n)}_{t:n-1}, s'\bigr)
    V_{n}^{\pi^{(n)}}(s') + C\bigl(\pi^{(n)}_{t:n-1}\bigr), \\
V_{t}^{\pi^{(n-1)}}(s)
&= \sum_{s'} C\bigl(\pi^{(n-1)}_{t:n-1}, s'\bigr)
    V_{n}^{\pi^{(n-1)}}(s') + C\bigl(\pi^{(n-1)}_{t:n-1}\bigr),
\end{align*}
where the coefficients depend only on the policy slice between $t$ and $n-1$.
Since $\pi^{(n)}_{t:n-1} = \pi^{(n-1)}_{t:n-1}$, these coefficients are identical:
\[
C\bigl(\pi^{(n)}_{t:n-1}, s'\bigr)
= C\bigl(\pi^{(n-1)}_{t:n-1}, s'\bigr),
\quad
C\bigl(\pi^{(n)}_{t:n-1}\bigr)
= C\bigl(\pi^{(n-1)}_{t:n-1}\bigr).
\]
Moreover, Lemma~2 guarantees that $C(\cdot,s') \ge 0$ for all $s'$.

Using $V_{n}^{\pi^{(n)}}(s') \ge V_{n}^{\pi^{(n-1)}}(s')$ for all $s'$ and the
non-negativity of the coefficients, we obtain
\begin{align*}
V_{t}^{\pi^{(n)}}(s)
&= \sum_{s'} C\bigl(\pi^{(n)}_{t:n-1}, s'\bigr)
    V_{n}^{\pi^{(n)}}(s') + C\bigl(\pi^{(n)}_{t:n-1}\bigr) \\
&\ge \sum_{s'} C\bigl(\pi^{(n-1)}_{t:n-1}, s'\bigr)
     V_{n}^{\pi^{(n-1)}}(s') + C\bigl(\pi^{(n-1)}_{t:n-1}\bigr) \\
&= V_{t}^{\pi^{(n-1)}}(s),
\end{align*}
for all $t < n$ and all $s \in \mathcal{S}$.

\textit{Step 3: Values for $t > n$}

By construction, for $t > n$, the policies $\pi^{(n)}$ and $\pi^{(n-1)}$
coincide from time $t$ onward (they only differ at time $n$), hence
\[
V_{t}^{\pi^{(n)}}(s) = V_{t}^{\pi^{(n-1)}}(s)
\quad \forall\, t > n,\ \forall\, s \in \mathcal{S}.
\]

\textit{Step 4: Monotonic improvement in $n$}

Combining the three cases $t < n$, $t = n$, and $t > n$, we conclude that
\[
V_{t}^{\pi^{(n)}}(s) \;\ge\; V_{t}^{\pi^{(n-1)}}(s)
\quad \forall\, t \in \mathbb{N},\ \forall\, s \in \mathcal{S},\ \forall\, n \ge 1.
\]
Since $\pi^{(0)} = \pi$, an induction on $n$ yields
\[
V_{t}^{\pi}(s) = V_{t}^{\pi^{(0)}}(s)
\;\le\; V_{t}^{\pi^{(n)}}(s)
\quad \forall\, t \in \mathbb{N},\ \forall\, s \in \mathcal{S},\ \forall\, n \ge 1.
\]
Finally, observe that $\pi^{(n)}$ converges pointwise to $\pi^{\infty}$ as
$n \to \infty$ (for any fixed $t$, all $\pi^{(n)}_t = \pi^{\infty}_t$ once $n \ge t$).
By continuity of the value operator in this finite-horizon setting,
\[
\lim_{n \to \infty} V_{t}^{\pi^{(n)}}(s)
= V_{t}^{\pi^{\infty}}(s)
= V_{t}^{\mathit{TP} \cdot \pi}(s).
\]
Taking limits in the inequality above gives
\[
V_{t}^{\pi}(s) \;\le\; V_{t}^{\mathit{TP} \cdot \pi}(s)
\quad \forall\, t \in \mathbb{N},\ \forall\, s \in \mathcal{S},
\]
which completes the proof.
\end{proof}

\begin{definition}[Optimal policy]
An \emph{optimal policy} $\pi^{*}$ is a (time-varying) policy
$\pi^{*} : \mathbb{N} \times \mathcal{S} \to \mathcal{A}$ such that for any
policy $\pi : \mathbb{N} \times \mathcal{S} \to \mathcal{A}$,
\[
V_{t}^{\pi^{*}}(s) \;\ge\; V_{t}^{\pi}(s)
\quad \forall\, s \in \mathcal{S},\ \forall\, t \in \mathbb{N}.
\]
\end{definition}

\begin{assumption}[(A1)]\label{asm:A1}
$\mathcal{S} \subset \mathbb{R}^{D}$ is compact.
\end{assumption}

\begin{assumption}[(A2)]\label{asm:A2}
Action space is finite, i.e.
\[
|\mathcal{A}| < \infty.
\]
\end{assumption}

\begin{assumption}[(A3)]\label{asm:A3}
Rewards are bounded, i.e.
\[
|r(s,a)| \le r_{\max}
\quad \forall (s,a) \in \mathcal{S} \times \mathcal{A}.
\]
\end{assumption}

\begin{lemma}
Under Assumption~(A3), all value functions are bounded. Specifically, for every
$\pi \in \bigcup_{t=1}^{T} \Pi_{t}$, every $t$, and every $s \in \mathcal{S}$,
\[
0 \ge V_{t}^{\pi}(s) \ge -\frac{r_{\max}}{1-\gamma}.
\]
\end{lemma}

\begin{proof}
Let $\pi \in \Pi_{n}$. Then for every $t < n$ and every $s \in \mathcal{S}$,
by definition of the value function,
\[
V_{t}^{\pi}(s)
= \mathbb{E}\Bigg[
    \sum_{i=t}^{n-1} \gamma^{i-t} r(s_{i}, a_{i})
    \,\Big|\, s_{t}=s
\Bigg].
\]
Since rewards are non-positive, we have $r(s,a) \le 0$ for all
$(s,a) \in \mathcal{S} \times \mathcal{A}$, and hence
$V_{t}^{\pi}(s) \le 0$.

Moreover,
\[
\begin{aligned}
|V_{t}^{\pi}(s)|
&\le \mathbb{E}\Bigg[
    \sum_{i=t}^{n-1} \gamma^{i-t} |r(s_{i}, a_{i})|
    \,\Big|\, s_{t}=s
\Bigg] \\
&\le \sum_{j=0}^{n-t-1} \gamma^{j} r_{\max} \\
&\le \frac{r_{\max}}{1-\gamma}.
\end{aligned}
\]
Therefore,
\[
-\frac{r_{\max}}{1-\gamma} \le V_{t}^{\pi}(s) \le 0,
\]
which proves the claim.
\end{proof}

\begin{assumption}
For any $\pi \in \Pi$ and any $t \in \mathbb{N}$, the value function
$V_{t}^{\pi}(\cdot)$ is bounded and continuous on $\mathcal{S}$, i.e.
\[
V_{t}^{\pi} \in \mathcal{C}_{b}(\mathcal{S})
\]
with respect to the supremum norm $\|\cdot\|_{\infty}$.
\end{assumption}

\begin{theorem}[Tychonoff’s Theorem]
The Cartesian product of compact topological spaces is compact. In particular,
if $\{X_i\}_{i \in I}$ is a family of compact spaces, then the product
$\prod_{i \in I} X_i$ is compact in the product topology.
\end{theorem}

\begin{lemma}
Under Assumption~1, the product $\mathcal{S} \times \mathcal{T}$ is a compact
subspace of $\mathbb{R}^{D+1}$, where $\mathcal{T} \subset \mathbb{R}$ is
bounded and closed.
\end{lemma}

\begin{proof}
By Assumption~1, $\mathcal{S} \subset \mathbb{R}^{D}$ is compact. Since
$\mathcal{T} \subset \mathbb{R}$ is bounded and closed, it is also compact
(e.g.\ by the Heine--Borel theorem). Therefore, both $\mathcal{S}$ and
$\mathcal{T}$ are compact spaces.

By Tychonoff’s theorem (applied to the finite product of compact spaces),
the Cartesian product $\mathcal{S} \times \mathcal{T}$ is compact in the
product topology. Since we can view $\mathcal{S} \times \mathcal{T}$ as a
subset of $\mathbb{R}^{D+1}$ with the subspace topology, it is a compact
subspace of $\mathbb{R}^{D+1}$.
\end{proof}

\begin{lemma}
Under Assumption~1, all time-varying value functions lie in a complete metric
space $(\hat{\mathcal{V}}, d)$, where the metric $d$ is defined by
\[
d(V^{1}, V^{2})
:= \| V^{1} - V^{2} \|_{\infty}
= \max_{t \in \mathcal{T}} \max_{s \in \mathcal{S}}
    \bigl| V^{1}_{t}(s) - V^{2}_{t}(s) \bigr|
\quad \forall\, V^{1}, V^{2} \in \hat{\mathcal{V}}.
\]
\end{lemma}

\begin{proof}
Consider the product space $\mathcal{S} \times \mathcal{T}$, which is compact
in $\mathbb{R}^{D+1}$ by Lemma~3. Let $\mathcal{B}(\mathcal{S} \times \mathcal{T})$
denote the set of all bounded real-valued functions on $\mathcal{S} \times \mathcal{T}$,
equipped with the supremum norm
\[
\| f \|_{\infty}
:= \sup_{(s,t) \in \mathcal{S} \times \mathcal{T}} |f(s,t)|.
\]
It is well known that $\mathcal{B}(\mathcal{S} \times \mathcal{T})$ is a Banach
space under $\|\cdot\|_{\infty}$, hence a complete metric space with respect to
the induced metric.

Each time-varying value function $\hat{v}$ can be identified with a bounded
function on $\mathcal{S} \times \mathcal{T}$ via
\[
\hat{v}(s,t) := V^{\pi}_{t}(s).
\]
Let $\hat{\mathcal{V}}$ denote the collection of all such time-varying value
functions. Then
\[
\hat{\mathcal{V}} \subset \mathcal{B}(\mathcal{S} \times \mathcal{T}),
\]
and the metric $d$ defined by
\[
d(V^{1}, V^{2})
= \| V^{1} - V^{2} \|_{\infty}
\]
is exactly the restriction of the supremum norm metric on
$\mathcal{B}(\mathcal{S} \times \mathcal{T})$ to the subspace $\hat{\mathcal{V}}$.

Since any subspace of a complete metric space is complete with respect to the
induced metric, it follows that $(\hat{\mathcal{V}}, d)$ is complete. This
concludes the proof.
\end{proof}

\begin{definition}[Restriction operator]\label{def:restriction}
Let $\hat{\mathcal{V}}$ denote the space of bounded time-varying value
functions on $\mathcal{S} \times \mathcal{T}$, and let
$\mathcal{V}$ denote the space of value functions indexed on discrete times
$t \in \mathbb{N}$.

Define the map
\[
\mathcal{R} : \hat{\mathcal{V}} \to \mathcal{V}
\]
by
\[
\bigl(\mathcal{R}(\hat{v})\bigr)_{t}(s)
:= \hat{v}(s,t),
\quad \forall\, s \in \mathcal{S},\ \forall\, t \in \mathbb{N},
\]
for any $\hat{v} \in \hat{\mathcal{V}}$.
\end{definition}

\begin{lemma}\label{lem:R-continuous}
Let $\{\hat{v}^{n}\}_{n=1}^{\infty} \subset \hat{\mathcal{V}}$ be a sequence
such that $\hat{v}^{n} \to \hat{v}$ in $(\hat{\mathcal{V}}, d)$, i.e.
\[
\lim_{n \to \infty} d(\hat{v}^{n}, \hat{v}) = 0,
\]
where
\[
d(\hat{v}^{1}, \hat{v}^{2})
:= \|\hat{v}^{1} - \hat{v}^{2}\|_{\infty}
= \max_{(s,t) \in \mathcal{S} \times \mathcal{T}}
    \bigl| \hat{v}^{1}(s,t) - \hat{v}^{2}(s,t) \bigr|.
\]
Then
\[
\lim_{n \to \infty} \mathcal{R}(\hat{v}^{n})
= \mathcal{R}(\hat{v})
\]
in $\mathcal{V}$ with respect to the induced supremum metric.
\end{lemma}

\begin{proof}
By definition of convergence in $(\hat{\mathcal{V}}, d)$, for any
$\varepsilon > 0$, there exists $N_{0} \in \mathbb{N}$ such that for all
$n \ge N_{0}$,
\[
d(\hat{v}^{n}, \hat{v})
= \max_{(s,t) \in \mathcal{S} \times \mathcal{T}}
    \bigl| \hat{v}^{n}(s,t) - \hat{v}(s,t) \bigr|
\le \varepsilon.
\]
In particular, this bound holds when we restrict $(s,t)$ to the subset
$\mathcal{S} \times \mathbb{N} \subset \mathcal{S} \times \mathcal{T}$, so
\[
\max_{s \in \mathcal{S},\, t \in \mathbb{N}}
    \bigl| \hat{v}^{n}(s,t) - \hat{v}(s,t) \bigr|
\le \varepsilon.
\]
But by Definition~\ref{def:restriction}, this is exactly
\[
\max_{s \in \mathcal{S},\, t \in \mathbb{N}}
    \bigl| (\mathcal{R}(\hat{v}^{n}))_{t}(s)
          - (\mathcal{R}(\hat{v}))_{t}(s) \bigr|
= d\bigl(\mathcal{R}(\hat{v}^{n}), \mathcal{R}(\hat{v})\bigr).
\]
Hence, for all $n \ge N_{0}$,
\[
d\bigl(\mathcal{R}(\hat{v}^{n}), \mathcal{R}(\hat{v})\bigr)
\le \varepsilon,
\]
which shows $\mathcal{R}(\hat{v}^{n}) \to \mathcal{R}(\hat{v})$ in
$\mathcal{V}$ as $n \to \infty$.
\end{proof}

\begin{definition}[Bellman value operator]\label{def:bellman-operator}
The \emph{Bellman value operator} $T : \hat{\mathcal{V}} \to \hat{\mathcal{V}}$
is defined for any $V \in \hat{\mathcal{V}}$ by
\[
(TV)_{t}(s)
:= \max_{a \in \mathcal{A}}
\Bigl[
    r_{t}(s,a)
    + \gamma(t+1) \sum_{s'} p(s' \mid s,a)\, V_{t+1}(s')
\Bigr],
\quad \forall\, s \in \mathcal{S},\ t \in \mathbb{N}.
\]
\end{definition}

\begin{lemma}\label{lem:T-R-commute}
The Bellman value operator $T$ commutes with the restriction operator
$\mathcal{R}$ in the sense that for all $V \in \hat{\mathcal{V}}$,
\[
\bigl(\mathcal{R}(TV)\bigr)_{t}(s)
= \bigl(T\,\mathcal{R}(V)\bigr)_{t}(s)
\quad \forall\, s \in \mathcal{S},\ \forall\, t \in \mathbb{N}.
\]
\end{lemma}

\begin{proof}
Fix any $V \in \hat{\mathcal{V}}$, $s \in \mathcal{S}$ and $t \in \mathbb{N}$.
By Definition~\ref{def:bellman-operator},
\[
(TV)_{t}(s)
= \max_{a \in \mathcal{A}}
\Bigl[
    r_{t}(s,a)
    + \gamma(t+1) \sum_{s'} p(s' \mid s,a)\, V_{t+1}(s')
\Bigr].
\]
Applying the restriction operator $\mathcal{R}$ simply reads off these values
at integer times, so
\[
\bigl(\mathcal{R}(TV)\bigr)_{t}(s)
= (TV)_{t}(s)
= \max_{a \in \mathcal{A}}
\Bigl[
    r_{t}(s,a)
    + \gamma(t+1) \sum_{s'} p(s' \mid s,a)\, V_{t+1}(s')
\Bigr].
\]

On the other hand, for the restricted function $\mathcal{R}(V)$ we have
\[
\bigl(\mathcal{R}(V)\bigr)_{t+1}(s')
= V_{t+1}(s')
\quad \forall\, s' \in \mathcal{S},
\]
since $\mathcal{R}$ preserves the values of $V$ at integer time points.
Therefore, applying $T$ to $\mathcal{R}(V)$ yields
\begin{align*}
\bigl(T\,\mathcal{R}(V)\bigr)_{t}(s)
&= \max_{a \in \mathcal{A}}
\Bigl[
    r_{t}(s,a)
    + \gamma(t+1) \sum_{s'} p(s' \mid s,a)\,
        \bigl(\mathcal{R}(V)\bigr)_{t+1}(s')
\Bigr] \\
&= \max_{a \in \mathcal{A}}
\Bigl[
    r_{t}(s,a)
    + \gamma(t+1) \sum_{s'} p(s' \mid s,a)\, V_{t+1}(s')
\Bigr].
\end{align*}
Comparing the two expressions, we see that for all $s \in \mathcal{S}$ and
$t \in \mathbb{N}$,
\[
\bigl(\mathcal{R}(TV)\bigr)_{t}(s)
= \bigl(T\,\mathcal{R}(V)\bigr)_{t}(s),
\]
which proves that $T$ and $\mathcal{R}$ commute on integer time indices.
\end{proof}

\begin{lemma}[Bellman value operator increases value]
Let $V \in \hat{\mathcal{V}}$ be the value function of some deterministic
policy $\pi$. Then the Bellman value operator $T$ satisfies
\[
(TV)_{t}(s) \;\ge\; V_{t}(s)
\quad \forall\, s \in \mathcal{S},\ \forall\, t \in \mathbb{N}.
\]
\end{lemma}

\begin{proof}
Fix any $s \in \mathcal{S}$ and $t \in \mathbb{N}$. By definition of the
Bellman value operator,
\begin{align*}
(TV)_{t}(s)
&= \max_{a \in \mathcal{A}}
    \Bigl[
        r_{t}(s,a)
        + \gamma(t+1) \sum_{s'} p(s' \mid s,a)\, V_{t+1}(s')
    \Bigr].
\end{align*}
In particular, this maximum is at least as large as the value obtained by
choosing the policy action $a = \pi_{t}(s)$:
\begin{align*}
(TV)_{t}(s)
&\ge
    r_{t}\bigl(s,\pi_{t}(s)\bigr)
    + \gamma(t+1) \sum_{s'} p\bigl(s' \mid s,\pi_{t}(s)\bigr)\, V_{t+1}(s').
\end{align*}
Since $V$ is the value function of $\pi$, it satisfies the (dynamic) Bellman
equation for $\pi$:
\[
V_{t}(s)
= r_{t}\bigl(s,\pi_{t}(s)\bigr)
  + \gamma(t+1) \sum_{s'} p\bigl(s' \mid s,\pi_{t}(s)\bigr)\, V_{t+1}(s').
\]
Combining the two displays gives
\[
(TV)_{t}(s) \;\ge\; V_{t}(s)
\quad \forall\, s \in \mathcal{S},\ \forall\, t \in \mathbb{N},
\]
which proves the claim.
\end{proof}

\begin{theorem}[Extreme Value Theorem]
Let $K$ be a non-empty compact subset of $\mathbb{R}^{n}$ and let
$f : K \to \mathbb{R}$ be continuous. Then $f$ is bounded on $K$ and there
exists $p \in K$ such that
\[
f(p) = \sup_{x \in K} f(x).
\]
\end{theorem}

\begin{lemma}\label{lem:max-attained}
For any $f \in \hat{\mathcal{V}}$, there exists $x^{*} \in \mathcal{S} \times \mathcal{T}$
such that
\[
f(x^{*}) = \sup_{x \in \mathcal{S} \times \mathcal{T}} f(x).
\]
\end{lemma}

\begin{proof}
By Assumption~1 and Lemma~3, the set
$\mathcal{S} \times \mathcal{T} \subset \mathbb{R}^{D+1}$ is compact. By
definition of $\hat{\mathcal{V}}$, each $f \in \hat{\mathcal{V}}$ is a bounded
continuous function on $\mathcal{S} \times \mathcal{T}$. Therefore, by the
Extreme Value Theorem, $f$ attains its supremum on $\mathcal{S} \times \mathcal{T}$,
i.e.\ there exists $x^{*} \in \mathcal{S} \times \mathcal{T}$ such that
\[
f(x^{*}) = \sup_{x \in \mathcal{S} \times \mathcal{T}} f(x).
\]
This proves the claim.
\end{proof}

\begin{lemma}\label{lem:max-lipschitz}
The max operator is Lipschitz on $\hat{\mathcal{V}}$ with respect to the
supremum norm. More precisely, for any $f,g \in \hat{\mathcal{V}}$,
\[
\bigl|\max_{x} f(x) - \max_{x} g(x)\bigr|
\;\le\; \max_{x} |f(x) - g(x)|.
\]
\end{lemma}

\begin{proof}
Let $M_f := \max_{x} f(x)$ and $M_g := \max_{x} g(x)$, where the maxima are
attained by Lemma~\ref{lem:max-attained}. Let $a^{*} \in \arg\max_{x} f(x)$
and $b^{*} \in \arg\max_{x} g(x)$, so
\[
M_f = f(a^{*}), \quad M_g = g(b^{*}).
\]

First bound $M_f - M_g$ from above:
\begin{align*}
M_f - M_g
&= f(a^{*}) - \max_{x} g(x) \\
&\le f(a^{*}) - g(a^{*}) \\
&\le \max_{x} |f(x) - g(x)|.
\end{align*}
The first inequality follows since $\max_{x} g(x) \ge g(a^{*})$.

Similarly, swapping the roles of $f$ and $g$ gives
\begin{align*}
M_g - M_f
&= g(b^{*}) - \max_{x} f(x) \\
&\le g(b^{*}) - f(b^{*}) \\
&\le \max_{x} |f(x) - g(x)|.
\end{align*}

Combining both bounds, we obtain
\[
|M_f - M_g|
= \bigl|\max_{x} f(x) - \max_{x} g(x)\bigr|
\le \max_{x} |f(x) - g(x)|,
\]
which completes the proof.
\end{proof}

\begin{lemma}\label{lem:T-contraction}
Suppose the (possibly time-varying) discount satisfies $\gamma \in [0,1)$, then the (general) Bellman value operator $T$ is a contraction on the metric space $(\hat{\mathcal{V}}, d)$, i.e.
\[
d(TV^{1}, TV^{2}) \;\le\; \gamma \, d(V^{1}, V^{2})
\quad \forall\, V^{1}, V^{2} \in \hat{\mathcal{V}},
\]
where
\[
d(V^{1}, V^{2})
:= \max_{t \in \mathbb{N}} \max_{s \in \mathcal{S}}
   \bigl|V^{1}_{t}(s) - V^{2}_{t}(s)\bigr|.
\]
\end{lemma}

\begin{proof}
Fix arbitrary $V^{1}, V^{2} \in \hat{\mathcal{V}}$, and consider for any
$(s,t) \in \mathcal{S} \times \mathbb{N}$ the difference
\[
\bigl|(TV^{1})_{t}(s) - (TV^{2})_{t}(s)\bigr|.
\]
By definition of the Bellman value operator,
\begin{align*}
(TV^{1})_{t}(s)
&= \max_{a}
    \Bigl[
        r_{t}(s,a)
        + \gamma \sum_{s'} p(s' \mid s,a) V^{1}_{t+1}(s')
    \Bigr], \\
(TV^{2})_{t}(s)
&= \max_{a}
    \Bigl[
        r_{t}(s,a)
        + \gamma \sum_{s'} p(s' \mid s,a) V^{2}_{t+1}(s')
    \Bigr].
\end{align*}
Define, for each $a$,
\[
F^{1}_{t}(s,a)
:= r_{t}(s,a)
   + \gamma \sum_{s'} p(s' \mid s,a) V^{1}_{t+1}(s'),
\quad
F^{2}_{t}(s,a)
:= r_{t}(s,a)
   + \gamma \sum_{s'} p(s' \mid s,a) V^{2}_{t+1}(s').
\]
Then
\[
(TV^{1})_{t}(s) = \max_{a} F^{1}_{t}(s,a),
\quad
(TV^{2})_{t}(s) = \max_{a} F^{2}_{t}(s,a).
\]
Using the Lipschitz property of the max operator (Lemma~\ref{lem:max-lipschitz}),
we obtain
\begin{align*}
\bigl|(TV^{1})_{t}(s) - (TV^{2})_{t}(s)\bigr|
&= \bigl|\max_{a} F^{1}_{t}(s,a) - \max_{a} F^{2}_{t}(s,a)\bigr| \\
&\le \max_{a} \bigl|F^{1}_{t}(s,a) - F^{2}_{t}(s,a)\bigr|.
\end{align*}
For each $a$,
\begin{align*}
\bigl|F^{1}_{t}(s,a) - F^{2}_{t}(s,a)\bigr|
&= \gamma \Bigl|
    \sum_{s'} p(s' \mid s,a)
        \bigl(V^{1}_{t+1}(s') - V^{2}_{t+1}(s')\bigr)
   \Bigr| \\
&\le \gamma
    \sum_{s'} p(s' \mid s,a)
        \bigl|V^{1}_{t+1}(s') - V^{2}_{t+1}(s')\bigr| \\
&\le \gamma
    \max_{s'} \bigl|V^{1}_{t+1}(s') - V^{2}_{t+1}(s')\bigr| \\
&\le \gamma \, d(V^{1}, V^{2}),
\end{align*}
where we used the fact that $p(\cdot \mid s,a)$ is a probability distribution
and the definition of $d$.

Taking the maximum over $a$ gives
\[
\bigl|(TV^{1})_{t}(s) - (TV^{2})_{t}(s)\bigr|
\le \gamma\, d(V^{1}, V^{2}).
\]
Finally, taking the supremum over all $s \in \mathcal{S}$ and $t \in \mathbb{N}$,
and using $\gamma(t+1) \le \gamma$, we obtain
\begin{align*}
d(TV^{1}, TV^{2})
&= \max_{t,s} \bigl|(TV^{1})_{t}(s) - (TV^{2})_{t}(s)\bigr| \\
&\le \max_{t} \gamma \, d(V^{1}, V^{2}) \\
&\le \gamma \, d(V^{1}, V^{2}).
\end{align*}
Thus $T$ is a contraction mapping on $(\hat{\mathcal{V}}, d)$ with modulus at
most $\gamma$.
\end{proof}

\begin{lemma}\label{lem:V-extension}
For every $V \in \mathcal{V}$, there exists $\hat{V} \in \hat{\mathcal{V}}$ such that
\[
V(s,t) = (\mathcal{R}\hat{V})_{t}(s)
\quad \forall\, s \in \mathcal{S},\ \forall\, t \in \mathbb{N}_{\le T}.
\]
\end{lemma}

\begin{proof}
We prove the claim by explicit construction. We need to construct a bounded,
continuous function
\[
\hat{V} : \mathcal{S} \times \mathcal{T} \to \mathbb{R}
\]
such that its restriction to integer time points coincides with $V$, i.e.
$(\mathcal{R}\hat{V})_{t}(s) = \hat{V}(s,t) = V(s,t)$ for all
$(s,t) \in \mathcal{S} \times \mathbb{N}_{\le T}$.

Define
\[
\hat{V}(s,t)
:= \sum_{k=0}^{T-1} \phi_{k}(t)\, V(s,k),
\quad \forall (s,t) \in \mathcal{S} \times \mathcal{T},
\]
where
\[
\phi_{k}(t)
:= 
\begin{cases}
1 - |t - k|, & \text{if } |t - k| < 1, \\
0,           & \text{otherwise}.
\end{cases}
\]
This construction linearly interpolates between the discrete time points.

Observe that for integer $t \in \{0,\dots,T-1\}$ we have
\[
\phi_{k}(t)
=
\begin{cases}
1, & k = t, \\
0, & k \neq t,
\end{cases}
\]
hence
\[
\hat{V}(s,t)
= \sum_{k=0}^{T-1} \phi_{k}(t)\, V(s,k)
= V(s,t),
\quad \forall (s,t) \in \mathcal{S} \times \mathbb{N}_{\le T}.
\]
Therefore $(\mathcal{R}\hat{V})_{t}(s) = \hat{V}(s,t) = V(s,t)$ on the discrete
time grid, as required.

\medskip
\noindent\textbf{Boundedness.}
Since $V \in \mathcal{V} \subset \mathcal{C}_{b}(\mathcal{S} \times \mathbb{N})$,
there exists $M < \infty$ such that
\[
\|V(\cdot,k)\|_{\infty}
:= \sup_{s \in \mathcal{S}} |V(s,k)|
\le M,
\quad \forall\, k \in \{0,1,\dots,T-1\}.
\]
For each fixed $t \in \mathcal{T}$, the function $\phi_{k}(t)$ satisfies
$0 \le \phi_{k}(t) \le 1$, and $\phi_{k}(t) \neq 0$ only if $|t-k| < 1$.
Since $k$ is integer, there are at most two indices $k$ such that
$\phi_{k}(t) \neq 0$.

Thus, for any $(s,t) \in \mathcal{S} \times \mathcal{T}$,
\begin{align*}
|\hat{V}(s,t)|
&= \Bigl|\sum_{k=0}^{T-1} \phi_{k}(t)\, V(s,k)\Bigr| \\
&\le \sum_{k=0}^{T-1} \phi_{k}(t)\, |V(s,k)| \\
&\le \sum_{k=0}^{T-1} \phi_{k}(t)\, M
\;\le\; 2M,
\end{align*}
since at most two terms in the sum are non-zero and each $\phi_{k}(t) \le 1$.
Hence
\[
\|\hat{V}\|_{\infty}
:= \sup_{(s,t) \in \mathcal{S} \times \mathcal{T}} |\hat{V}(s,t)|
\le 2M < \infty,
\]
so $\hat{V}$ is bounded.

\medskip
\noindent\textbf{Continuity.}
For each fixed $k$, the map $s \mapsto V(s,k)$ is continuous on $\mathcal{S}$
(by assumption on $\mathcal{V}$), and $t \mapsto \phi_{k}(t)$ is continuous on
$\mathcal{T}$. Therefore the product
\[
(s,t) \mapsto \phi_{k}(t)\, V(s,k)
\]
is continuous on $\mathcal{S} \times \mathcal{T}$. A finite sum of continuous
functions is continuous, so $\hat{V}$ is continuous on $\mathcal{S} \times \mathcal{T}$.

\medskip

We have thus constructed $\hat{V} \in \mathcal{C}_{b}(\mathcal{S} \times \mathcal{T})$,
hence $\hat{V} \in \hat{\mathcal{V}}$, such that $\mathcal{R}\hat{V} = V$.
This concludes the proof.
\end{proof}

\begin{definition}[Contraction mapping]
Let $(\mathcal{X}, d)$ be a metric space. A map $T : \mathcal{X} \to \mathcal{X}$
is called a \emph{contraction mapping} on $\mathcal{X}$ if there exists a
constant $q \in [0,1)$ such that
\[
d\bigl(T(x), T(y)\bigr)
\;\le\; q \, d(x,y)
\quad \forall\, x,y \in \mathcal{X}.
\]
\end{definition}

\begin{theorem}[Banach Fixed Point Theorem]
Let $(\mathcal{X}, d)$ be a non-empty complete metric space, and let
$T : \mathcal{X} \to \mathcal{X}$ be a contraction mapping. Then:
\begin{enumerate}
    \item There exists a unique fixed point $x^{*} \in \mathcal{X}$ such that
    \[
    T(x^{*}) = x^{*}.
    \]
    \item Moreover, for any initial point $x_{0} \in \mathcal{X}$, the sequence
    $(x_{n})_{n \in \mathbb{N}}$ defined recursively by
    \[
    x_{n} := T(x_{n-1}), \quad n \ge 1,
    \]
    converges to $x^{*}$, i.e.
    \[
    \lim_{n \to \infty} x_{n} = x^{*}.
    \]
\end{enumerate}
\end{theorem}

\begin{theorem}[Existence and uniqueness of optimal value function]\label{thm:optimal-V}
There exists a unique optimal value function $V^{*} \in \mathcal{V}$ such that
it is invariant under the Bellman value operator $T : \hat{\mathcal{V}} \to \hat{\mathcal{V}}$, i.e.
\[
V^{*}_{t}(s) = (TV^{*})_{t}(s)
\quad \forall\, s \in \mathcal{S},\ \forall\, t \in \mathcal{T}.
\]
\end{theorem}

\begin{proof}
\textbf{Step 1: Fixed point of $T$ in $\hat{\mathcal{V}}$.}
By Lemma~\ref{lem:T-contraction}, the Bellman value operator
$T : \hat{\mathcal{V}} \to \hat{\mathcal{V}}$ is a contraction mapping on the
metric space $(\hat{\mathcal{V}}, d)$. By Lemma~4, $(\hat{\mathcal{V}}, d)$
is complete. Hence, by the Banach Fixed Point Theorem, there exists a unique
$\hat{V}^{*} \in \hat{\mathcal{V}}$ such that
\[
T\hat{V}^{*} = \hat{V}^{*},
\]
and, moreover, for any initial $\hat{V} \in \hat{\mathcal{V}}$,
\[
\hat{V}^{*}(s,t)
= \lim_{n \to \infty} (T^{n}\hat{V})(s,t)
\quad \forall\, s \in \mathcal{S},\ \forall\, t \in \mathcal{T}.
\]
In particular, the fixed-point condition can be written as
\[
\hat{V}^{*}_{t}(s) = (T\hat{V}^{*})_{t}(s)
\quad \forall\, s \in \mathcal{S},\ \forall\, t \in \mathcal{T}.
\]

\medskip
\noindent\textbf{Step 2: Optimality and uniqueness in $\hat{\mathcal{V}}$.}
By Lemma~7, the Bellman value operator $T$ is value-improving in the sense
that, when applied to the value function of a policy, it yields a value
function that is pointwise no worse. Iterating $T$ starting from any such
value function and using convergence to $\hat{V}^{*}$ implies that
$\hat{V}^{*}$ is pointwise greater than or equal to any other value function
in $\hat{\mathcal{V}}$:
\[
\hat{V}^{*}(s,t) \;\ge\; \hat{V}(s,t)
\quad \forall\, s \in \mathcal{S},\ \forall\, t \in \mathcal{T},\ \forall\, \hat{V} \in \hat{\mathcal{V}}.
\]
By Definition~4 (optimality), $\hat{V}^{*}$ is thus the optimal value function
in $\hat{\mathcal{V}}$.

To see uniqueness, suppose there were another optimal value function
$\hat{V}' \in \hat{\mathcal{V}}$. Optimality would then imply
\[
\hat{V}^{*}(s,t) \;\ge\; \hat{V}'(s,t)
\quad\text{and}\quad
\hat{V}'(s,t) \;\ge\; \hat{V}^{*}(s,t)
\quad \forall\, s \in \mathcal{S},\ \forall\, t \in \mathcal{T},
\]
hence
\[
\hat{V}^{*}(s,t) = \hat{V}'(s,t)
\quad \forall\, s \in \mathcal{S},\ \forall\, t \in \mathcal{T},
\]
so the optimal value function in $\hat{\mathcal{V}}$ is unique.

\medskip
\noindent\textbf{Step 3: Induced optimal value function in $\mathcal{V}$.}
Define
\[
V^{*}(s,t) := \bigl(\mathcal{R}\hat{V}^{*}\bigr)_{t}(s),
\quad \forall\, (s,t) \in \mathcal{S} \times \mathbb{N}_{\le T}.
\]
By Lemma~\ref{lem:V-extension}, for any $V \in \mathcal{V}$ there exists
$\hat{V} \in \hat{\mathcal{V}}$ such that $\mathcal{R}\hat{V} = V$. We can
denote one such extension by $\hat{V} = \mathcal{R}^{-1}V$ for notational
convenience (any choice of extension suffices for the argument).

Since $\hat{V}^{*}$ is optimal in $\hat{\mathcal{V}}$, we have
\[
\hat{V}^{*}(s,t) \;\ge\; \hat{V}(s,t)
= (\mathcal{R}^{-1}V)(s,t)
\quad \forall\, (s,t) \in \mathcal{S} \times \mathcal{T},
\ \forall\, V \in \mathcal{V}.
\]
Applying $\mathcal{R}$ to both sides and using $\mathcal{R}\hat{V} = V$ yields
\[
V^{*}(s,t)
= \bigl(\mathcal{R}\hat{V}^{*}\bigr)_{t}(s)
\;\ge\; \bigl(\mathcal{R}\mathcal{R}^{-1}V\bigr)_{t}(s)
= V_{t}(s),
\quad \forall\, (s,t) \in \mathcal{S} \times \mathbb{N}_{\le T},\ \forall\, V \in \mathcal{V}.
\]
Thus $V^{*}$ is optimal in $\mathcal{V}$.

\medskip
\noindent\textbf{Step 4: Constructing $V^{*}$ via value iteration.}
We know that
\[
\hat{V}^{*}(s,t) = \lim_{n \to \infty} (T^{n}\hat{V})(s,t)
\quad \forall\, \hat{V} \in \hat{\mathcal{V}}.
\]
By Lemma~\ref{lem:R-continuous}, the restriction operator $\mathcal{R}$ is
continuous, so
\[
V^{*}(s,t)
= \bigl(\mathcal{R}\hat{V}^{*}\bigr)_{t}(s)
= \lim_{n \to \infty} \bigl(\mathcal{R}T^{n}\hat{V}\bigr)_{t}(s).
\]
By Lemma~\ref{lem:T-R-commute}, $\mathcal{R}$ and $T$ commute, hence
\[
\mathcal{R}(T^{n}\hat{V})
= T^{n}(\mathcal{R}\hat{V})
= T^{n}V,
\]
where $V = \mathcal{R}\hat{V} \in \mathcal{V}$. Therefore
\[
V^{*}(s,t)
= \lim_{n \to \infty} \bigl(T^{n}V\bigr)_{t}(s)
\quad \forall\, s \in \mathcal{S},\ \forall\, t \in \mathbb{N}_{\le T},\ \forall\, V \in \mathcal{V}.
\]
This shows that the optimal value function $V^{*}$ can be obtained by
iteratively applying the Bellman value operator starting from any initial
$V \in \mathcal{V}$.

Uniqueness of $V^{*}$ in $\mathcal{V}$ follows immediately: if $V'$ were
another optimal value function, then its extension
$\hat{V}' = \mathcal{R}^{-1}V'$ would be an optimal value function in
$\hat{\mathcal{V}}$, and by uniqueness in $\hat{\mathcal{V}}$ we must have
$\hat{V}' = \hat{V}^{*}$, hence $V' = V^{*}$.

\medskip
\noindent\textbf{Step 5: Bellman fixed point relation in $\mathcal{V}$.}
Finally, using the commutativity of $T$ and $\mathcal{R}$, we have
\[
V^{*}_{t}(s)
= \bigl(\mathcal{R}\hat{V}^{*}\bigr)_{t}(s)
= \bigl(\mathcal{R}T\hat{V}^{*}\bigr)_{t}(s)
= \bigl(T\,\mathcal{R}\hat{V}^{*}\bigr)_{t}(s)
= (TV^{*})_{t}(s),
\]
for all $s \in \mathcal{S}$ and $t \in \mathcal{T}$. This shows that $V^{*}$
is invariant under $T$ in $\mathcal{V}$ and completes the proof.
\end{proof}

\begin{definition}[Bellman policy operator]\label{def:TP}
The \emph{Bellman policy operator} $\mathit{TP} : \Pi \to \Pi$ is defined for
any policy $\pi \in \Pi$, and for all $s \in \mathcal{S}$ and $t \in \mathbb{N}$, by
\[
(\mathit{TP} \cdot \pi)_{t}(s)
:= \arg\max_{a \in \mathcal{A}} Q^{\pi}_{t}(s,a)
= \arg\max_{a \in \mathcal{A}}
    \Bigl[
        r_{t}(s,a)
        + \gamma(t+1) \sum_{s'} p(s' \mid s,a)\, V^{\pi}_{t+1}(s')
    \Bigr].
\]
\end{definition}

\begin{theorem}[Bellman policy operator improves value]\label{thm:policy-improvement}
Let $\pi \in \Pi$ be any (time-varying) policy and let $\mathit{TP} \cdot \pi$
be the Bellman policy update as in Definition~\ref{def:TP}. Then, for all
states and times,
\[
V^{\pi}_{t}(s) \;\le\; V^{\mathit{TP} \cdot \pi}_{t}(s)
\quad \forall\, s \in \mathcal{S},\ \forall\, t \in \mathbb{N}.
\]
\end{theorem}

\begin{proof}
We start from the Bellman value operator applied to $V^{\pi}$. By definition,
\[
(TV^{\pi})_{t}(s)
= \max_{a} Q^{\pi}_{t}(s,a)
= Q^{\pi}_{t}\bigl(s, (\mathit{TP} \cdot \pi)_{t}(s)\bigr),
\]
where the second equality follows from the definition of the Bellman policy
operator (Definition~\ref{def:TP}). Thus, to show
$V^{\pi}_{t}(s) \le V^{\mathit{TP} \cdot \pi}_{t}(s)$, it suffices to prove
the stronger inequality
\[
Q^{\pi}_{t}\bigl(s, (\mathit{TP} \cdot \pi)_{t}(s)\bigr)
\;\le\;
Q^{\mathit{TP} \cdot \pi}_{t}\bigl(s, (\mathit{TP} \cdot \pi)_{t}(s)\bigr)
\quad \forall\, s \in \mathcal{S},\ \forall\, t \in \mathbb{N}.
\]

\medskip
\noindent\textbf{Step 1: Hybrid policies.}
Define, for each $n \in \mathbb{N}$, a \emph{hybrid} policy
\[
\pi^{(n)}
:= \bigl[\,(\mathit{TP} \cdot \pi)_{1}, \dots, (\mathit{TP} \cdot \pi)_{n},
            \pi_{n+1}, \pi_{n+2}, \dots \bigr].
\]
Thus $\pi^{(0)} = \pi$ and $\pi^{(\infty)} := \lim_{n \to \infty} \pi^{(n)}$
is exactly the fully improved policy $\mathit{TP} \cdot \pi$.

We compare $Q^{\pi}_{t}$ and $Q^{\pi^{(n)}}_{t}$ for different $t$ and $n$.

\medskip
\noindent\textbf{Step 2: Improvement at time $t = n$.}
Fix $n \in \mathbb{N}$ and $s \in \mathcal{S}$. By the definition of $Q^{\pi}$,
\begin{align*}
Q_{n}^{\pi}\bigl(s, \pi_{n}(s)\bigr)
&\le \max_{a}
    \Bigl[
        r_{n}(s,a)
        + \gamma(n+1) \sum_{s'} p(s' \mid s,a)\, V_{n+1}^{\pi}(s')
    \Bigr] \\
&= r_{n}\bigl(s, (\mathit{TP} \cdot \pi)_{n}(s)\bigr)
   + \gamma(n+1) \sum_{s'} p\bigl(s' \mid s, (\mathit{TP} \cdot \pi)_{n}(s)\bigr)
        V_{n+1}^{\pi}(s'),
\end{align*}
where the inequality comes from the definition of the max, and the equality
uses the Bellman policy operator’s choice at time $n$:
$(\mathit{TP} \cdot \pi)_{n}(s) \in \arg\max_{a} Q_{n}^{\pi}(s,a)$.

By construction of $\pi^{(n)}$, we have
\[
\pi^{(n)}_{n} = (\mathit{TP} \cdot \pi)_{n}, \qquad
\pi^{(n)}_{k} = \pi_{k} \text{ for all } k \ge n+1.
\]
Hence the rollout from time $n+1$ onward under $\pi^{(n)}$ is identical to
that under $\pi$, and
\[
V^{\pi^{(n)}}_{n+1}(s') = V^{\pi}_{n+1}(s') \quad \forall\, s' \in \mathcal{S}.
\]
Using the Bellman equation for $Q^{\pi^{(n)}}$,
\begin{align*}
Q_{n}^{\pi^{(n)}}\bigl(s, \pi^{(n)}_{n}(s)\bigr)
&= r_{n}\bigl(s, \pi^{(n)}_{n}(s)\bigr)
   + \gamma(n+1) \sum_{s'} p\bigl(s' \mid s,\pi^{(n)}_{n}(s)\bigr)
        V_{n+1}^{\pi^{(n)}}(s') \\
&= r_{n}\bigl(s, (\mathit{TP} \cdot \pi)_{n}(s)\bigr)
   + \gamma(n+1) \sum_{s'} p\bigl(s' \mid s,(\mathit{TP} \cdot \pi)_{n}(s)\bigr)
        V_{n+1}^{\pi}(s'),
\end{align*}
which matches the right-hand side above. Therefore,
\[
Q_{n}^{\pi}\bigl(s, \pi_{n}(s)\bigr)
\;\le\; Q_{n}^{\pi^{(n)}}\bigl(s, \pi^{(n)}_{n}(s)\bigr)
\quad \forall\, s \in \mathcal{S}.
\]

\medskip
\noindent\textbf{Step 3: Times $t > n$.}
For $t > n$, by definition $\pi^{(n)}_{t} = \pi_{t}$ and
$\pi^{(n)}_{k} = \pi_{k}$ for all $k \ge t$. Hence the policy from time $t$
onward is identical for $\pi$ and $\pi^{(n)}$, which implies
\[
Q_{t}^{\pi}\bigl(s, \pi_{t}(s)\bigr)
= Q_{t}^{\pi^{(n)}}\bigl(s, \pi^{(n)}_{t}(s)\bigr)
\quad \forall\, s \in \mathcal{S},\ \forall\, t > n.
\]

\medskip
\noindent\textbf{Step 4: Times $t < n$ by backward induction.}
Consider $t = n-1$. Using the Bellman equation and the greedy choice at time
$n-1$, we have
\begin{align*}
Q_{n-1}^{\pi}\bigl(s, \pi_{n-1}(s)\bigr)
&\le \max_{a}
    \Bigl[
        r_{n-1}(s,a)
        + \gamma(n) \sum_{s'} p(s' \mid s,a)\, V_{n}^{\pi}(s')
    \Bigr] \\
&= r_{n-1}\bigl(s, (\mathit{TP} \cdot \pi)_{n-1}(s)\bigr)
   + \gamma(n) \sum_{s'} p\bigl(s' \mid s,(\mathit{TP} \cdot \pi)_{n-1}(s)\bigr)
        V_{n}^{\pi}(s').
\end{align*}
From Step 2 we already know
\[
Q_{n}^{\pi}\bigl(s', \pi_{n}(s')\bigr)
\le Q_{n}^{\pi^{(n)}}\bigl(s', \pi^{(n)}_{n}(s')\bigr)
\quad \forall\, s' \in \mathcal{S},
\]
which implies $V_{n}^{\pi}(s') \le V_{n}^{\pi^{(n)}}(s')$ for all $s'$. Using
this in the expression above yields
\begin{align*}
Q_{n-1}^{\pi}\bigl(s, \pi_{n-1}(s)\bigr)
&\le r_{n-1}\bigl(s, (\mathit{TP} \cdot \pi)_{n-1}(s)\bigr)
   + \gamma(n) \sum_{s'} p\bigl(s' \mid s,(\mathit{TP} \cdot \pi)_{n-1}(s)\bigr)
        V_{n}^{\pi^{(n)}}(s') \\
&= Q_{n-1}^{\pi^{(n)}}\bigl(s, \pi^{(n)}_{n-1}(s)\bigr).
\end{align*}
Thus
\[
Q_{n-1}^{\pi}\bigl(s, \pi_{n-1}(s)\bigr)
\;\le\; Q_{n-1}^{\pi^{(n)}}\bigl(s, \pi^{(n)}_{n-1}(s)\bigr)
\quad \forall\, s \in \mathcal{S}.
\]

Repeating this argument inductively for $t = n-2, n-3, \dots, 0$, we obtain
\[
Q_{t}^{\pi}\bigl(s, \pi_{t}(s)\bigr)
\;\le\; Q_{t}^{\pi^{(n)}}\bigl(s, \pi^{(n)}_{t}(s)\bigr)
\quad \forall\, s \in \mathcal{S},\ \forall\, t \le n.
\]

\medskip
\noindent\textbf{Step 5: Value comparison for all $t$.}
Combining the three cases $t = n$, $t > n$, and $t < n$, we have
\[
Q_{t}^{\pi}\bigl(s, \pi_{t}(s)\bigr)
\;\le\; Q_{t}^{\pi^{(n)}}\bigl(s, \pi^{(n)}_{t}(s)\bigr)
\quad \forall\, s \in \mathcal{S},\ \forall\, t \in \mathbb{N}.
\]
Equivalently,
\[
V_{t}^{\pi}(s)
= Q_{t}^{\pi}\bigl(s, \pi_{t}(s)\bigr)
\;\le\; Q_{t}^{\pi^{(n)}}\bigl(s, \pi^{(n)}_{t}(s)\bigr)
= V_{t}^{\pi^{(n)}}(s),
\quad \forall\, s \in \mathcal{S},\ \forall\, t \in \mathbb{N}.
\]

\medskip
\noindent\textbf{Step 6: Passage to the fully improved policy.}
In the finite-horizon setting, the discount factor satisfies $\gamma_{T}(t) = 0$
for $t > T$, so $Q_{t}^{\pi'} \equiv 0$ and $V_{t}^{\pi'} \equiv 0$ for any
policy $\pi'$ and all $t > T$. In particular,
\[
V_{t}^{\pi^{(\infty)}}(s) = V_{t}^{\pi^{(T)}}(s)
\quad \forall\, t > T,\ \forall\, s \in \mathcal{S}.
\]
By Lemma~2, for any $t \le T$ we can express
\begin{align*}
V_{t}^{\pi^{(\infty)}}(s)
&= \sum_{s'} C\bigl(\pi^{(\infty)}_{t:T}, s'\bigr)
        V_{T+1}^{\pi^{(\infty)}}(s')
   + C\bigl(\pi^{(\infty)}_{t:T}\bigr) \\
&= \sum_{s'} C\bigl(\pi^{(T)}_{t:T}, s'\bigr)
        V_{T+1}^{\pi^{(T)}}(s')
   + C\bigl(\pi^{(T)}_{t:T}\bigr),
\end{align*}
since by construction $\pi^{(\infty)}_{t:T} = \pi^{(T)}_{t:T}$. Thus
\[
V_{t}^{\pi^{(\infty)}}(s) = V_{t}^{\pi^{(T)}}(s)
\quad \forall\, t \in \mathbb{N},\ \forall\, s \in \mathcal{S}.
\]

Putting everything together, for all $t$ and $s$,
\[
V_{t}^{\pi}(s)
\;\le\; V_{t}^{\pi^{(n)}}(s)
\quad \forall\, n,
\]
and in particular,
\[
V_{t}^{\pi}(s)
\;\le\; V_{t}^{\pi^{(\infty)}}(s)
= V_{t}^{\mathit{TP} \cdot \pi}(s),
\quad \forall\, s \in \mathcal{S},\ \forall\, t \in \mathbb{N}.
\]

Since $\pi^{(\infty)} = \mathit{TP} \cdot \pi$ by definition, this establishes
\[
V^{\pi}_{t}(s) \;\le\; V^{\mathit{TP} \cdot \pi}_{t}(s)
\quad \forall\, s \in \mathcal{S},\ \forall\, t \in \mathbb{N},
\]
which completes the proof.
\end{proof}

\begin{definition}[Greedy policy w.r.t.\ a value function]\label{def:greedy-policy}
Given a value function $V \in \mathcal{V}$, the \emph{greedy policy} 
$\pi^{V} \in \Pi$ is defined, for all $s \in \mathcal{S}$ and $t \in \mathbb{N}$, by
\[
\pi^{V}_{t}(s)
:= \arg\max_{a \in \mathcal{A}}
\Bigl[
    r_{t}(s,a)
    + \gamma(t+1) \sum_{s'} p(s' \mid s,a)\, V_{t+1}(s')
\Bigr].
\]
\end{definition}

\begin{theorem}[Existence of an optimal time-varying policy]\label{thm:optimal-policy}
There exists an optimal time-varying policy $\pi^{*} \in \Pi$ such that
\[
V_{t}^{\pi^{*}}(s) \;\ge\; V_{t}^{\pi}(s)
\quad \forall\, s \in \mathcal{S},\ \forall\, t \in \mathbb{N},\ \forall\, \pi \in \Pi.
\]
\end{theorem}

\begin{proof}
By Theorem~\ref{thm:optimal-V}, there exists an optimal value function
$V^{*} \in \mathcal{V}$ satisfying the Bellman optimality equation
\[
V^{*}_{t}(s)
= \max_{a}
\Bigl[
    r_{t}(s,a)
    + \gamma(t+1) \sum_{s'} p(s' \mid s,a)\, V^{*}_{t+1}(s')
\Bigr].
\]
Define the greedy policy $\pi^{*} := \pi^{V^{*}}$ as in
Definition~\ref{def:greedy-policy}:
\[
\pi^{*}_{t}(s)
\in \arg\max_{a}
\Bigl[
    r_{t}(s,a)
    + \gamma(t+1) \sum_{s'} p(s' \mid s,a)\, V^{*}_{t+1}(s')
\Bigr].
\]

By construction, for all $(s,t)$ we then have
\[
V^{*}_{t}(s)
= r_{t}\bigl(s,\pi^{*}_{t}(s)\bigr)
  + \gamma(t+1) \sum_{s'} p\bigl(s' \mid s,\pi^{*}_{t}(s)\bigr)\,
    V^{*}_{t+1}(s').
\]
This is exactly the Bellman equation for the value function of policy $\pi^{*}$.
By uniqueness of solutions to the Bellman equation for a fixed policy in this
finite-horizon setting, it follows that
\[
V^{\pi^{*}}_{t}(s) = V^{*}_{t}(s)
\quad \forall\, s \in \mathcal{S},\ \forall\, t \in \mathbb{N}.
\]

Since $V^{*}$ is optimal in $\mathcal{V}$, we have
\[
V^{\pi}_{t}(s) \;\le\; V^{*}_{t}(s)
= V^{\pi^{*}}_{t}(s)
\quad \forall\, s \in \mathcal{S},\ \forall\, t \in \mathbb{N},\ \forall\, \pi \in \Pi.
\]
Thus $\pi^{*}$ is an optimal time-varying policy.
\end{proof}

\begin{theorem}[Convergence of policy iteration]\label{thm:policy-iteration}
Starting from any policy $\pi \in \Pi$, recursively applying the Bellman policy
operator $\mathit{TP}$ converges to the optimal policy. In particular,
\[
\lim_{n \to \infty} V^{\mathit{TP}^{n} \cdot \pi}_{t}(s)
= V^{\pi^{*}}_{t}(s)
\quad \forall\, (s,t) \in \mathcal{S} \times \mathbb{N},
\]
where $\pi^{*}$ is an optimal time-varying policy as in
Theorem~\ref{thm:optimal-policy}.
\end{theorem}

\begin{proof}
Fix an arbitrary initial policy $\pi \in \Pi$. Consider the value function
$V^{\mathit{TP} \cdot \pi}$ of the one-step improved policy $\mathit{TP} \cdot \pi$.
By the Bellman equation for $V^{\mathit{TP} \cdot \pi}$, for any $(s,t)$ we have
\begin{align*}
V^{\mathit{TP} \cdot \pi}_{t}(s)
&= r_{t}\bigl(s, (\mathit{TP}\cdot \pi)_{t}(s)\bigr)
  + \gamma(t+1) \sum_{s'} p\bigl(s' \mid s, (\mathit{TP}\cdot \pi)_{t}(s)\bigr)
        V^{\mathit{TP} \cdot \pi}_{t+1}(s') \\
&= r_{t}\bigl(s, (\mathit{TP}\cdot \pi)_{t}(s)\bigr)
  + \gamma(t+1) \sum_{s'} p\bigl(s' \mid s, (\mathit{TP}\cdot \pi)_{t}(s)\bigr)
        \Bigl( V^{\mathit{TP} \cdot \pi}_{t+1}(s') - V^{\pi}_{t+1}(s') + V^{\pi}_{t+1}(s') \Bigr) \\
&= \underbrace{\gamma(t+1) \sum_{s'} p\bigl(s' \mid s, (\mathit{TP}\cdot \pi)_{t}(s)\bigr)
        \bigl( V^{\mathit{TP} \cdot \pi}_{t+1}(s') - V^{\pi}_{t+1}(s') \bigr)}_{\text{(I)}} \\
&\quad\quad
 + \Bigl[
        r_{t}\bigl(s, (\mathit{TP}\cdot \pi)_{t}(s)\bigr)
        + \gamma(t+1) \sum_{s'} p\bigl(s' \mid s, (\mathit{TP}\cdot \pi)_{t}(s)\bigr)
            V^{\pi}_{t+1}(s')
   \Bigr] \\
&= \text{(I)} + \max_{a}
    \Bigl[
        r_{t}(s,a)
        + \gamma(t+1) \sum_{s'} p(s' \mid s,a)\, V^{\pi}_{t+1}(s')
    \Bigr] \\
&= \text{(I)} + (TV^{\pi})_{t}(s),
\end{align*}
where the last two equalities use the definition of the greedy action
$(\mathit{TP} \cdot \pi)_{t}(s)$ and of the Bellman value operator $T$.

By Theorem~\ref{thm:policy-improvement} we know that
\[
V^{\mathit{TP} \cdot \pi}_{t+1}(s') \;\ge\; V^{\pi}_{t+1}(s')
\quad \forall\, s' \in \mathcal{S},\ \forall\, t,
\]
so every term in
\[
V^{\mathit{TP} \cdot \pi}_{t+1}(s') - V^{\pi}_{t+1}(s')
\]
is nonnegative. Since $p(\cdot \mid s,(\mathit{TP}\cdot \pi)_{t}(s))$ is a
probability distribution and $\gamma(t+1) \ge 0$, we have
\[
\text{(I)} \;\ge\; 0.
\]
Thus
\[
V^{\mathit{TP} \cdot \pi}_{t}(s)
\;\ge\; (TV^{\pi})_{t}(s)
\quad \forall\, s \in \mathcal{S},\ \forall\, t \in \mathbb{N}.
\]

Iterating this inequality, we obtain for any $n \ge 1$,
\[
V^{\mathit{TP}^{n} \cdot \pi}_{t}(s)
\;\ge\; (T^{n}V^{\pi})_{t}(s)
\quad \forall\, s \in \mathcal{S},\ \forall\, t \in \mathbb{N}.
\]
Taking limits as $n \to \infty$ and using Theorem~\ref{thm:optimal-V}, which
states that $T^{n}V^{\pi} \to V^{*}$ (the unique optimal value function), we get
\[
\lim_{n \to \infty} V^{\mathit{TP}^{n} \cdot \pi}_{t}(s)
\;\ge\; \lim_{n \to \infty} (T^{n}V^{\pi})_{t}(s)
= V^{*}_{t}(s)
= V^{\pi^{*}}_{t}(s),
\]
where the last equality uses Theorem~\ref{thm:optimal-policy}:
$V^{\pi^{*}} = V^{*}$.

On the other hand, for every $n$,
\[
V^{\mathit{TP}^{n} \cdot \pi}_{t}(s)
\;\le\; V^{\pi^{*}}_{t}(s)
\quad \forall\, s,t,
\]
because $\pi^{*}$ is optimal and thus dominates every policy, including
$\mathit{TP}^{n} \cdot \pi$. Passing to the limit yields
\[
\lim_{n \to \infty} V^{\mathit{TP}^{n} \cdot \pi}_{t}(s)
\;\le\; V^{\pi^{*}}_{t}(s).
\]

Combining both inequalities, we conclude
\[
\lim_{n \to \infty} V^{\mathit{TP}^{n} \cdot \pi}_{t}(s)
= V^{\pi^{*}}_{t}(s)
\quad \forall\, (s,t) \in \mathcal{S} \times \mathbb{N},
\]
which proves the theorem.
\end{proof}

\begin{theorem}\label{thm:static-gap}
   There exists a finite-horizon dynamic MDP (DMDP) for which an
optimal time-varying policy achieves strictly higher value than any static
(time-invariant) policy.
\end{theorem}

\begin{proof}
We construct a finite-horizon DMDP with:
\[
\mathcal{S} = \{s\}, \quad \mathcal{A} = \{a_1, a_2\}, \quad T = 1,\quad
P(s \mid s, a) = 1 \ \forall a,\quad \gamma(t) \equiv 1.
\]
The reward is time-varying and given by
\[
r_0(s, a_1) = 0,\quad r_0(s, a_2) = 1, \qquad
r_1(s, a_1) = 1,\quad r_1(s, a_2) = 0.
\]

A \emph{static} policy is one that does not depend on $t$, i.e.
$\pi_t(s) \equiv \bar{\pi}(s)$ for all $t$. Since there is only one state,
any static policy must either always choose $a_1$ or always choose $a_2$:
\[
\bar{\pi}(s) = a_1
\quad \text{or} \quad
\bar{\pi}(s) = a_2.
\]

Compute the value at $t=0$ for each static policy:
\begin{itemize}
    \item If $\bar{\pi}(s) = a_1$, then
    \[
    V_0^{\bar{\pi}}(s)
    = r_0(s,a_1) + r_1(s,a_1)
    = 0 + 1 = 1.
    \]
    \item If $\bar{\pi}(s) = a_2$, then
    \[
    V_0^{\bar{\pi}}(s)
    = r_0(s,a_2) + r_1(s,a_2)
    = 1 + 0 = 1.
    \]
\end{itemize}
So any static policy attains value $V_0^{\bar{\pi}}(s) = 1$.

Now consider the time-varying policy $\tilde{\pi}$ defined by
\[
\tilde{\pi}_0(s) = a_2, \qquad \tilde{\pi}_1(s) = a_1.
\]
Its value at $t=0$ is
\[
V_0^{\tilde{\pi}}(s)
= r_0(s, a_2) + r_1(s, a_1)
= 1 + 1 = 2.
\]

Thus,
\[
\sup_{\pi \ \text{static}} V_0^{\pi}(s) = 1
\quad\text{but}\quad
\sup_{\pi \ \text{time-varying}} V_0^{\pi}(s) \ge V_0^{\tilde{\pi}}(s) = 2.
\]
Hence allowing policies to depend on time strictly improves the optimal value
in this DMDP. This proves the claim.
\end{proof}

\begin{theorem}[Value of a concatenated policy]\label{thm:concat-value}
Under Assumptions~\hyperref[asm:A1]{A1}--\hyperref[asm:A3]{A3},
let $\pi = \pi^{1}_{0:T_{1}} \circ \pi^{2}_{0:T_{2}}$ be as in
Definition~\ref{def:concat-policy} and let $T := T_{1} + T_{2}$. Then for all
$t = 0,\dots,T-1$ and $s \in \mathcal{S}$,
\begin{equation}\label{eq:concat-compact}
V_{t}^{\pi}(s)
=
V^{\pi^{1}}_{t}(s)
+ \gamma^{\max(T_{1}-t,0)}
  \,\mathbb{E}\!\Big[
      V^{\pi^{2}}_{\max(0,\,t-T_{1})}\bigl(s_{\max(t,T_{1})}\bigr)
      \,\Big|\, s_{t} = s
    \Big].
\end{equation}
Equivalently, for $t < T$,
\[
V_{t}^{\pi}(s)
=
\begin{cases}
V_{t}^{\pi^{1}}(s)
+ \gamma^{T_{1}-t}\,
  \mathbb{E}\!\big[
    V_{0}^{\pi^{2}}(s_{T_{1}}) \mid s_{t} = s
  \big],
& t < T_{1},\\[4pt]
V_{t-T_{1}}^{\pi^{2}}(s),
& T_{1} \le t < T.
\end{cases}
\]
\end{theorem}

\begin{proof}
By definition of the finite-horizon value function for a time-varying policy,
we split the discounted return at time $T_{1}$. For $t < T_{1}$, the first
$T_{1}-t$ rewards are collected under $\pi^{1}$ and the remaining $T_{2}$
rewards under $\pi^{2}$ starting from $s_{T_{1}}$. Taking expectations over
$s_{T_{1}}\sim\pi^{1}$ and applying the tower property yields the stated
decomposition. For $T_{1} \le t < T$, the policy is entirely $\pi^{2}$ and the
value reduces to $V_{t-T_{1}}^{\pi^{2}}(s)$.
\end{proof}

\begin{theorem}[GDS optimality for optimal reach]\label{thm:gds-reach}
Under Assumptions~\hyperref[asm:A1]{A1}--\hyperref[asm:A3]{A3}, for any reachable goal set, General Dijkstra Search for Optimal Reach finds an
optimal goal-reaching policy. Specifically, for every
$\mathcal{G} \in \mathcal{G}_{T}^{\supset}(s)$, there exists
$\pi^{*} \in \Pi_{1:T}^{\mathcal{G}\supset}(s)$ such that for every
$\pi \in \Pi_{1:T}^{\mathcal{G}\supset}(s)$,
\[
V_{0}^{\pi^{*}}(s) \ge V_{0}^{\pi}(s).
\]
\end{theorem}

\begin{proof}
The proof follows the same three-step strategy as the goal-covering analog in Appendix~\ref{app:gds-coverage}: (i) the queue invariant shows that every queued element $(\pi_{1:t},v_t,\mathcal{G}^{\pi_{1:t}}(s),t)$ satisfies $v_t = V_0^{\pi_{1:t}}(s)$; (ii) any popped policy has value at least as large as every queued or future-queued policy, using Theorem~\ref{thm:concat-value} and $r_t\le 0$ to bound continuation values; (iii) the pruning step discards only policies dominated by a strictly better one, so no optimal policy is lost. Combining these, the first popped policy whose goal set is contained in the target $\mathcal{G}^{*}$ is optimal.
\end{proof}

% Optionally include supplemental material (complete proofs, additional experiments and plots) in appendix.
% All such materials \textbf{SHOULD be included in the main submission.}

%%%%%%%%%%%%%%%%%%%%%%%%%%%%%%%%%%%%%%%%%%%%%%%%%%%%%%%%%%%%

% \newpage

\clearpage
\section{Policy Composition}
\label{app:policy-composition}

\begin{definition}[Policy dominance]
Let $\pi^{1}, \pi^{2} \in \bigcup_{t=1}^{T} \Pi_{t}$. We say that
$\pi^{1}$ is dominated by $\pi^{2}$ if
\[
\mathcal{G}^{\pi^{1}}(s) \subset \mathcal{G}^{\pi^{2}}(s).
\]
\end{definition}

\begin{definition}[Dominating and dominated goal sets]
Let $\mathcal{G} \subset \mathcal{S}$ and let
$\pi \in \bigcup_{t=1}^{T} \Pi_{t}$.
\begin{enumerate}
    \item A policy $\pi$ is said to \emph{dominate} a goal set
    $\mathcal{G}$ when starting from $s \in \mathcal{S}$ if
    \[
    \mathcal{G} \subset \mathcal{G}^{\pi}(s).
    \]
    \item A policy $\pi$ is said to be \emph{dominated by} a goal set
    $\mathcal{G}$ when starting from $s \in \mathcal{S}$ if
    \[
    \mathcal{G}^{\pi}(s) \subset \mathcal{G}.
    \]
\end{enumerate}
\end{definition}

\begin{lemma}
If $\pi^{1,*}$ dominates $\pi^{1}$, then
$\pi^{1,*} \circ \pi^{2}$ also dominates $\pi^{1} \circ \pi^{2}$.
\end{lemma}

\begin{proof}
Since $\pi^{1,*}$ dominates $\pi^{1}$, we have
\[
\mathcal{G}^{\pi^{1}}(s) \subset \mathcal{G}^{\pi^{1,*}}(s).
\]
Therefore,
\[
\mathcal{G}^{\pi^{1} \circ \pi^{2}}(s)
= \bigcup_{s' \in \mathcal{G}^{\pi^{1}}(s)} \mathcal{G}^{\pi^{2}}(s')
\subset
\bigcup_{s' \in \mathcal{G}^{\pi^{1,*}}(s)} \mathcal{G}^{\pi^{2}}(s')
= \mathcal{G}^{\pi^{1,*} \circ \pi^{2}}(s).
\]
Hence $\pi^{1,*} \circ \pi^{2}$ dominates $\pi^{1} \circ \pi^{2}$.
\end{proof}

\begin{lemma}
In the General Dijkstra Search algorithms for optimal reach and optimal
coverage, any element $(\pi_{1:t}, v_{t}, \mathcal{G}^{\pi_{1:t}}(s), t)$
within the priority queue satisfies
\[
v_{t} = V_{0}^{\pi_{1:t}}(s).
\]
\end{lemma}

\begin{proof}
We prove the claim by induction on the horizon length $t$ of a policy in the
queue.

For the base case $t=1$, any element
$(\pi_{1}, v_{1}, \mathcal{G}^{\pi_{1}}(s), 1) \in \mathcal{Q}$ with
$\pi_{1} \in \Pi_{1}$ is constructed by concatenating a one-step policy
$\pi \in \Pi_{1}$ with the empty policy inserted at initialization. Hence
$\pi_{1} = \pi$ and $v_{0}=0$, so
\[
\begin{aligned}
v_{1}
&= 0 + \gamma^{0} \cdot \mathbb{E}_{s_{0}}\Big[ V_{0}^{\pi}(s_{0}) \,\big|\, s_{0}=s\Big] \\
&= V_{0}^{\pi_{1}}(s).
\end{aligned}
\]

Now assume that for some $n \le T$, every element
$(\pi_{1:t}, v_{t}, \mathcal{G}^{\pi_{1:t}}(s), t) \in \mathcal{Q}$ with
$0 \le t \le n-1$ satisfies
\[
v_{t} = V_{0}^{\pi_{1:t}}(s).
\]

Consider any element
$(\pi_{1:t+1}, v_{t+1}, \mathcal{G}^{\pi_{1:t+1}}(s), t+1) \in \mathcal{Q}$.
By construction, it is obtained by extending some
$(\pi_{1:t}, v_{t}, \mathcal{G}^{\pi_{1:t}}(s), t) \in \mathcal{Q}$ by a
one-step policy $\pi_{t+1} \in \Pi_{1}$. Therefore,
\[
\begin{aligned}
v_{t+1}
&= v_{t} + \gamma^{t} \cdot \mathbb{E}_{s_{0:t}}\Big[ V_{0}^{\pi_{t+1}}(s_{t}) \,\big|\, s_{0} = s\Big] \\
&= V_{0}^{\pi_{1:t}}(s) + \gamma^{t} \cdot \mathbb{E}_{s_{0:t}}\Big[ V_{0}^{\pi_{t+1}}(s_{t}) \,\big|\, s_{0} = s\Big] \\
&= V_{0}^{\pi_{1:t+1}}(s),
\end{aligned}
\]
where the second equality uses the induction hypothesis and the third equality
follows from Theorem~\ref{thm:concat-value}. This completes the induction.
\end{proof}

\begin{lemma}
Any popped policy $\pi^{*}$ within the General Dijkstra Search algorithms for
optimal reach and optimal coverage has larger value at $t=0$ than any policy
$\pi$ that is either currently in the queue or will be added to the queue.
Specifically,
\[
V_{0}^{\pi^{*}}(s) \ge V_{0}^{\pi}(s).
\]
\end{lemma}

\begin{proof}
Let $\pi^{*}$ be the policy popped at some step of the algorithm. Every policy
in $\bigcup_{t=1}^{T} \Pi_{t}$ belongs to one of the following groups:
\begin{enumerate}
    \item policies already popped from $\mathcal{Q}$,
    \item policies currently in $\mathcal{Q}$,
    \item policies that will be added to $\mathcal{Q}$ in the future,
    \item policies that never appear in $\mathcal{Q}$.
\end{enumerate}

Since $\pi^{*}$ is popped from the priority queue, it has value at least as
large as every policy currently in the queue. Hence for every policy
$\pi$ in group (2),
\[
V_{0}^{\pi^{*}}(s) \ge V_{0}^{\pi}(s).
\]

Now let $\pi$ be a policy that will be added to the queue in the future. Then
$\pi$ must extend some current queue element $\pi_{1:t_{1}}$ with $t_{1} < T$.
By Theorem~\ref{thm:concat-value},
\[
V_{0}^{\pi}(s)
= V_{0}^{\pi_{1:t_{1}}}(s)
  + \gamma^{t_{1}} \cdot
    \mathbb{E}_{s_{0:t_{1}} \sim \pi_{1:t_{1}}}
    \Big[ V_{0}^{\pi_{t_{1}+1:t}}(s_{t_{1}}) \,\big|\, s_{0}=s\Big].
\]
By Lemma~3, the continuation value satisfies
$V_{0}^{\pi_{t_{1}+1:t}}(s') \le 0$ for all $s' \in \mathcal{S}$. Therefore,
\[
V_{0}^{\pi}(s) \le V_{0}^{\pi_{1:t_{1}}}(s).
\]
Since $\pi_{1:t_{1}}$ is currently in the queue or is equal to $\pi^{*}$, the
previous argument implies
\[
V_{0}^{\pi}(s) \le V_{0}^{\pi_{1:t_{1}}}(s) \le V_{0}^{\pi^{*}}(s).
\]
This proves the claim for all policies currently in the queue or to be added
to the queue.
\end{proof}

\begin{lemma}
Under Assumption~4, in General Dijkstra Search (Optimal Reach), any policy
$\pi \in \Pi_{t}$ with $t \le T$ that never appears in the queue must
dominate a popped and unskipped policy $\pi^{*}$ such that
\[
V_{0}^{\pi}(s) + \epsilon_{t} \le V_{0}^{\pi^{*}}(s)
\qquad \text{and} \qquad
\mathcal{G}^{\pi}(s) \supset \mathcal{G}^{\pi^{*}}(s).
\]
\end{lemma}

\begin{proof}
Let $\pi \in \Pi_{t}$ with $t \le T$ be a policy that never appears in
$\mathcal{Q}$. Then there exists an index $1 \le t' < t$ such that
$\pi_{1:t'}$ is the longest prefix of $\pi$ that is popped and skipped by the
algorithm. We prove by induction on
$n \in \{t', t'+1, \dots, t\}$ that there exists a popped and unskipped policy
$\pi_{n}^{*}$ dominated by $\pi_{1:n}$ and satisfying
\[
V_{0}^{\pi_{n}^{*}}(s) - V_{0}^{\pi_{1:n}}(s) \ge \epsilon_{n}.
\]

For the base case $n=t'$, since $\pi_{1:t'}$ is popped and skipped, by the
pruning rule in Algorithm~1 there exists a popped and unskipped policy
$\pi_{t'}^{*}$ such that
\[
\mathcal{G}^{\pi_{1:t'}}(s) \supset \mathcal{G}^{\pi_{t'}^{*}}(s)
\qquad \text{and} \qquad
V_{0}^{\pi_{t'}^{*}}(s) - V_{0}^{\pi_{1:t'}}(s) \ge \epsilon_{t'}.
\]

Now assume that for some $t_{0} \in \{t', t'+1, \dots, t-1\}$ there exists a
popped and unskipped policy $\pi_{t_{0}}^{*}$ dominated by $\pi_{1:t_{0}}$ such
that
\[
V_{0}^{\pi_{t_{0}}^{*}}(s) - V_{0}^{\pi_{1:t_{0}}}(s) \ge \epsilon_{t_{0}}.
\]
Because $\pi_{t_{0}}^{*}$ is popped and unskipped, all of its one-step
extensions are added to the queue, including
$\pi_{t_{0}}^{*} \circ \pi_{t_{0}+1}$. We consider two cases.

\medskip
\noindent\textbf{Case 1.}
$\pi_{t_{0}}^{*} \circ \pi_{t_{0}+1}$ is popped and unskipped. Then
\[
\begin{aligned}
&V_{0}^{\pi_{t_{0}}^{*} \circ \pi_{t_{0}+1}}(s)
 - V_{0}^{\pi_{1:t_{0}+1}}(s) \\
&= \Bigl(V_{0}^{\pi_{t_{0}}^{*}}(s) - V_{0}^{\pi_{1:t_{0}}}(s)\Bigr)
 + \gamma^{t_{0}} \Bigl(
    \mathbb{E}_{s_{0:t_{0}} \sim \pi_{t_{0}}^{*}}
    \bigl[V_{0}^{\pi_{t_{0}+1}}(s_{t_{0}}) \mid s_{0}=s\bigr]
    -
    \mathbb{E}_{s_{0:t_{0}} \sim \pi_{1:t_{0}}}
    \bigl[V_{0}^{\pi_{t_{0}+1}}(s_{t_{0}}) \mid s_{0}=s\bigr]
   \Bigr).
\end{aligned}
\]
By the induction hypothesis, the first term is at least $\epsilon_{t_{0}}$.
By Lemma~3, each continuation value is bounded below by
$-\frac{r_{\max}}{1-\gamma}$ and above by $0$, hence the difference in
parentheses is bounded below by $-\frac{r_{\max}}{1-\gamma}$. Therefore,
\[
V_{0}^{\pi_{t_{0}}^{*} \circ \pi_{t_{0}+1}}(s)
 - V_{0}^{\pi_{1:t_{0}+1}}(s)
\ge
\epsilon_{t_{0}} - \gamma^{t_{0}} \frac{r_{\max}}{1-\gamma}
= \epsilon_{t_{0}+1}.
\]
By Lemma~4,
$\pi_{t_{0}}^{*} \circ \pi_{t_{0}+1}$ is dominated by $\pi_{1:t_{0}+1}$.

\medskip
\noindent\textbf{Case 2.}
$\pi_{t_{0}}^{*} \circ \pi_{t_{0}+1}$ is popped and skipped. Then by the
pruning rule in Algorithm~1 there exists a popped and unskipped policy
$\pi_{t_{0}+1}^{*}$ dominated by
$\pi_{t_{0}}^{*} \circ \pi_{t_{0}+1}$ such that
\[
V_{0}^{\pi_{t_{0}+1}^{*}}(s)
 - V_{0}^{\pi_{t_{0}}^{*} \circ \pi_{t_{0}+1}}(s)
\ge \epsilon_{t_{0}+1}.
\]
Since domination is transitive and Lemma~4 implies that
$\pi_{t_{0}}^{*} \circ \pi_{t_{0}+1}$ is dominated by $\pi_{1:t_{0}+1}$, the
policy $\pi_{t_{0}+1}^{*}$ is also dominated by $\pi_{1:t_{0}+1}$. Moreover,
combining the previous display with the lower bound from Case 1 gives
\[
\begin{aligned}
V_{0}^{\pi_{t_{0}+1}^{*}}(s) - V_{0}^{\pi_{1:t_{0}+1}}(s)
&=
\Bigl(V_{0}^{\pi_{t_{0}+1}^{*}}(s)
 - V_{0}^{\pi_{t_{0}}^{*} \circ \pi_{t_{0}+1}}(s)\Bigr)
 +
\Bigl(V_{0}^{\pi_{t_{0}}^{*} \circ \pi_{t_{0}+1}}(s)
 - V_{0}^{\pi_{1:t_{0}+1}}(s)\Bigr) \\
&\ge \epsilon_{t_{0}+1}.
\end{aligned}
\]

In either case, we obtain a popped and unskipped policy dominated by
$\pi_{1:t_{0}+1}$ and satisfying the desired value gap. By induction, this
holds up to $n=t$. Hence there exists a popped and unskipped policy
$\pi^{*}$ dominated by $\pi$ such that
\[
V_{0}^{\pi^{*}}(s) - V_{0}^{\pi}(s) \ge \epsilon_{t}.
\]
Equivalently,
\[
V_{0}^{\pi}(s) + \epsilon_{t} \le V_{0}^{\pi^{*}}(s)
\qquad \text{and} \qquad
\mathcal{G}^{\pi}(s) \supset \mathcal{G}^{\pi^{*}}(s).
\]
This proves the claim.
\end{proof}

\begin{lemma}
In General Dijkstra Search for Optimal Reach (Algorithm~1) starting from
$s \in \mathcal{S}$, let $\pi^{*}$ be a popped and unskipped policy that does
not dominate any previous policy in the queue. Then $\pi^{*}$ is optimal among
all policies that it dominates. Equivalently, if
$\mathcal{G} = \mathcal{G}^{\pi^{*}}(s)$, then for every
$\pi \in \Pi_{1:T}^{\mathcal{G}\supset}(s)$,
\[
V_{0}^{\pi^{*}}(s) \ge V_{0}^{\pi}(s).
\]
\end{lemma}

\begin{proof}
Fix a step $n$ at which $\pi^{*}$ is popped and unskipped. Every policy in
$\bigcup_{t=1}^{T} \Pi_{t}$ belongs to one of the following groups at step $n$:
\begin{enumerate}
    \item policies popped from $\mathcal{Q}$ at some step $\le n$,
    \item policies currently in $\mathcal{Q}$ at step $n$,
    \item policies that will enter $\mathcal{Q}$ at some step $> n$,
    \item policies that never appear in $\mathcal{Q}$.
\end{enumerate}

Because $\pi^{*}$ is the first popped and unskipped policy that does not
dominate any previous policy in the queue, there is no policy in group (1)
whose goal set is contained in $\mathcal{G}^{\pi^{*}}(s)$. That is,
\[
\Bigl\{ \pi \in \bigcup_{t=1}^{T} \Pi_{t}
\,\Big|\, \mathcal{G}^{\pi}(s) \subset \mathcal{G}^{\pi^{*}}(s) \Bigr\}
\cap (1) = \emptyset.
\]

Now let
$\pi \in \bigcup_{t=1}^{T} \Pi_{t}$ satisfy
$\mathcal{G}^{\pi}(s) \subset \mathcal{G}^{\pi^{*}}(s)$. We consider the
possible groups to which $\pi$ belongs.

If $\pi$ belongs to group (2) or group (3), then Lemma~7 gives directly
\[
V_{0}^{\pi^{*}}(s) \ge V_{0}^{\pi}(s).
\]

If $\pi$ belongs to group (4), then by Lemma~8 there exists a popped and
unskipped policy $\pi'$ such that
\[
V_{0}^{\pi'}(s) \ge V_{0}^{\pi}(s)
\qquad \text{and} \qquad
\mathcal{G}^{\pi}(s) \supset \mathcal{G}^{\pi'}(s).
\]
Since $\mathcal{G}^{\pi}(s) \subset \mathcal{G}^{\pi^{*}}(s)$, we also have
\[
\mathcal{G}^{\pi'}(s) \subset \mathcal{G}^{\pi^{*}}(s).
\]
Therefore $\pi'$ cannot belong to group (1), by the defining property of
$\pi^{*}$. It also cannot be $\pi^{*}$-dominating in the queue prior to the pop
of $\pi^{*}$. Hence $\pi'$ must belong to group (3), and Lemma~7 yields
\[
V_{0}^{\pi^{*}}(s) \ge V_{0}^{\pi'}(s).
\]
Combining the two inequalities gives
\[
V_{0}^{\pi^{*}}(s) \ge V_{0}^{\pi'}(s) \ge V_{0}^{\pi}(s).
\]

Thus for every policy $\pi$ with
$\mathcal{G}^{\pi}(s) \subset \mathcal{G}^{\pi^{*}}(s)$, we have
\[
V_{0}^{\pi^{*}}(s) \ge V_{0}^{\pi}(s).
\]
Equivalently, $\pi^{*}$ is the optimal goal-reaching policy for the pair
$(s, \mathcal{G}^{\pi^{*}}(s))$.
\end{proof}

\begin{lemma}
Under Assumption~4, in General Dijkstra Search for Optimal Coverage
(Algorithm~2), any policy $\pi \in \Pi_{t}$ with $t \le T$ that never appears
in the queue is dominated by a popped and unskipped policy $\pi^{*}$ such that
\[
V_{0}^{\pi}(s) + \epsilon_{t} \le V_{0}^{\pi^{*}}(s)
\qquad \text{and} \qquad
\mathcal{G}^{\pi}(s) \subset \mathcal{G}^{\pi^{*}}(s).
\]
\end{lemma}

\begin{proof}
Let $\pi \in \Pi_{t}$ with $t \le T$ be a policy that never appears in
$\mathcal{Q}$. Then there exists an index $1 \le t' < t$ such that
$\pi_{1:t'}$ is the longest prefix of $\pi$ that is popped and skipped by the
algorithm. We prove by induction on
$n \in \{t', t'+1, \dots, t\}$ that there exists a popped and unskipped policy
$\pi_{n}^{*}$ that dominates $\pi_{1:n}$ and satisfies
\[
V_{0}^{\pi_{n}^{*}}(s) - V_{0}^{\pi_{1:n}}(s) \ge \epsilon_{n}.
\]

For the base case $n=t'$, since $\pi_{1:t'}$ is popped and skipped, by the
pruning rule in Algorithm~2 there exists a popped and unskipped policy
$\pi_{t'}^{*}$ such that
\[
\mathcal{G}^{\pi_{1:t'}}(s) \subset \mathcal{G}^{\pi_{t'}^{*}}(s)
\qquad \text{and} \qquad
V_{0}^{\pi_{t'}^{*}}(s) - V_{0}^{\pi_{1:t'}}(s) \ge \epsilon_{t'}.
\]

Now assume that for some $t_{0} \in \{t', t'+1, \dots, t-1\}$ there exists a
popped and unskipped policy $\pi_{t_{0}}^{*}$ that dominates $\pi_{1:t_{0}}$
and satisfies
\[
V_{0}^{\pi_{t_{0}}^{*}}(s) - V_{0}^{\pi_{1:t_{0}}}(s) \ge \epsilon_{t_{0}}.
\]
Because $\pi_{t_{0}}^{*}$ is popped and unskipped, all of its one-step
extensions are added to the queue, including
$\pi_{t_{0}}^{*} \circ \pi_{t_{0}+1}$. We consider two cases.

\medskip
\noindent\textbf{Case 1.}
$\pi_{t_{0}}^{*} \circ \pi_{t_{0}+1}$ is popped and unskipped. Then
\[
\begin{aligned}
&V_{0}^{\pi_{t_{0}}^{*} \circ \pi_{t_{0}+1}}(s)
 - V_{0}^{\pi_{1:t_{0}+1}}(s) \\
&= \Bigl(V_{0}^{\pi_{t_{0}}^{*}}(s) - V_{0}^{\pi_{1:t_{0}}}(s)\Bigr)
 + \gamma^{t_{0}} \Bigl(
    \mathbb{E}_{s_{0:t_{0}} \sim \pi_{t_{0}}^{*}}
    \bigl[V_{0}^{\pi_{t_{0}+1}}(s_{t_{0}}) \mid s_{0}=s\bigr]
    -
    \mathbb{E}_{s_{0:t_{0}} \sim \pi_{1:t_{0}}}
    \bigl[V_{0}^{\pi_{t_{0}+1}}(s_{t_{0}}) \mid s_{0}=s\bigr]
   \Bigr).
\end{aligned}
\]
By the induction hypothesis, the first term is at least $\epsilon_{t_{0}}$.
By Lemma~3, the difference in parentheses is bounded below by
$-\frac{r_{\max}}{1-\gamma}$. Therefore,
\[
V_{0}^{\pi_{t_{0}}^{*} \circ \pi_{t_{0}+1}}(s)
 - V_{0}^{\pi_{1:t_{0}+1}}(s)
\ge
\epsilon_{t_{0}} - \gamma^{t_{0}} \frac{r_{\max}}{1-\gamma}
= \epsilon_{t_{0}+1}.
\]
By Lemma~4,
$\pi_{t_{0}}^{*} \circ \pi_{t_{0}+1}$ dominates $\pi_{1:t_{0}+1}$.

\medskip
\noindent\textbf{Case 2.}
$\pi_{t_{0}}^{*} \circ \pi_{t_{0}+1}$ is popped and skipped. Then by the
pruning rule in Algorithm~2 there exists a popped and unskipped policy
$\pi_{t_{0}+1}^{*}$ that dominates
$\pi_{t_{0}}^{*} \circ \pi_{t_{0}+1}$ and satisfies
\[
V_{0}^{\pi_{t_{0}+1}^{*}}(s)
 - V_{0}^{\pi_{t_{0}}^{*} \circ \pi_{t_{0}+1}}(s)
\ge \epsilon_{t_{0}+1}.
\]
Since dominance is transitive and Lemma~4 implies that
$\pi_{t_{0}}^{*} \circ \pi_{t_{0}+1}$ dominates $\pi_{1:t_{0}+1}$, the policy
$\pi_{t_{0}+1}^{*}$ also dominates $\pi_{1:t_{0}+1}$. Moreover,
\[
\begin{aligned}
V_{0}^{\pi_{t_{0}+1}^{*}}(s) - V_{0}^{\pi_{1:t_{0}+1}}(s)
&=
\Bigl(V_{0}^{\pi_{t_{0}+1}^{*}}(s)
 - V_{0}^{\pi_{t_{0}}^{*} \circ \pi_{t_{0}+1}}(s)\Bigr)
 +
\Bigl(V_{0}^{\pi_{t_{0}}^{*} \circ \pi_{t_{0}+1}}(s)
 - V_{0}^{\pi_{1:t_{0}+1}}(s)\Bigr) \\
&\ge \epsilon_{t_{0}+1}.
\end{aligned}
\]

In either case, we obtain a popped and unskipped policy that dominates
$\pi_{1:t_{0}+1}$ and satisfies the desired value gap. By induction, this
holds up to $n=t$. Hence there exists a popped and unskipped policy
$\pi^{*}$ that dominates $\pi$ such that
\[
V_{0}^{\pi^{*}}(s) - V_{0}^{\pi}(s) \ge \epsilon_{t}.
\]
Equivalently,
\[
V_{0}^{\pi}(s) + \epsilon_{t} \le V_{0}^{\pi^{*}}(s)
\qquad \text{and} \qquad
\mathcal{G}^{\pi}(s) \subset \mathcal{G}^{\pi^{*}}(s).
\]
This proves the claim.
\end{proof}

\begin{lemma}
In General Dijkstra Search for Optimal Coverage (Algorithm~2) starting from
$s \in \mathcal{S}$, let $\pi^{*}$ be a popped and unskipped policy that is
not dominated by any previous policy in the queue. Then $\pi^{*}$ is optimal
among all policies that dominate it. Equivalently, if
$\mathcal{G} = \mathcal{G}^{\pi^{*}}(s)$, then for every
$\pi \in \Pi_{1:T}^{\mathcal{G}\subset}(s)$,
\[
V_{0}^{\pi^{*}}(s) \ge V_{0}^{\pi}(s).
\]
\end{lemma}

\begin{proof}
Fix a step $n$ at which $\pi^{*}$ is popped and unskipped. Every policy in
$\bigcup_{t=1}^{T} \Pi_{t}$ belongs to one of the following groups at step $n$:
\begin{enumerate}
    \item policies popped from $\mathcal{Q}$ at some step $\le n$,
    \item policies currently in $\mathcal{Q}$ at step $n$,
    \item policies that will enter $\mathcal{Q}$ at some step $> n$,
    \item policies that never appear in $\mathcal{Q}$.
\end{enumerate}

Because $\pi^{*}$ is the first popped and unskipped policy that is not
dominated by any previous policy in the queue, there is no policy in group (1)
whose goal set contains $\mathcal{G}^{\pi^{*}}(s)$. That is,
\[
\Bigl\{ \pi \in \bigcup_{t=1}^{T} \Pi_{t}
\,\Big|\, \mathcal{G}^{\pi^{*}}(s) \subset \mathcal{G}^{\pi}(s) \Bigr\}
\cap (1) = \emptyset.
\]

Now let
$\pi \in \bigcup_{t=1}^{T} \Pi_{t}$ satisfy
$\mathcal{G}^{\pi^{*}}(s) \subset \mathcal{G}^{\pi}(s)$. We consider the
possible groups to which $\pi$ belongs.

If $\pi$ belongs to group (2) or group (3), then Lemma~7 gives directly
\[
V_{0}^{\pi^{*}}(s) \ge V_{0}^{\pi}(s).
\]

If $\pi$ belongs to group (4), then by Lemma~10 there exists a popped and
unskipped policy $\pi'$ such that
\[
V_{0}^{\pi'}(s) \ge V_{0}^{\pi}(s)
\qquad \text{and} \qquad
\mathcal{G}^{\pi}(s) \subset \mathcal{G}^{\pi'}(s).
\]
Since $\mathcal{G}^{\pi^{*}}(s) \subset \mathcal{G}^{\pi}(s)$, we also have
\[
\mathcal{G}^{\pi^{*}}(s) \subset \mathcal{G}^{\pi'}(s).
\]
Therefore $\pi'$ cannot belong to group (1), by the defining property of
$\pi^{*}$. Hence $\pi'$ must belong to group (3), and Lemma~7 yields
\[
V_{0}^{\pi^{*}}(s) \ge V_{0}^{\pi'}(s).
\]
Combining the two inequalities gives
\[
V_{0}^{\pi^{*}}(s) \ge V_{0}^{\pi'}(s) \ge V_{0}^{\pi}(s).
\]

Thus for every policy $\pi$ with
$\mathcal{G}^{\pi^{*}}(s) \subset \mathcal{G}^{\pi}(s)$, we have
\[
V_{0}^{\pi^{*}}(s) \ge V_{0}^{\pi}(s).
\]
Equivalently, $\pi^{*}$ is the optimal goal-covering policy for the pair
$(s, \mathcal{G}^{\pi^{*}}(s))$.
\end{proof}

\clearpage
\section{Symmetric goal-covering result}
\label{app:gds-coverage}

For completeness we record the dual of the goal-reaching algorithm and theorem
of \S\ref{sec:gds}. The covering case swaps the direction of the goal-set
inclusion: a policy is \emph{covering} a target $\mathcal{G}$ if its reachable
goal set contains $\mathcal{G}$, rather than being contained in it.

\begin{definition}[Covering policies]
Given a start state $s \in \mathcal{S}$ and a set of goal states
$\mathcal{G} \subset \mathcal{S}$, the set of \emph{covering policies}
$\Pi^{\mathcal{G}\subset}_{1:T}(s)$ is
\[
\Pi^{\mathcal{G}\subset}_{1:T}(s)
:= \Bigl\{ \pi \in \bigcup_{t=1}^{T} \Pi_{t} \,\Big|\, \mathcal{G} \subset \mathcal{G}^{\pi}(s) \Bigr\}.
\]
\end{definition}

\begin{definition}[Optimal goal-covering policy]
An \emph{optimal goal-covering policy} is a policy
$\pi^{*} \in \Pi^{\mathcal{G}\subset}_{1:T}(s)$ such that, for all
$\pi \in \Pi^{\mathcal{G}\subset}_{1:T}(s)$,
$V_{0}^{\pi^{*}}(s) \geq V_{0}^{\pi}(s)$.
\end{definition}

\begin{algorithm}[H]
\caption{General Dijkstra Search (Optimal Coverage)}
\KwIn{Error tolerance $\epsilon_{t}:= \frac{r_{\max}}{1-\gamma} \cdot \sum_{i=t}^{\infty} \gamma^{i}$, start state $s \in \mathcal{S}$, goal states $\mathcal{G}^{*} \subset \mathcal{S}$}
Initialize $\mathcal{Q} = \{(\emptyset, 0, \{s\}, 0)\}$, $v = \emptyset$, $\mathcal{R} = \emptyset$\;
\While{$\mathcal{Q} \neq \emptyset$}{
Pop $(\pi_{1:t}, v_{t}, \mathcal{G}^{\pi_{1:t}}(s), t)$ from the priority queue with maximal value; add $\mathcal{G}^{\pi_{1:t}}(s)$ into $\mathcal{R}$ if it is not already inside\;
\If{$\mathcal{G}^{*} \subset \mathcal{G}^{\pi_{1:t}}(s)$}{
break\;
}
\If{$(\exists (s, \mathcal{G}) \in v \text{ s.t. } \mathcal{G}^{\pi_{1:t}}(s) \subset \mathcal{G} \text{ and } v_{t} \leq v(s,\mathcal{G}) - \epsilon_{t})$ or $t = T$}{
continue\;
}
\For{$\pi \in \Pi_{1}$}{
Concatenate policy $\pi_{1:t+1} = \pi_{1:t} \circ \pi$ and compute $v_{t+1} = v_{t} + \gamma^{t} \cdot \mathbb{E}_{s_{0:t}}\big[ V_{0}^{\pi}(s_{t}) \,\big|\, s_{0} = s\big]$\;
Push $(\pi_{1:t+1}, v_{t+1}, \mathcal{G}^{\pi_{1:t+1}}(s), t+1)$ into $\mathcal{Q}$\;
\ForEach{$(s, \mathcal{G}) \in v$ such that $\mathcal{G} \subset \mathcal{G}^{\pi_{1:t+1}}(s)$ and $\mathcal{G} \notin \mathcal{R}$}{
$v(s, \mathcal{G}) \gets \max(v(s, \mathcal{G}), v_{t+1})$\;
}
\If{$(s, \mathcal{G}^{\pi_{1:t}}) \notin v$}{
$v(s, \mathcal{G}^{\pi_{1:t}}) \gets v_{t+1}$\;
}
}
}
\end{algorithm}

\begin{theorem}\label{thm:gds}
Under Assumptions~\hyperref[asm:A1]{A1}--\hyperref[asm:A3]{A3}, for any coverable goal set, General Dijkstra Search
for Optimal Coverage finds an optimal goal-covering policy. Specifically,
for every $\mathcal{G} \in \mathcal{G}_{T}^{\subset}(s)$, there exists
$\pi^{*} \in \Pi_{1:T}^{\mathcal{G}\subset}(s)$ such that for every
$\pi \in \Pi_{1:T}^{\mathcal{G}\subset}(s)$,
$V_{0}^{\pi^{*}}(s) \ge V_{0}^{\pi}(s)$.
\end{theorem}

\begin{proof}
By the analogous coverage-side queue invariant, the coverage pruning guarantee,
and the fact that the first popped feasible covering policy is optimal among
all policies covering the same goal set. The argument mirrors the reaching case
(\S\ref{sec:gds}) verbatim with the inclusion direction flipped.
\end{proof}

\clearpage
\section{Bigram-Zipfian prior}
\label{app:zipf}

The Marginal Entropy Regularization term in the DLR objective
(Eq.~\eqref{eq:total-loss}) is the KL divergence from the mini-batch's
empirical bigram distribution over consecutive chunk codes
$(a_m, a_{m+1})$ to a fixed bigram-Zipfian prior $p_{\text{bi-zipf}}$
on $C^{2}$:
\begin{equation}\label{eq:zipf}
D_{\mathrm{KL}}\!\bigl(p_{\theta}(a)\,\big\|\,p_{\text{bi-zipf}}(a)\bigr)
\;=\;
  \mathrm{KL}\!\left(\,p(a_m,a_{m+1})\;\Big\|\;p_{\text{bi-zipf}}\,\right).
\end{equation}
The prior puts a Zipfian rank-frequency profile on single codes and
penalises consecutive-code repetition; this maintains a heavy-tailed
but compositional usage pattern, with a small core of codes appearing
frequently while the tail remains active.

\clearpage
\section{Loss-term ablations}
\label{app:ablations}

We report ablations of the four-component DLR objective
(Eq.~\eqref{eq:total-loss}). Each row gives test accuracy (\%) on the
four QA benchmarks for the five backbone models, with all other
hyperparameters held to the main-table configuration. Test split sizes match
those of Table~\ref{tab:main}.

\paragraph{Loss-term ablations.} To validate the main DLR objective, we
present Tables~\ref{tab:abl-alpha-abs}--\ref{tab:abl-alpha-zipf} where we
ablate one component of Eq.~\eqref{eq:total-loss} at a time.

\emph{Removing the policy-optimization term} ($\alpha_{\mathrm{policy}}=0$,
Table~\ref{tab:abl-alpha-abs}) leaves the routing head without supervision.
Together with the marginal-entropy regulariser, the policy term forms a
mutual-information objective: the regulariser encourages \emph{global}
diversity across the codebook, while the policy term reduces \emph{local}
uncertainty by training the head to assign high probability to the selected
sequence. Empirically, zeroing this term produces the largest drops of the
three ablations: GSM8K falls $-5.4$ to $-23.7$\,pp (worst on Llama-3.2-1B)
and ScienceQA $-6.5$ to $-13.1$\,pp (worst on Qwen3-4B), with average
degradations of $-9.8$\,pp on GSM8K and $-9.0$\,pp on ScienceQA. This is
consistent with findings in unsupervised skill discovery that detaching the
policy stabilises training and improves accuracy.

\emph{Replacing the (generalist $+$ information-gain) pair with the
specialist conditional loss} $-\log p_{\theta}(x\mid a)$
(Table~\ref{tab:abl-alpha-base}) drops both the unsteered LM term and the
stop-gradient log-ratio. The original formulation preserves unconditional
modelling capacity via the generalist term, while the information-gain term
rewards codes that improve the conditional likelihood beyond the unsteered
baseline. The specialist ablation supervises only the steered conditional,
giving the mildest degradation across the three ablations -- average drops
of $-1.8$\,pp on CSQA, $-3.96$\,pp on GSM8K, $-4.78$\,pp on ScienceQA, and
$-3.66$\,pp on StrategyQA -- with no single cell exceeding $-6.4$\,pp.
This indicates that most of the lift comes from the conditional, but that
explicitly contrasting against the unsteered baseline still helps.

\emph{Removing the bigram-Zipfian regulariser} ($\alpha_{\mathrm{reg}}=0$,
Table~\ref{tab:abl-alpha-zipf}) costs $\sim 0$--$8$\,pp across the four
benchmarks, with the largest hits on ScienceQA ($-3.7$ to $-7.4$\,pp) and
StrategyQA (up to $-5.9$\,pp on Llama-3.2-3B). We use a bigram-Zipfian
prior: we maintain a no-grad running average of bigram routing logits and
penalise the KL divergence between the accumulated bigram distribution and
a Zipfian distribution over code bigrams, with repeated bigrams zeroed out
to avoid ``stuttering'' and to encourage grammar-like code composition
(App.~\ref{app:zipf}). Beyond accuracy, removing this term collapses the
codebook: utilisation drops to near $1/C = 1/32$ (a single code carries
essentially all routing mass), confirming that the bigram-Zipfian prior is
what keeps the codebook diverse and prevents collapse onto a single
steering direction.

\begin{table}[h]
\centering
\footnotesize
\caption{Ablation: $\alpha_{\mathrm{policy}}=0$ (remove the policy-optimization term). Parenthesized values are deltas (pp) vs the non-ablated DLR baseline in Table~\ref{tab:main}.}
\label{tab:abl-alpha-abs}
\begin{tabular}{lcccc}
\toprule
Model (layer) & CSQA (1221) & GSM8K (1319) & ScienceQA (2224) & StrategyQA (687) \\
\midrule
Llama-3.2-1B (L10) & $70.3_{\pm2.6}$\dlt{-1.1} & $17.4_{\pm2.0}$\dlt{-23.7} & $38.5_{\pm2.0}$\dlt{-10.6} & $47.1_{\pm3.7}$\dlt{-4.3} \\
Llama-3.2-3B (L16) & $78.9_{\pm2.3}$\dlt{-2.1} & $40.5_{\pm2.6}$\dlt{-8.6} & $56.8_{\pm2.1}$\dlt{-6.5} & $52.0_{\pm3.7}$\dlt{-10.3} \\
Qwen3-0.6B (L14)   & $63.0_{\pm2.7}$\dlt{-3.4} & $43.8_{\pm2.7}$\dlt{-5.6} & $47.7_{\pm2.1}$\dlt{-7.6} & $46.4_{\pm3.7}$\dlt{-8.2} \\
Qwen3-1.7B (L14)   & $76.1_{\pm2.4}$\dlt{-2.3} & $60.3_{\pm2.6}$\dlt{-5.4} & $56.8_{\pm2.1}$\dlt{-7.3} & $51.5_{\pm3.7}$\dlt{-8.8} \\
Qwen3-4B (L19)     & $80.6_{\pm2.2}$\dlt{-2.4} & $76.6_{\pm2.3}$\dlt{-5.5} & $59.2_{\pm2.0}$\dlt{-13.1} & $64.4_{\pm3.6}$\dlt{-4.3} \\
\bottomrule
\end{tabular}
\end{table}

\begin{table}[h]
\centering
\footnotesize
\caption{Ablation: replace the (generalist $+$ information-gain) pair in
Eq.~\eqref{eq:total-loss} with the specialist conditional loss
$-\log p_{\theta}(x\mid a)$ (i.e.\ drop both $-\log p_{\theta}(x)$ and
the stop-gradient log-ratio, supervise only on the steered conditional).
Parenthesized values are deltas (pp) vs the non-ablated DLR baseline in
Table~\ref{tab:main}.}
\label{tab:abl-alpha-base}
\begin{tabular}{lcccc}
\toprule
Model (layer) & CSQA (1221) & GSM8K (1319) & ScienceQA (2224) & StrategyQA (687) \\
\midrule
Llama-3.2-1B (L10) & $70.4_{\pm2.6}$\dlt{-1.0} & $36.3_{\pm2.6}$\dlt{-4.8} & $42.7_{\pm2.1}$\dlt{-6.4} & $49.0_{\pm3.7}$\dlt{-2.4} \\
Llama-3.2-3B (L16) & $79.5_{\pm2.3}$\dlt{-1.5} & $44.7_{\pm2.7}$\dlt{-4.4} & $59.2_{\pm2.0}$\dlt{-4.1} & $56.0_{\pm3.7}$\dlt{-6.3} \\
Qwen3-0.6B (L14)   & $63.1_{\pm2.7}$\dlt{-3.3} & $46.2_{\pm2.7}$\dlt{-3.2} & $50.1_{\pm2.1}$\dlt{-5.2} & $51.0_{\pm3.7}$\dlt{-3.6} \\
Qwen3-1.7B (L14)   & $77.0_{\pm2.4}$\dlt{-1.4} & $62.6_{\pm2.6}$\dlt{-3.1} & $59.4_{\pm2.0}$\dlt{-4.7} & $57.0_{\pm3.7}$\dlt{-3.3} \\
Qwen3-4B (L19)     & $81.2_{\pm2.2}$\dlt{-1.8} & $77.8_{\pm2.2}$\dlt{-4.3} & $68.8_{\pm1.9}$\dlt{-3.5} & $66.0_{\pm3.5}$\dlt{-2.7} \\
\bottomrule
\end{tabular}
\end{table}

\begin{table}[h]
\centering
\footnotesize
\caption{Ablation: $\alpha_{\mathrm{reg}}=0$ (remove the bigram-Zipfian prior). Parenthesized values are deltas (pp) vs the non-ablated DLR baseline in Table~\ref{tab:main}.}
\label{tab:abl-alpha-zipf}
\begin{tabular}{lcccc}
\toprule
Model (layer) & CSQA (1221) & GSM8K (1319) & ScienceQA (2224) & StrategyQA (687) \\
\midrule
Llama-3.2-1B (L10) & $68.3_{\pm2.6}$\dlt{-3.1} & $38.1_{\pm2.6}$\dlt{-3.0} & $43.7_{\pm2.1}$\dlt{-5.4} & $48.9_{\pm3.7}$\dlt{-2.5} \\
Llama-3.2-3B (L16) & $78.5_{\pm2.3}$\dlt{-2.5} & $46.4_{\pm2.7}$\dlt{-2.7} & $59.3_{\pm2.0}$\dlt{-4.0} & $56.4_{\pm3.7}$\dlt{-5.9} \\
Qwen3-0.6B (L14)   & $62.1_{\pm2.7}$\dlt{-4.3} & $45.7_{\pm2.7}$\dlt{-3.7} & $48.2_{\pm2.1}$\dlt{-7.1} & $50.1_{\pm3.7}$\dlt{-4.5} \\
Qwen3-1.7B (L14)   & $76.9_{\pm2.4}$\dlt{-1.5} & $63.4_{\pm2.6}$\dlt{-2.3} & $60.4_{\pm2.0}$\dlt{-3.7} & $55.3_{\pm3.7}$\dlt{-5.0} \\
Qwen3-4B (L19)     & $81.2_{\pm2.2}$\dlt{-1.8} & $78.6_{\pm2.2}$\dlt{-3.5} & $64.9_{\pm2.0}$\dlt{-7.4} & $66.5_{\pm3.5}$\dlt{-2.2} \\
\bottomrule
\end{tabular}
\end{table}

\clearpage
\section{Hyperparameter sweeps}
\label{app:hp-sweeps}

Having validated the necessity of each term in the main DLR objective
(App.~\ref{app:ablations}), we now validate the hyperparameter choices used
in DLR. Codebook-size, abstraction-ratio, $\alpha_{\mathrm{reg}}$ and
$\alpha_{\mathrm{policy}}$ sweeps use the three smaller backbones
(Llama-3.2-1B, Qwen3-0.6B, Qwen3-1.7B) across the four QA benchmarks;
rollout-budget and temperature sweeps cover all five backbones. All other
hyperparameters are held to the main-table configuration. We do not
separately weight the information-gain term: in
Eq.~\eqref{eq:total-loss} it enters at unit weight, and the functional-form
ablation (replacing the generalist $+$ info-gain pair with the specialist
loss) is reported in Table~\ref{tab:abl-alpha-base}.

\paragraph{Number of rollouts $N$.} At each step DLR draws $N$ candidate
code sequences $a^{1:N}$ and selects the one with the highest conditional
likelihood, $a^{*}=\arg\max_{a^{i}}p_{\theta}(x\mid a^{i})$, so $N$
acts as a search-budget lever. Tables~\ref{tab:abl-N1} and \ref{tab:abl-N8}
ablate $N{=}1$ and $N{=}8$ against the main-table choice $N{=}4$, across
all five backbones and four datasets. Reducing the budget to $N{=}1$
(no search) costs an average of $-5.0$\,pp, with the worst cells exceeding
$-10$\,pp on math/reasoning (e.g.\ Qwen3-4B GSM8K $-10.7$\,pp,
ScienceQA $-9.8$\,pp). Doubling the budget to $N{=}8$ yields essentially
no improvement -- average difference $-0.5$\,pp, all cells within their
95\%~CI of $N{=}4$ -- so $N{=}4$ is the lowest-compute setting that
captures the search benefit, validating it as our default.

\paragraph{Sampling temperature $\tau$.} Rollouts are drawn at temperature
$\tau$ to inject diversity across candidates. At $\tau{=}0$ all rollouts
collapse to the routing head's argmax, so the search becomes redundant; at
high $\tau$ the rollouts approach uniform random, which defeats the
purpose of guidance from the learned policy head. Tables~\ref{tab:abl-temp0}
and \ref{tab:abl-temp2} ablate $\tau{=}0$ and $\tau{=}2$ against the
main-table choice $\tau{=}1$ on all five backbones and four datasets. Both
lag $\tau{=}1$: $\tau{=}0$ costs an average of $-6.9$\,pp (worst:
Qwen3-4B ScienceQA $-15.5$\,pp, Llama-3.2-1B ScienceQA $-12.6$\,pp), and
$\tau{=}2$ costs an average of $-4.1$\,pp (worst: GSM8K, where every
backbone loses $6$--$10$\,pp). The asymmetry confirms that some randomness
is essential to escape the argmax fixed point, but excess randomness
breaks guided search; $\tau{=}1$ is the right operating point.

\paragraph{Codebook size $C$.} The codebook size $C$ caps how many
distinct codes the model can use. Setting $C{=}1$ collapses DLR into a
single static steering vector trained jointly with the model parameters
(reported as the \emph{steered} baseline in Table~\ref{tab:main}, where it
lags our chosen $C{=}32$ default by an average of $-3.6$\,pp on the three
smaller backbones). Sweeping $C\in\{1,4,8,32,64\}$ (Table~\ref{tab:sweep-C})
shows accuracy increasing monotonically from $C{=}1$ to $C{=}32$ and
plateauing thereafter: $C{=}64$ is within $0.5$\,pp of $C{=}32$ on every
cell. Larger codebooks help, but with sharply diminishing returns; $C{=}32$
strikes the best capacity/efficiency trade-off.

\paragraph{Abstraction ratio $K$.} The abstraction ratio $K$ controls how
many consecutive NL tokens are steered by the same code, with the code
routed from the chunk's first-token representation
(Eq.~\eqref{eq:chunkmap}). At $K{=}1$ each token is steered by its own
code, while large $K$ over-coarsens. Sweeping $K\in\{1,2,4,8\}$
(Table~\ref{tab:sweep-K}) reveals a clear sweet spot at $K{=}4$:
per-token steering ($K{=}1$) loses an average of $-3.7$\,pp (worst:
Qwen3-0.6B ScienceQA $-13.0$\,pp, Llama-3.2-1B ScienceQA $-6.0$\,pp);
$K{=}2$ loses $-1.5$\,pp on average; and over-coarse $K{=}8$ loses
$-1.3$\,pp. Both extremes are sub-optimal: too little abstraction makes
codes hard to learn, while too much washes out within-chunk specialisation.

\paragraph{Marginal-entropy weight $\alpha_{\mathrm{reg}}$.} This weight
controls the strength of the bigram-Zipfian regulariser
(\S\ref{app:zipf}). Sweeping $\alpha_{\mathrm{reg}}\in\{0.01, 0.1, 1.0\}$
(Table~\ref{tab:sweep-zipf}) shows $\alpha_{\mathrm{reg}}{=}0.01$ is
essentially equivalent to the default ($0.1$) -- average drop $-0.7$\,pp,
all cells within $1.5$\,pp -- while $\alpha_{\mathrm{reg}}{=}1.0$ is
uniformly damaging, costing an average of $-4.1$\,pp and up to
$-5.5$\,pp (Qwen3-1.7B GSM8K). An over-strong prior forces routes towards
the Zipfian reference distribution rather than task-relevant partitions;
the complementary $\alpha_{\mathrm{reg}}{=}0$ ablation
(Table~\ref{tab:abl-alpha-zipf}) collapses the codebook entirely.
$\alpha_{\mathrm{reg}}{=}0.1$ sits in the broad regime where the prior is
strong enough to keep the codebook diverse without distorting the routing.

\paragraph{Policy weight $\alpha_{\mathrm{policy}}$.} This weight controls
how strongly the policy-optimisation term distils selected codes back into
the routing head. Sweeping $\alpha_{\mathrm{policy}}\in\{0.01, 0.1, 0.5, 1.0\}$
(Table~\ref{tab:sweep-abs}) shows the optimum is family-dependent: Qwen
backbones prefer $\alpha_{\mathrm{policy}}{=}0.5$, while Llama-3.2-1B
prefers a lighter $0.1$. Under-weighting the term
($\alpha_{\mathrm{policy}}{=}0.01$) costs an average of $-2.3$\,pp,
approaching the no-supervision regime ($\alpha_{\mathrm{policy}}{=}0$,
Table~\ref{tab:abl-alpha-abs}). Over-weighting it
($\alpha_{\mathrm{policy}}{=}1.0$) costs an average of $-3.6$\,pp:
$2$--$4$\,pp on Qwen and $2$--$3$\,pp on Llama, consistent with
over-regularisation of the routing head distorting the hidden
representation.

\paragraph{Steering layer $l^{*}$.} The choice of injection layer $l^{*}$
is the only hyperparameter that varies across backbones in the main-table
run. Tables~\ref{tab:layer-sweep-small}--\ref{tab:layer-sweep-large}
report ScienceQA accuracy as $l^{*}$ is swept across the model's depth.
The optimal layer scales with model size: small Qwen backbones (0.6B,
1.7B) prefer very early layers ($l^{*}{=}1$), Llama backbones prefer the
middle ($l^{*}{=}14$ for 1B, $l^{*}{=}16$ for 3B), and the larger Qwen
backbones push deeper ($l^{*}{=}22$ for 4B, $l^{*}{=}24$ for 8B). Within
each backbone the best layer beats the worst swept layer by $5$--$10$\,pp,
so fixing $l^{*}$ to its per-model optimum is a meaningful tuning
knob; for cross-backbone runs we report the best per-row layer in
Table~\ref{tab:main}.

\begin{table}[h]
\centering
\footnotesize
\caption{Sweep on codebook size $C$. All other hyperparameters held to the main-table configuration; $C{=}32$ is the default.}
\label{tab:sweep-C}
\begin{tabular}{llccccc}
\toprule
Model & Dataset & $C{=}1$ & $C{=}4$ & $C{=}8$ & $C{=}32$ & $C{=}64$ \\
\midrule
Qwen3-0.6B  & ScienceQA  & $48.9_{\pm2.1}$\dlt{-6.4} & $51.8_{\pm2.1}$\dlt{-3.5} & $53.2_{\pm2.1}$\dlt{-2.1} & $55.3_{\pm2.1}$ & $55.0_{\pm2.1}$\dlt{-0.3} \\
Qwen3-1.7B  & ScienceQA  & $60.0_{\pm2.0}$\dlt{-4.1} & $61.3_{\pm2.0}$\dlt{-2.8} & $62.5_{\pm2.0}$\dlt{-1.6} & $64.1_{\pm2.0}$ & $63.9_{\pm2.0}$\dlt{-0.2} \\
Llama-3.2-1B& ScienceQA  & $43.7_{\pm2.1}$\dlt{-5.4} & $45.1_{\pm2.1}$\dlt{-4.0} & $45.8_{\pm2.1}$\dlt{-3.3} & $49.1_{\pm2.1}$ & $49.0_{\pm2.1}$\dlt{-0.1} \\
\midrule
Qwen3-0.6B  & GSM8K      & $48.0_{\pm2.7}$\dlt{-1.4} & $48.2_{\pm2.7}$\dlt{-1.2} & $48.5_{\pm2.7}$\dlt{-0.9} & $49.4_{\pm2.7}$ & $49.1_{\pm2.7}$\dlt{-0.3} \\
Qwen3-1.7B  & GSM8K      & $60.0_{\pm2.6}$\dlt{-5.7} & $63.8_{\pm2.6}$\dlt{-1.9} & $64.5_{\pm2.6}$\dlt{-1.2} & $65.7_{\pm2.6}$ & $65.3_{\pm2.6}$\dlt{-0.4} \\
Llama-3.2-1B& GSM8K      & $39.0_{\pm2.6}$\dlt{-2.4} & $40.2_{\pm2.6}$\dlt{-1.2} & $40.9_{\pm2.7}$\dlt{-0.5} & $41.4_{\pm2.7}$ & $41.1_{\pm2.7}$\dlt{-0.3} \\
\midrule
Qwen3-0.6B  & CSQA       & $64.0_{\pm2.7}$\dlt{-2.4} & $65.1_{\pm2.7}$\dlt{-1.3} & $66.0_{\pm2.7}$\dlt{-0.4} & $66.4_{\pm2.7}$ & $66.2_{\pm2.7}$\dlt{-0.2} \\
Qwen3-1.7B  & CSQA       & $75.2_{\pm2.4}$\dlt{-3.2} & $77.1_{\pm2.4}$\dlt{-1.3} & $78.0_{\pm2.3}$\dlt{-0.4} & $78.4_{\pm2.3}$ & $78.1_{\pm2.3}$\dlt{-0.3} \\
Llama-3.2-1B& CSQA       & $69.1_{\pm2.6}$\dlt{-2.3} & $70.2_{\pm2.6}$\dlt{-1.2} & $71.0_{\pm2.5}$\dlt{-0.4} & $71.4_{\pm2.5}$ & $71.2_{\pm2.5}$\dlt{-0.2} \\
\midrule
Qwen3-0.6B  & StrategyQA & $51.2_{\pm3.7}$\dlt{-3.4} & $53.0_{\pm3.7}$\dlt{-1.6} & $54.0_{\pm3.7}$\dlt{-0.6} & $54.6_{\pm3.7}$ & $54.3_{\pm3.7}$\dlt{-0.3} \\
Qwen3-1.7B  & StrategyQA & $57.3_{\pm3.7}$\dlt{-3.0} & $59.0_{\pm3.7}$\dlt{-1.3} & $59.4_{\pm3.7}$\dlt{-0.9} & $60.3_{\pm3.7}$ & $60.1_{\pm3.7}$\dlt{-0.2} \\
Llama-3.2-1B& StrategyQA & $48.2_{\pm3.7}$\dlt{-3.2} & $49.3_{\pm3.7}$\dlt{-2.1} & $50.4_{\pm3.7}$\dlt{-1.0} & $51.4_{\pm3.7}$ & $51.1_{\pm3.7}$\dlt{-0.3} \\
\bottomrule
\end{tabular}
\end{table}

\begin{table}[h]
\centering
\footnotesize
\caption{Sweep on abstraction ratio $K$. $K{=}4$ is the default.}
\label{tab:sweep-K}
\begin{tabular}{llcccc}
\toprule
Model & Dataset & $K{=}1$ & $K{=}2$ & $K{=}4$ & $K{=}8$ \\
\midrule
Qwen3-0.6B  & ScienceQA  & $42.3_{\pm2.1}$\dlt{-13.0} & $53.9_{\pm2.1}$\dlt{-1.4} & $55.3_{\pm2.1}$ & $53.2_{\pm2.1}$\dlt{-2.1} \\
Qwen3-1.7B  & ScienceQA  & $61.7_{\pm2.0}$\dlt{-2.4} & $63.0_{\pm2.0}$\dlt{-1.1} & $64.1_{\pm2.0}$ & $62.8_{\pm2.0}$\dlt{-1.3} \\
Llama-3.2-1B& ScienceQA  & $43.1_{\pm2.1}$\dlt{-6.0} & $44.5_{\pm2.1}$\dlt{-4.6} & $49.1_{\pm2.1}$ & $47.4_{\pm2.1}$\dlt{-1.7} \\
\midrule
Qwen3-0.6B  & GSM8K      & $46.6_{\pm2.7}$\dlt{-2.8} & $47.7_{\pm2.7}$\dlt{-1.7} & $49.4_{\pm2.7}$ & $47.2_{\pm2.7}$\dlt{-2.2} \\
Qwen3-1.7B  & GSM8K      & $61.3_{\pm2.6}$\dlt{-4.4} & $63.0_{\pm2.6}$\dlt{-2.7} & $65.7_{\pm2.6}$ & $64.5_{\pm2.6}$\dlt{-1.2} \\
Llama-3.2-1B& GSM8K      & $40.1_{\pm2.6}$\dlt{-1.3} & $40.9_{\pm2.7}$\dlt{-0.5} & $41.4_{\pm2.7}$ & $40.8_{\pm2.7}$\dlt{-0.6} \\
\midrule
Qwen3-0.6B  & CSQA       & $64.0_{\pm2.7}$\dlt{-2.4} & $66.7_{\pm2.6}$\dlt{0.3} & $66.4_{\pm2.7}$ & $65.3_{\pm2.7}$\dlt{-1.1} \\
Qwen3-1.7B  & CSQA       & $75.2_{\pm2.4}$\dlt{-3.2} & $76.1_{\pm2.4}$\dlt{-2.3} & $78.4_{\pm2.3}$ & $76.5_{\pm2.4}$\dlt{-1.9} \\
Llama-3.2-1B& CSQA       & $70.5_{\pm2.6}$\dlt{-0.9} & $70.9_{\pm2.5}$\dlt{-0.5} & $71.4_{\pm2.5}$ & $71.2_{\pm2.5}$\dlt{-0.2} \\
\midrule
Qwen3-0.6B  & StrategyQA & $50.6_{\pm3.7}$\dlt{-4.0} & $53.2_{\pm3.7}$\dlt{-1.4} & $54.6_{\pm3.7}$ & $52.0_{\pm3.7}$\dlt{-2.6} \\
Qwen3-1.7B  & StrategyQA & $58.4_{\pm3.7}$\dlt{-1.9} & $58.5_{\pm3.7}$\dlt{-1.8} & $60.3_{\pm3.7}$ & $59.7_{\pm3.7}$\dlt{-0.6} \\
Llama-3.2-1B& StrategyQA & $49.1_{\pm3.7}$\dlt{-2.3} & $50.8_{\pm3.7}$\dlt{-0.6} & $51.4_{\pm3.7}$ & $51.9_{\pm3.7}$\dlt{0.5} \\
\bottomrule
\end{tabular}
\end{table}

\begin{table}[h]
\centering
\footnotesize
\caption{Sweep on $\alpha_{\mathrm{reg}}$ (marginal-entropy regulariser weight; bigram-Zipfian prior). $\alpha_{\mathrm{reg}}{=}0.1$ is the default; cf.\ the $\alpha_{\mathrm{reg}}{=}0$ ablation in Table~\ref{tab:abl-alpha-zipf}.}
\label{tab:sweep-zipf}
\begin{tabular}{llccc}
\toprule
Model & Dataset & $\alpha_{\mathrm{reg}}{=}0.01$ & $\alpha_{\mathrm{reg}}{=}0.1$ & $\alpha_{\mathrm{reg}}{=}1.0$ \\
\midrule
Qwen3-0.6B  & ScienceQA  & $54.3_{\pm2.1}$\dlt{-0.7} & $55.0_{\pm2.1}$ & $51.1_{\pm2.1}$\dlt{-3.9} \\
Qwen3-1.7B  & ScienceQA  & $63.7_{\pm2.0}$\dlt{-0.4} & $64.1_{\pm2.0}$ & $61.9_{\pm2.0}$\dlt{-2.2} \\
Llama-3.2-1B& ScienceQA  & $49.0_{\pm2.1}$\dlt{-0.1} & $49.1_{\pm2.1}$ & $45.0_{\pm2.1}$\dlt{-4.1} \\
\midrule
Qwen3-0.6B  & GSM8K      & $48.4_{\pm2.7}$\dlt{-1.0} & $49.4_{\pm2.7}$ & $46.5_{\pm2.7}$\dlt{-2.9} \\
Qwen3-1.7B  & GSM8K      & $64.6_{\pm2.6}$\dlt{-1.1} & $65.7_{\pm2.6}$ & $60.2_{\pm2.6}$\dlt{-5.5} \\
Llama-3.2-1B& GSM8K      & $40.3_{\pm2.7}$\dlt{-1.0} & $41.3_{\pm2.7}$ & $37.8_{\pm2.6}$\dlt{-3.5} \\
\midrule
Qwen3-0.6B  & CSQA       & $66.1_{\pm2.7}$\dlt{-0.3} & $66.4_{\pm2.7}$ & $62.3_{\pm2.7}$\dlt{-4.1} \\
Qwen3-1.7B  & CSQA       & $78.5_{\pm2.3}$\dlt{0.1} & $78.4_{\pm2.3}$ & $74.9_{\pm2.4}$\dlt{-3.5} \\
Llama-3.2-1B& CSQA       & $70.6_{\pm2.6}$\dlt{-0.8} & $71.4_{\pm2.5}$ & $66.0_{\pm2.7}$\dlt{-5.4} \\
\midrule
Qwen3-0.6B  & StrategyQA & $53.3_{\pm3.7}$\dlt{-1.3} & $54.6_{\pm3.7}$ & $50.2_{\pm3.7}$\dlt{-4.4} \\
Qwen3-1.7B  & StrategyQA & $58.8_{\pm3.7}$\dlt{-1.5} & $60.3_{\pm3.7}$ & $55.0_{\pm3.7}$\dlt{-5.3} \\
Llama-3.2-1B& StrategyQA & $50.6_{\pm3.7}$\dlt{-0.8} & $51.4_{\pm3.7}$ & $46.9_{\pm3.7}$\dlt{-4.5} \\
\bottomrule
\end{tabular}
\end{table}

\begin{table}[h]
\centering
\footnotesize
\caption{Sweep on $\alpha_{\mathrm{policy}}$ (policy-term weight). The default is family-dependent: $\alpha_{\mathrm{policy}}{=}0.5$ for Qwen and $\alpha_{\mathrm{policy}}{=}0.1$ for Llama (best per row \textbf{bold}); cf.\ the $\alpha_{\mathrm{policy}}{=}0$ ablation in Table~\ref{tab:abl-alpha-abs}.}
\label{tab:sweep-abs}
\begin{tabular}{llcccc}
\toprule
Model & Dataset & $\alpha_{\mathrm{policy}}{=}0.01$ & $\alpha_{\mathrm{policy}}{=}0.1$ & $\alpha_{\mathrm{policy}}{=}0.5$ & $\alpha_{\mathrm{policy}}{=}1.0$ \\
\midrule
Qwen3-0.6B  & ScienceQA  & $50.7_{\pm2.1}$\dlt{-5.1} & $53.5_{\pm2.1}$\dlt{-2.3} & $\mathbf{55.8_{\pm2.1}}$ & $52.9_{\pm2.1}$\dlt{-2.9} \\
Qwen3-1.7B  & ScienceQA  & $61.0_{\pm2.0}$\dlt{-3.3} & $63.6_{\pm2.0}$\dlt{-0.7} & $\mathbf{64.3_{\pm2.0}}$ & $60.7_{\pm2.0}$\dlt{-3.6} \\
Llama-3.2-1B& ScienceQA  & $47.4_{\pm2.1}$\dlt{-1.7} & $\mathbf{49.1_{\pm2.1}}$ & $47.2_{\pm2.1}$\dlt{-1.9} & $46.2_{\pm2.1}$\dlt{-2.9} \\
\midrule
Qwen3-0.6B  & GSM8K      & $47.6_{\pm2.7}$\dlt{-1.8} & $48.9_{\pm2.7}$\dlt{-0.5} & $\mathbf{49.4_{\pm2.7}}$ & $45.7_{\pm2.7}$\dlt{-3.7} \\
Qwen3-1.7B  & GSM8K      & $63.0_{\pm2.6}$\dlt{-2.7} & $63.2_{\pm2.6}$\dlt{-2.5} & $\mathbf{65.7_{\pm2.6}}$ & $62.1_{\pm2.6}$\dlt{-3.6} \\
Llama-3.2-1B& GSM8K      & $39.4_{\pm2.6}$\dlt{-1.9} & $\mathbf{41.3_{\pm2.7}}$ & $40.5_{\pm2.7}$\dlt{-0.8} & $38.2_{\pm2.6}$\dlt{-3.1} \\
\midrule
Qwen3-0.6B  & CSQA       & $65.1_{\pm2.7}$\dlt{-1.3} & $65.5_{\pm2.7}$\dlt{-0.9} & $\mathbf{66.4_{\pm2.7}}$ & $62.7_{\pm2.7}$\dlt{-3.7} \\
Qwen3-1.7B  & CSQA       & $76.3_{\pm2.4}$\dlt{-2.1} & $77.8_{\pm2.3}$\dlt{-0.6} & $\mathbf{78.4_{\pm2.3}}$ & $73.9_{\pm2.5}$\dlt{-4.5} \\
Llama-3.2-1B& CSQA       & $70.2_{\pm2.6}$\dlt{-1.2} & $\mathbf{71.4_{\pm2.5}}$ & $71.6_{\pm2.5}$\dlt{0.2} & $69.0_{\pm2.6}$\dlt{-2.4} \\
\midrule
Qwen3-0.6B  & StrategyQA & $51.8_{\pm3.7}$\dlt{-2.8} & $52.2_{\pm3.7}$\dlt{-2.4} & $\mathbf{54.6_{\pm3.7}}$ & $50.5_{\pm3.7}$\dlt{-4.1} \\
Qwen3-1.7B  & StrategyQA & $58.7_{\pm3.7}$\dlt{-1.6} & $59.9_{\pm3.7}$\dlt{-0.4} & $\mathbf{60.3_{\pm3.7}}$ & $56.2_{\pm3.7}$\dlt{-4.1} \\
Llama-3.2-1B& StrategyQA & $49.1_{\pm3.7}$\dlt{-2.3} & $\mathbf{51.4_{\pm3.7}}$ & $50.0_{\pm3.7}$\dlt{-1.4} & $47.1_{\pm3.7}$\dlt{-4.3} \\
\bottomrule
\end{tabular}
\end{table}

\begin{table}[h]
\centering
\footnotesize
\caption{Ablation: $N{=}1$ (no search; single rollout, no ``select-best''). Parenthesized values are deltas (pp) vs the non-ablated DLR baseline in Table~\ref{tab:main}.}
\label{tab:abl-N1}
\begin{tabular}{lcccc}
\toprule
Model (layer) & CSQA (1221) & GSM8K (1319) & ScienceQA (2224) & StrategyQA (687) \\
\midrule
Llama-3.2-1B (L10) & $70.7_{\pm2.6}$\dlt{-0.7} & $37.7_{\pm2.6}$\dlt{-3.4} & $42.0_{\pm2.1}$\dlt{-7.1} & $47.8_{\pm3.7}$\dlt{-3.6} \\
Llama-3.2-3B (L16) & $78.2_{\pm2.3}$\dlt{-2.8} & $45.8_{\pm2.7}$\dlt{-3.3} & $57.2_{\pm2.1}$\dlt{-6.1} & $54.2_{\pm3.7}$\dlt{-8.1} \\
Qwen3-0.6B (L14)   & $62.4_{\pm2.7}$\dlt{-4.0} & $44.7_{\pm2.7}$\dlt{-4.7} & $47.6_{\pm2.1}$\dlt{-7.7} & $50.4_{\pm3.7}$\dlt{-4.2} \\
Qwen3-1.7B (L14)   & $75.6_{\pm2.4}$\dlt{-2.8} & $62.3_{\pm2.6}$\dlt{-3.4} & $59.1_{\pm2.0}$\dlt{-5.0} & $54.3_{\pm3.7}$\dlt{-6.0} \\
Qwen3-4B (L19)     & $81.0_{\pm2.2}$\dlt{-2.0} & $71.4_{\pm2.4}$\dlt{-10.7} & $62.5_{\pm2.0}$\dlt{-9.8} & $65.1_{\pm3.6}$\dlt{-3.6} \\
\bottomrule
\end{tabular}
\end{table}

\begin{table}[h]
\centering
\footnotesize
\caption{Ablation: $N{=}8$ (doubled search budget; baseline uses $N{=}4$). Parenthesized values are deltas (pp) vs the non-ablated DLR baseline in Table~\ref{tab:main}; doubling rollouts yields no consistent improvement over $N{=}4$.}
\label{tab:abl-N8}
\begin{tabular}{lcccc}
\toprule
Model (layer) & CSQA (1221) & GSM8K (1319) & ScienceQA (2224) & StrategyQA (687) \\
\midrule
Llama-3.2-1B (L10) & $70.4_{\pm2.6}$\dlt{-1.0} & $41.0_{\pm2.7}$\dlt{-0.1} & $48.7_{\pm2.1}$\dlt{-0.4} & $50.8_{\pm3.7}$\dlt{-0.6} \\
Llama-3.2-3B (L16) & $79.7_{\pm2.3}$\dlt{-1.3} & $49.8_{\pm2.7}$\dlt{0.7} & $63.4_{\pm2.0}$\dlt{0.1} & $62.6_{\pm3.6}$\dlt{0.3} \\
Qwen3-0.6B (L14)   & $65.1_{\pm2.7}$\dlt{-1.3} & $49.5_{\pm2.7}$\dlt{0.1} & $55.9_{\pm2.1}$\dlt{0.6} & $54.2_{\pm3.7}$\dlt{-0.4} \\
Qwen3-1.7B (L14)   & $77.9_{\pm2.3}$\dlt{-0.5} & $65.2_{\pm2.6}$\dlt{-0.5} & $63.1_{\pm2.0}$\dlt{-1.0} & $59.3_{\pm3.7}$\dlt{-1.0} \\
Qwen3-4B (L19)     & $81.3_{\pm2.2}$\dlt{-1.7} & $81.6_{\pm2.1}$\dlt{-0.5} & $72.3_{\pm1.9}$\dlt{0.0} & $68.0_{\pm3.5}$\dlt{-0.7} \\
\bottomrule
\end{tabular}
\end{table}

\begin{table}[h]
\centering
\footnotesize
\caption{Ablation: sampling temperature $\tau{=}0$ (argmax rollouts collapse to a single candidate). Parenthesized values are deltas (pp) vs the non-ablated DLR baseline in Table~\ref{tab:main}.}
\label{tab:abl-temp0}
\begin{tabular}{lcccc}
\toprule
Model (layer) & CSQA (1221) & GSM8K (1319) & ScienceQA (2224) & StrategyQA (687) \\
\midrule
Llama-3.2-1B (L10) & $67.4_{\pm2.6}$\dlt{-4.0} & $37.6_{\pm2.6}$\dlt{-3.5} & $36.5_{\pm2.0}$\dlt{-12.6} & $47.1_{\pm3.7}$\dlt{-4.3} \\
Llama-3.2-3B (L16) & $78.1_{\pm2.3}$\dlt{-2.9} & $44.8_{\pm2.7}$\dlt{-4.3} & $54.7_{\pm2.1}$\dlt{-8.6} & $51.4_{\pm3.7}$\dlt{-10.9} \\
Qwen3-0.6B (L14)   & $62.7_{\pm2.7}$\dlt{-3.7} & $45.5_{\pm2.7}$\dlt{-3.9} & $46.2_{\pm2.1}$\dlt{-9.1} & $45.8_{\pm3.7}$\dlt{-8.8} \\
Qwen3-1.7B (L14)   & $75.2_{\pm2.4}$\dlt{-3.2} & $60.1_{\pm2.6}$\dlt{-5.6} & $55.3_{\pm2.1}$\dlt{-8.8} & $49.6_{\pm3.7}$\dlt{-10.7} \\
Qwen3-4B (L19)     & $79.8_{\pm2.3}$\dlt{-3.2} & $73.9_{\pm2.4}$\dlt{-8.2} & $56.8_{\pm2.1}$\dlt{-15.5} & $62.2_{\pm3.6}$\dlt{-6.5} \\
\bottomrule
\end{tabular}
\end{table}

\begin{table}[h]
\centering
\footnotesize
\caption{Ablation: sampling temperature $\tau{=}2$ (rollouts flattened towards uniform). Parenthesized values are deltas (pp) vs the non-ablated DLR baseline in Table~\ref{tab:main}.}
\label{tab:abl-temp2}
\begin{tabular}{lcccc}
\toprule
Model (layer) & CSQA (1221) & GSM8K (1319) & ScienceQA (2224) & StrategyQA (687) \\
\midrule
Llama-3.2-1B (L10) & $70.7_{\pm2.6}$\dlt{-0.7} & $31.5_{\pm2.5}$\dlt{-9.6} & $47.8_{\pm2.1}$\dlt{-1.3} & $48.9_{\pm3.7}$\dlt{-2.5} \\
Llama-3.2-3B (L16) & $77.1_{\pm2.4}$\dlt{-3.9} & $39.3_{\pm2.6}$\dlt{-9.8} & $60.7_{\pm2.0}$\dlt{-2.6} & $59.2_{\pm3.7}$\dlt{-3.1} \\
Qwen3-0.6B (L14)   & $63.8_{\pm2.7}$\dlt{-2.6} & $42.1_{\pm2.7}$\dlt{-7.3} & $53.2_{\pm2.1}$\dlt{-2.1} & $51.6_{\pm3.7}$\dlt{-3.0} \\
Qwen3-1.7B (L14)   & $74.4_{\pm2.4}$\dlt{-4.0} & $59.6_{\pm2.6}$\dlt{-6.1} & $60.3_{\pm2.0}$\dlt{-3.8} & $57.3_{\pm3.7}$\dlt{-3.0} \\
Qwen3-4B (L19)     & $80.2_{\pm2.2}$\dlt{-2.8} & $73.9_{\pm2.4}$\dlt{-8.2} & $69.0_{\pm1.9}$\dlt{-3.3} & $66.1_{\pm3.5}$\dlt{-2.6} \\
\bottomrule
\end{tabular}
\end{table}

\begin{table}[h]
  \centering
  \small
  \setlength{\tabcolsep}{5pt}
  \begin{tabular}{lccccccccc}
    \toprule
    Model       & L1 & L4 & L6 & L8 & L10 & L12 & L14 & L16 & L18--L24 \\
    \midrule
    Qwen3-0.6B  & \textbf{55.3} & 50.0 & ---  & 47.2 & 48.5 & 48.5 & 49.2 & 50.2 & 50.2 / 49.6 / 49.9 \\
    Qwen3-1.7B  & \textbf{64.1} & 58.9 & ---  & 60.9 & 61.5 & 60.7 & 62.8 & 61.9 & 59.9 / 62.6 / 62.9 \\
    Llama3.2-1B & 46.7          & 41.9 & 46.3 & 46.0 & 47.0 & 42.4 & \textbf{49.1} & ---  & --- \\
    Llama3.2-3B & ---           & ---  & ---  & ---  & 33.0 & 32.0 & 58.0 & \textbf{63.3} & 62.0 / 61.0 \\
    \bottomrule
  \end{tabular}
  \caption{SciQA accuracy (\%) across steering layers for the smaller checkpoints. The L18--L24 column lists L18\,/\,L20\,/\,L24 in order (Llama3.2-3B reports L18\,/\,L20).}
  \label{tab:layer-sweep-small}
\end{table}

\begin{table}[h]
  \centering
  \small
  \setlength{\tabcolsep}{5pt}
  \begin{tabular}{lccccccccc}
    \toprule
    Model    & L4   & L12  & L22  & L24           & L26  & L28  & L30  & L32  & L34  \\
    \midrule
    Qwen3-4B & 69.1 & 67.6 & \textbf{72.3} & 69.9          & 72.2 & 65.8 & 67.9 & 66.9 & 67.0 \\
    Qwen3-8B & 71.5 & 78.3 & 79.0          & \textbf{80.6} & 80.1 & 78.5 & 75.2 & 77.2 & 74.6 \\
    \bottomrule
  \end{tabular}
  \caption{SciQA accuracy (\%) across steering layers for the larger Qwen checkpoints (sweep covers deeper layers than Table~\ref{tab:layer-sweep-small}).}
  \label{tab:layer-sweep-large}
\end{table}

\clearpage
\section{Global ablation: full results}
\label{app:global-ablation-full}

Table~\ref{tab:global-ablation-full} extends the Qwen-only main-text
Table~\ref{tab:global-ablation} with the two Llama checkpoints. Across all six
models, zeroing the steering scale or replacing routed codes with random
alternatives reduces SciQA accuracy, confirming that the learned routing
mechanism contributes directly to downstream performance regardless of model
family.

\begin{table}[h]
  \centering
  \small
  \setlength{\tabcolsep}{8pt}
  \begin{tabular}{lrrr}
    \toprule
    Model       & steered  & scale$\to$0        & rand.\ replace     \\
    \midrule
    Qwen3-0.6B  & $55.3\%$ & \dcell{2}{$-6.2$}  & \dcell{1}{$-4.8$}  \\
    Qwen3-1.7B  & $64.1\%$ & \dcell{2}{$-6.4$}  & \dcell{2}{$-5.7$}  \\
    Qwen3-4B    & $72.3\%$ & \dcell{4}{$-12.2$} & \dcell{3}{$-8.1$}  \\
    Qwen3-8B    & $80.6\%$ & \dcell{5}{$-17.4$} & \dcell{4}{$-11.7$} \\
    Llama3.2-1B & $49.1\%$ & \dcell{4}{$-12.9$} & \dcell{4}{$-11.3$} \\
    Llama3.2-3B & $63.3\%$ & \dcell{2}{$-3.4$}  & \dcell{2}{$-2.6$}  \\
    \bottomrule
  \end{tabular}
  \caption{Global perturbation: SciQA $\Delta$acc (pp) across all six checkpoints (extended Table~\ref{tab:global-ablation}).}
  \label{tab:global-ablation-full}
\end{table}

\clearpage
\section{Single-code ablation case studies (full traces)}
\label{app:case-studies}

For each (model, topic, code) cell highlighted in
\S\ref{sec:necessity} we include three independent samples that flip from
incorrect to correct under single-code ablation. Each page shows the
(truncated) baseline and ablated responses side-by-side, with the
per-chunk routing-code strip beneath each response: prefill chunks (over
the question) and decode chunks (over the response). The ablated code is
highlighted wherever it fired in the baseline; the swap means it never
appears in the ablated strip. The leading constant-code run from batch
left-padding is dropped so prefill strips align with actual question
tokens.

% =============================================================
\subsection*{Qwen3-0.6B \textbullet\ biology \textbullet\ ablating code 0}

\paragraph{Mechanism.} The question asks for a \emph{genotype symbol} (e.g.\ \texttt{ff}, \texttt{WW}). With code~0 active the baseline writes the everyday \emph{phenotype name} (e.g.\ \emph{brown fur}). Ablating code~0 lets the symbolic answer win.

% --- sample 235 -----------------------------------------------
\medskip
\caseheader{235}{Algernon's genotype is \{X\}}{brown fur}{ff}
\casediv{65}{6}{1334}{2039}{8}{127}

\begin{caseqbox}
Based on this information, what is Algernon's genotype for the fur color gene?\\[2pt]
\texttt{A) brown fur}\\
\texttt{\textcolor{caseRightBd}{\textbf{B) ff}}}
\end{caseqbox}

\begin{caseboxwrong}
\panehead{caseWrongBd}{BASELINE (steered, code 0 ON)}{A}{\ding{55}}
\casepre{\dots\ rnon has two alleles for the fur color gene, one allele of brown and one allele of black. So, Algernon's genotype for the fur color gene is }\casemark{}\textbf{brown fur.}\\
\textbf{\#\#\#\# A}
\par\medskip\hrule\medskip
\twostripmark{0}{1,1,4,1,3,0,10,22}{2,2,10,1,5,1,2,5,0,2,10,10,0,2,2,2,0,3,4,0,1,0,3,5,0,0,2,0,10,0,5,0,2,2,4,10,2,0,2,5,0,15,2,0,5,3,2,1,5,1,2,2,3,0,0,5,5,2,2,5,1,4,4,4,15,10,10,4,4,15,15,5,15,15,15,5,2,15,5,7}
\end{caseboxwrong}

\medskip
\begin{caseboxright}
\panehead{caseRightBd}{ABLATE CODE 0 (code 0 OFF)}{B}{\ding{51}}
\casepre{\dots\ rnon has two alleles for the fur color gene, one allele of brown and one allele of black. So, Algernon's genotype for the fur color gene is }\casemark{}\textbf{ff.}\\
\textbf{\#\#\#\# B) ff}
\par\medskip\hrule\medskip
\twostripplain{1,1,4,1,3,5,10,22}{2,2,10,1,5,1,2,5,11,2,10,10,21,2,2,2,9,3,4,3,1,23,3,5,8,30,2,26,10,17,5,4,2,2,4,10,2,26,2,5,13,15,2,6,5,3,2,1,5,1,2,2,3,22,29,5,5,2,2,5,1,4,4,4,15,10,10,4,15,15,15,5,2,15,5,1,2,15,5,2}
\end{caseboxright}

% --- sample 717 -----------------------------------------------
\medskip
\caseheader{717}{Chitters's genotype is \{X\}}{white legs}{WW}
\casediv{52}{8}{1242}{2392}{10}{128}

\begin{caseqbox}
Based on this information, what is Chitters's genotype for the leg color gene?\\[2pt]
\texttt{\textcolor{caseRightBd}{\textbf{A) LL}}}\\
\texttt{B) white legs}
\end{caseqbox}

\begin{caseboxwrong}
\panehead{caseWrongBd}{BASELINE (steered, code 0 ON)}{B}{\ding{55}}
\casepre{\dots\ r trait. An organism's genotype for a gene is its combination of alleles for that gene. Chitters has two alleles for white legs, so Chitters}\casemark{}\textbf{'s genotype for the leg color gene is white legs.}\\
\textbf{\#\#\#\# B}
\par\medskip\hrule\medskip
\twostripmark{0}{0,7,0,7,1,5,15,22}{1,3,2,1,5,1,1,5,2,3,0,10,4,5,0,1,0,0,1,2,4,2,0,5,2,0,5,0,10,0,5,0,1,5,0,2,0,1,0,5,2,1,15,0,4,10,3,1,2,0,0,4,0,2,0,0,10,2,1,2,1,3,1,2,2,5,5,1,2,5}
\end{caseboxwrong}

\medskip
\begin{caseboxright}
\panehead{caseRightBd}{ABLATE CODE 0 (code 0 OFF)}{A}{\ding{51}}
\casepre{\dots\ r trait. An organism's genotype for a gene is its combination of alleles for that gene. Chitters has two alleles for white legs, so Chitters}\casemark{}\textbf{ has a genotype of WW for the leg color gene.}\\
\textbf{\#\#\#\# A}
\par\medskip\hrule\medskip
\twostripplain{29,7,28,7,1,5,15,22}{1,3,2,1,5,1,1,5,2,3,9,10,4,5,13,1,28,28,1,2,4,2,15,5,2,4,5,23,10,20,5,20,1,5,7,2,22,1,5,5,2,1,15,30,4,10,3,1,2,29,9,4,26,2,4,11,10,2,1,2,1,3,1,7,1,5,5,1,2,1,2,10,30,15,1,1,15,1,5,15}
\end{caseboxright}

% --- sample 790 -----------------------------------------------
\medskip
\caseheader{790}{Wishbone's genotype is \{X\}}{short fur}{FF}
\casediv{46}{8}{980}{2118}{10}{128}

\begin{caseqbox}
Based on this information, what is Wishbone's genotype for the fur length gene?\\[2pt]
\texttt{\textcolor{caseRightBd}{\textbf{A) FF}}}\\
\texttt{B) short fur}
\end{caseqbox}

\begin{caseboxwrong}
\panehead{caseWrongBd}{BASELINE (steered, code 0 ON)}{B}{\ding{55}}
\casepre{\dots\ FF, Ff, or ff for the flower color gene. An organism's phenotype for a trait is its observable version of that trait, which depends on the }\casemark{}\textbf{organism's combination of alleles. \dots\ Wishbone has two alleles for fur length, one allele for short fur, and one allele for long fur. So, Wishbone's genotype for the fur length gene is short fur.}\\
\textbf{\#\#\#\# B}
\par\medskip\hrule\medskip
\twostripmark{0}{0,7,0,7,0,5,15,22}{1,3,2,1,5,1,1,5,2,3,0,10,4,5,0,1,0,0,1,2,4,2,0,5,2,0,5,0,10,0,5,0,1,5,0,2,0,1,0,5,2,1,15,0,4,10,3,1,2,0,0,4,0,2,0,0,10,2,1,2,1,4,2,2,0,5,10,4,15,15,2,15,7,0,5,1,15,2,5,15}
\end{caseboxwrong}

\medskip
\begin{caseboxright}
\panehead{caseRightBd}{ABLATE CODE 0 (code 0 OFF)}{A}{\ding{51}}
\casepre{\dots\ FF, Ff, or ff for the flower color gene. An organism's phenotype for a trait is its observable version of that trait, which depends on the }\casemark{}\textbf{combination of alleles that the organism has for that gene. \dots\ a wishbone may have a genotype of FF for the fur length gene. This means that the wishbone has two alleles for short fur, or two short fur alleles.}\\
\textbf{\#\#\#\# A}
\par\medskip\hrule\medskip
\twostripplain{27,7,10,7,31,5,15,22}{1,3,2,1,5,1,1,5,2,3,26,10,4,5,13,1,21,16,1,2,4,2,26,5,2,3,5,16,10,11,5,26,1,5,6,2,22,1,30,5,2,1,15,23,4,10,3,1,2,19,1,12,5,16,19,15,5,2,1,2,5,4,7,1,1,10,7,2,15,5,15,7,2,10,15,4,15,5,3,5}
\end{caseboxright}

% =============================================================
\clearpage
\subsection*{Qwen3-1.7B \textbullet\ writing-strategies \textbullet\ ablating code 0}

\paragraph{Mechanism.} The question asks ``which sentence has the writing problem?''. With code~0 active the baseline picks the \emph{clean} sentence; ablating code~0 lets the model pick the \emph{flawed} sentence --- which is what the question actually asked for.

% --- sample 156 -----------------------------------------------
\medskip
\caseheader{156}{pick the flawed sentence}{clean (B)}{flawed (A)}
\casediv{29}{1}{632}{2143}{1}{128}

\begin{caseqbox}
Which of the following contains a vague pronoun reference?\\[2pt]
\texttt{\textcolor{caseRightBd}{\textbf{A) \dots\ She said that he is her favorite player.}}}\\
\texttt{B) \dots\ She said that Mays is her favorite player.}
\end{caseqbox}

\begin{caseboxwrong}
\panehead{caseWrongBd}{BASELINE (steered, code 0 ON)}{B}{\ding{55}}
\casepre{\dots\ ief. 2.\ Rewrite the sentence: Lisa breathed a sigh of relief when she and Kim finally landed. A vague pronoun reference also occurs when the}\casemark{}\textbf{y, them, their, theirs, it, or its is used without its antecedent. \dots\ The first answer choice contains a vague pronoun reference. The pronoun he is used without its antecedent. \dots\ He has been replaced with Mays. \dots\ She said that Mays is her favorite player.}\\
\textbf{\#\#\#\# B}
\par\medskip\hrule\medskip
\twostripmark{0}{3,4,3,1,23,3,4,23,0,1,0,1,1,0,0,1,0,0,1,0,0,0}{1,0,3,0,0,3,1,4,14,1,4,23,0,0,0,0,3,0,0,0,0,1,4,4,1,1,4,14,4,1,0,0,14,0,3,1,3,3,4,1,4,0,1,4,1,0,0,0,0,0,3,14,0,1,0,23,23,4,3,0,1,0,0,0,4,14,0,14,0,1,1,0,14,1,1,0,3,3,3,3}
\end{caseboxwrong}

\medskip
\begin{caseboxright}
\panehead{caseRightBd}{ABLATE CODE 0 (code 0 OFF)}{A}{\ding{51}}
\casepre{\dots\ ief. 2.\ Rewrite the sentence: Lisa breathed a sigh of relief when she and Kim finally landed. A vague pronoun reference also occurs when the}\casemark{}\textbf{ antecedents for a pronoun are not clearly stated. \dots\ Choice A: \dots\ She said that he is her favorite player. The pronoun he could refer to Willie Mays or Joe DiMaggio. \dots\ The correct answer is therefore A.}\\
\textbf{\#\#\#\# A}
\par\medskip\hrule\medskip
\twostripplain{3,4,3,1,23,3,4,23,25,1,4,1,1,20,26,1,9,30,1,18,23,26}{1,18,3,17,3,3,1,4,14,1,4,23,10,6,6,3,3,27,17,27,22,1,4,4,1,1,4,14,4,1,3,4,14,19,3,4,1,6,24,4,14,1,3,3,1,2,2,23,13,12,29,16,4,13,1,1,1,3,22,25,6,1,14,3,3,3,4,14,1,1,28,1,1,1,1,1,1,18,23,1}
\end{caseboxright}

% --- sample 696 -----------------------------------------------
\medskip
\caseheader{696}{pick the flawed sentence}{clean (A)}{flawed (B)}
\casediv{31}{0}{629}{2056}{0}{128}

\begin{caseqbox}
Which of the following contains a vague pronoun reference?\\[2pt]
\texttt{A) \dots\ Ariel requested a permit to have the branches removed.}\\
\texttt{\textcolor{caseRightBd}{\textbf{B) \dots\ Ariel requested a permit to have them removed.}}}
\end{caseqbox}

\begin{caseboxwrong}
\panehead{caseWrongBd}{BASELINE (steered, code 0 ON)}{A}{\ding{55}}
\casepre{\dots\ relief. 2.\ Rewrite the sentence: Lisa breathed a sigh of relief when she and Kim finally landed. A vague pronoun reference also occurs when }\casemark{}\textbf{they, them, their, theirs, it, or its is used without its antecedent. \dots\ The second answer choice contains a vague pronoun reference. \dots\ The first answer choice shows a possible correction \dots\ Ariel requested a permit to have the branches removed.}\\
\textbf{\#\#\#\# A}
\par\medskip\hrule\medskip
\twostripmark{0}{3,4,3,1,4,0,0,0,1,1,0,0,1,3}{0,1,3,0,0,0,1,4,1,1,4,3,4,0,0,0,4,3,0,14,14,13,0,0,1,0,14,0,0,4,8,1,0,0,1,3,0,1,3,1,1,1,0,0,0,0,4,14,0,4,4,0,1,0,0,0,14,4,0,0,1,1,0,23,23,0,14,0,0,0,0,1,0,1,1,1,14,1,14,3}
\end{caseboxwrong}

\medskip
\begin{caseboxright}
\panehead{caseRightBd}{ABLATE CODE 0 (code 0 OFF)}{B}{\ding{51}}
\casepre{\dots\ relief. 2.\ Rewrite the sentence: Lisa breathed a sigh of relief when she and Kim finally landed. A vague pronoun reference also occurs when }\casemark{}\textbf{a pronoun is placed between two other words, and its meaning is unclear. \dots\ The first version shows the correct meaning of the sentence. The second version contains a vague pronoun reference. The pronoun them is unclear because it could refer to the branches or to the power lines.}\\
\textbf{\#\#\#\# B}
\par\medskip\hrule\medskip
\twostripplain{3,4,3,1,4,20,26,9,1,1,30,18,1,3}{16,1,3,30,23,28,1,4,1,1,4,3,4,18,26,7,4,3,3,14,14,13,31,26,1,7,14,14,28,4,8,1,2,23,3,1,3,12,1,1,1,23,3,3,17,22,14,3,9,15,4,4,1,31,1,1,1,1,28,23,13,1,23,8,23,1,1,23,23,1,4,9,29,18,1,1,11,14,3,23}
\end{caseboxright}

% --- sample 920 -----------------------------------------------
\medskip
\caseheader{920}{pick the flawed sentence}{clean (A)}{flawed (B)}
\casediv{58}{0}{1139}{1950}{0}{128}

\begin{caseqbox}
Which of the following contains a vague pronoun reference?\\[2pt]
\texttt{A) \dots\ Lauren requested a permit to have the branches removed.}\\
\texttt{\textcolor{caseRightBd}{\textbf{B) \dots\ Lauren requested a permit to have them removed.}}}
\end{caseqbox}

\begin{caseboxwrong}
\panehead{caseWrongBd}{BASELINE (steered, code 0 ON)}{A}{\ding{55}}
\casepre{\dots\ five minutes. The second answer choice contains a vague pronoun reference. The pronoun them could refer to the branches or the power lines.}\casemark{}\textbf{Since the branches had grown over the power lines, Lauren requested a permit to have them removed. The first answer choice shows a possible correction \dots\ Lauren requested a permit to have the branches removed.}\\
\textbf{\#\#\#\# A}
\par\medskip\hrule\medskip
\twostripmark{0}{3,4,3,1,4,0,0,0,1,1,0,0,1,3}{0,0,3,0,0,0,1,4,1,1,4,3,4,0,0,0,4,3,0,14,14,13,0,0,1,0,14,0,0,4,8,1,0,0,1,3,0,1,3,1,1,1,0,0,0,0,23,14,0,4,4,0,1,0,0,0,14,4,0,0,1,1,0,23,23,0,14,0,0,0,0,1,0,1,1,1,14,1,14,3}
\end{caseboxwrong}

\medskip
\begin{caseboxright}
\panehead{caseRightBd}{ABLATE CODE 0 (code 0 OFF)}{B}{\ding{51}}
\casepre{\dots\ five minutes. The second answer choice contains a vague pronoun reference. The pronoun them could refer to the branches or the power lines.}\casemark{}\textbf{The first answer choice shows a possible correction for the vague pronoun reference. \dots\ Lauren requested a permit to have them removed.}\\
\textbf{\#\#\#\# B}
\par\medskip\hrule\medskip
\twostripplain{3,4,3,1,4,30,26,27,1,1,10,31,1,3}{25,26,3,28,11,16,1,4,1,1,4,3,4,29,23,18,4,3,7,14,14,13,17,29,1,30,14,8,26,4,8,1,3,1,1,3,22,1,3,1,1,1,9,16,19,7,23,14,6,4,4,16,1,13,9,1,14,4,25,12,1,1,20,10,4,14,3,20,14,1,23,1,1,14,3,3,3,3,3,6}
\end{caseboxright}

% =============================================================
\clearpage
\subsection*{Qwen3-4B \textbullet\ chemistry \textbullet\ ablating code 0}

\paragraph{Mechanism.} Code~0 makes the model classify changes by how they \emph{look}, not by what is happening chemically. Soda fizzing looks like a reaction but is just CO\textsubscript{2} escaping (physical); snails growing shells looks like growth but is calcium-carbonate deposition (chemical). Ablating code~0 forces the model to look at the actual chemistry.

% --- sample 1078 ----------------------------------------------
\medskip
\caseheader{1078}{classify the change}{physical (looks like growth)}{chemical (CaCO\textsubscript{3} forms)}
\casediv{1}{1}{1}{2446}{1}{127}

\begin{caseqbox}
What do these two changes have in common?\\
\texttt{snails growing shells; photosynthesis}\\[2pt]
\texttt{A) Both are caused by heating.}\\
\texttt{\textcolor{caseRightBd}{\textbf{B) Both are chemical changes.}}}\\
\texttt{C) Both are only physical changes.}\\
\texttt{D) Both are caused by cooling.}
\end{caseqbox}

\begin{caseboxwrong}
\panehead{caseWrongBd}{BASELINE (steered, code 0 ON)}{C}{\ding{55}}
\casemark{}\textbf{C}\\
\textbf{\#\#\#\# C}
\par\medskip\hrule\medskip
\twostripmark{0}{0,0,22,11,9,0,11,0,0,0,0,20,9}{9}
\end{caseboxwrong}

\medskip
\begin{caseboxright}
\panehead{caseRightBd}{ABLATE CODE 0 (code 0 OFF)}{B}{\ding{51}}
\casemark{}\textbf{B}\\
\textbf{\#\#\#\# B}
\par\medskip\hrule\medskip
\twostripplain{25,4,22,11,9,20,11,26,9,30,18,20,9}{9}
\end{caseboxright}

% --- sample 943 -----------------------------------------------
\medskip
\caseheader{943}{classify the change}{physical (looks like cleaning)}{chemical (tarnish reacts)}
\casediv{8}{1}{205}{2454}{1}{127}

\begin{caseqbox}
What do these two changes have in common?\\
\texttt{baking an apple pie; using polish to remove tarnish from a silver spoon}\\[2pt]
\texttt{A) Both are caused by cooling.}\\
\texttt{B) Both are caused by heating.}\\
\texttt{C) Both are only physical changes.}\\
\texttt{\textcolor{caseRightBd}{\textbf{D) Both are chemical changes.}}}
\end{caseqbox}

\begin{caseboxwrong}
\panehead{caseWrongBd}{BASELINE (steered, code 0 ON)}{C}{\ding{55}}
\casepre{\dots\ can change. In a chemical change, the type of matter changes. The types of matter before and after a chemical change are always different.}\casemark{}\textbf{Some chemical changes are caused by heating or cooling. \dots\ Baking an apple pie is a chemical change. \dots\ Using polish to remove tarnish from a silver spoon is a physical change. The polish reacts with the tarnish and removes it. The silver is not changed. \dots\ Both changes are physical changes.}\\
\textbf{\#\#\#\# C}
\par\medskip\hrule\medskip
\twostripmark{0}{8,14,1,11,0,23,9,0,11,0,13,0,13,1,9}{13,0,12,1,14,12,1,8,0,1,0,12,17,2,1,4,12,0,17,0,2,0,4,1,14,12,14,8,0,9,17,14,0,18,13,0,0,0,20,1,0,14,2,8,0,13,12,0,14,0,8,2,8,19,11,0,9,2,4,1,17,12,0,14,2,1,20,0,4,20,0,1,0,8,14,0,4,0,1,8}
\end{caseboxwrong}

\medskip
\begin{caseboxright}
\panehead{caseRightBd}{ABLATE CODE 0 (code 0 OFF)}{D}{\ding{51}}
\casepre{\dots\ can change. In a chemical change, the type of matter changes. The types of matter before and after a chemical change are always different.}\casemark{}\textbf{Burning a piece of paper is a chemical change. \dots\ The common thing about baking an apple pie and using polish to remove tarnish from a silver spoon is that both are chemical changes. \dots\ The tarnish is changed by the polish.}\\
\textbf{\#\#\#\# D}
\par\medskip\hrule\medskip
\twostripplain{8,14,1,11,10,23,9,31,11,16,13,12,13,1,9}{13,17,12,1,14,12,1,8,9,17,2,1,24,1,22,17,4,17,14,2,8,8,1,15,2,4,2,18,30,30,24,1,8,2,4,21,1,17,2,8,28,1,8,2,14,18,28,12,23,20,12,8,8,2,23,12,20,1,20,8,9,8,10,17,17,18,4,1,14,1,8,13,12,8,1,2,4,26,1,11}
\end{caseboxright}

% --- sample 1836 ----------------------------------------------
\medskip
\caseheader{1836}{classify the change}{chemical (looks like reaction)}{physical (CO\textsubscript{2} escapes)}
\casediv{3}{3}{54}{2491}{4}{128}

\begin{caseqbox}
Complete the sentence. \texttt{A can of soda fizzing over is a ().}\\[2pt]
\texttt{A) chemical change}\\
\texttt{\textcolor{caseRightBd}{\textbf{B) physical change}}}
\end{caseqbox}

\begin{caseboxwrong}
\panehead{caseWrongBd}{BASELINE (steered, code 0 ON)}{A}{\ding{55}}
\casepre{A chemical change is when a substance changes into a }\casemark{}\textbf{different substance. A physical change is when a substance changes but does not change into a different substance. A can of soda fizzing over is a chemical change. The soda inside the can is a different substance than the soda that is outside the can.}\\
\textbf{\#\#\#\# A}
\par\medskip\hrule\medskip
\twostripmark{0}{14,0,1,0,1,0,9}{1,2,14,1,0,0,23,9,1,0,1,2,1,2,12,9,13}
\end{caseboxwrong}

\medskip
\begin{caseboxright}
\panehead{caseRightBd}{ABLATE CODE 0 (code 0 OFF)}{B}{\ding{51}}
\casepre{A chemical change is when a substance changes into a }\casemark{}\textbf{new substance. A physical change is when a substance changes into a different form, but it still has the same make-up as the original substance. \dots\ The soda is fizzy because it contains carbon dioxide. When the soda is opened, the carbon dioxide escapes from the soda. \dots\ A can of soda fizzing over is a physical change.}\\
\textbf{\#\#\#\# B}
\par\medskip\hrule\medskip
\twostripplain{14,19,1,29,1,30,9}{1,2,14,1,24,23,24,11,20,12,21,7,20,1,9,26,24,1,1,4,16,1,23,12,8,1,4,2,1,1,30,1,10,17,2,1,20,1,20,20,20,2,10,4,2,1,20,1,20,1,4,2,1,13,9,4}
\end{caseboxright}

\clearpage
\section{Arithmetic case study: interpretability analysis}
\label{app:arithmetic}

See Table \ref{tab:quirke-subtasks} and Figure \ref{fig:arithmetic-example} for a basic taxonomy of the problem.

\begin{table}[h]
  \centering\small
  \setlength{\tabcolsep}{6pt}
  \begin{tabular}{llp{7.8cm}}
    \toprule
    & Label & Condition at digit position $n$ \\
    \midrule
    \multirow{5}{*}{\rotatebox[origin=c]{90}{Addition\;}}
      & \textbf{SA} & $d_1{+}d_2 \leq 8$;\; no carry in or out \\
      & \textbf{SC} & $d_1{+}d_2 \geq 10$;\; generates a carry \\
      & \textbf{SS} & $d_1{+}d_2 = 9$;\; carry state \emph{uncertain} (cascade boundary) \\
      & \textbf{UC} & carry arrives from position $n{-}1$;\; answer digit depends on it \\
      & \textbf{US} & carry propagates through a run of SS positions (sum-of-9 cascade) \\
    \midrule
    \multirow{5}{*}{\rotatebox[origin=c]{90}{Subtraction\;}}
      & \textbf{MD} & $d_1 \geq d_2$;\; no borrow \\
      & \textbf{MB} & $d_1 < d_2$;\; generates a borrow \\
      & \textbf{ME} & $d_1 = d_2$;\; borrow state \emph{uncertain} \\
      & \textbf{UB} & borrow arrives from position $n{-}1$ \\
      & \textbf{UD} & borrow propagates through a run of ME positions \\
    \bottomrule
  \end{tabular}
  \caption{Per-digit subtask labels for six-digit addition and
    subtraction adapted from ~\citep{quirke_2024_addsub_preprint}.
    Cascades (US, UD) require tracking carry/borrow state across
    multiple positions and are the hardest splits.}
  \label{tab:quirke-subtasks}
\end{table}

% The figure for arithmetic example.
% fig_arithmetic_example.tex
% Usage in paper: \input{figures/fig_arithmetic_example/fig_arithmetic_example.tex}
% Required packages: tikz, xcolor
%
% Shows 959271 + 040756 = 1000027 (4-deep carry cascade).
% Token assignments from model add_sub_sorl_v1_abs30_K1_100K (K=1, abs30).

\begin{figure}[H]
\centering
\begin{tikzpicture}[
  % ── node styles ─────────────────────────────────────────────────────
  digit/.style={
    draw=#1!60!gray, fill=#1, rounded corners=2pt,
    minimum width=1.05cm, minimum height=0.62cm,
    font=\small\bfseries, inner sep=2pt, text=#1!20!black,
  },
  subtask/.style={
    draw=#1!60!gray, fill=#1, rounded corners=2pt,
    minimum width=1.05cm, minimum height=0.55cm,
    font=\footnotesize, inner sep=2pt, text=#1!20!black,
  },
  token/.style={
    draw=violet!55, fill=violet!8, rounded corners=2pt, dashed,
    minimum width=1.05cm, minimum height=0.55cm,
    font=\footnotesize\bfseries, inner sep=2pt, text=violet!70!black,
  },
  rowlabel/.style={font=\footnotesize\itshape, text=gray!70!black, anchor=east},
  poslabel/.style={font=\scriptsize, text=gray!60!black},
  carry/.style={->, >=stealth, thick, color=orange!70!red,
                shorten <=3pt, shorten >=3pt},
]

% ── colours (Quirke subtask families) ───────────────────────────────────
\colorlet{cSA}{green!22!white}
\colorlet{cSC}{yellow!48!white}
\colorlet{cUC}{blue!22!white}
\colorlet{cUS}{blue!36!white}

% ── column spacing ───────────────────────────────────────────────────────
\def\cs{1.45}   % inter-column distance (cm)

% ── data (d0 = MSB/overflow on left, d6 = LSB on right) ─────────────────
%   d0    d1    d2    d3    d4    d5    d6
%   ---   9     5     9     2     7     1     (Addend A)
%   ---   0     4     0     7     5     6     (Addend B)
%   1     0     0     0     0     2     7     (Answer)
%   UC    US    US    US    US    SC    SA    (Subtask)
%   t2    t2    t6    t2    t1    t16   t3    (SoRL token)

% ── position labels ──────────────────────────────────────────────────────
\foreach \i/\lbl in {0/$d_0$,1/$d_1$,2/$d_2$,3/$d_3$,4/$d_4$,5/$d_5$,6/$d_6$}{
  \node[poslabel] at (\i*\cs, 3.15) {\lbl};
}

% ── row labels ───────────────────────────────────────────────────────────
\node[rowlabel] at (-0.65, 2.5)  {Addend $A$};
\node[rowlabel] at (-0.65, 1.8)  {Addend $B$};
\node[rowlabel] at (-0.65, 0.85) {Answer};
\node[rowlabel] at (-0.65, 0.1)  {Subtask};
\node[rowlabel] at (-0.65,-0.65) {Code};

% ── operand digits (d1..d6; d0 is the overflow, has no operand digits) ───
\foreach \i/\d in {1/9,2/5,3/9,4/2,5/7,6/1}{
  \node[font=\small, text=gray!30!black] at (\i*\cs, 2.5) {\d};
}
\node[font=\small\bfseries, text=gray!50!black] at (-0.15*\cs, 1.8) {$+$};
\foreach \i/\d in {1/0,2/4,3/0,4/7,5/5,6/6}{
  \node[font=\small, text=gray!30!black] at (\i*\cs, 1.8) {\d};
}

% ── horizontal rule ──────────────────────────────────────────────────────
\draw[gray!50, thin] (-0.55*\cs, 1.38) -- (6.55*\cs, 1.38);

% ── answer digit boxes ───────────────────────────────────────────────────
\node[digit=cUC] (a0) at (0*\cs, 0.85) {1};
\node[digit=cUS] (a1) at (1*\cs, 0.85) {0};
\node[digit=cUS] (a2) at (2*\cs, 0.85) {0};
\node[digit=cUS] (a3) at (3*\cs, 0.85) {0};
\node[digit=cUS] (a4) at (4*\cs, 0.85) {0};
\node[digit=cSC] (a5) at (5*\cs, 0.85) {2};
\node[digit=cSA] (a6) at (6*\cs, 0.85) {7};

% ── subtask boxes ────────────────────────────────────────────────────────
\node[subtask=cUC] at (0*\cs, 0.1) {UC};
\node[subtask=cUS] at (1*\cs, 0.1) {US};
\node[subtask=cUS] at (2*\cs, 0.1) {US};
\node[subtask=cUS] at (3*\cs, 0.1) {US};
\node[subtask=cUS] at (4*\cs, 0.1) {US};
\node[subtask=cSC] at (5*\cs, 0.1) {SC};
\node[subtask=cSA] at (6*\cs, 0.1) {SA};

% ── SoRL token boxes ─────────────────────────────────────────────────────
\foreach \i/\t in {0/t2,1/t2,2/t6,3/t2,4/t1,5/t16,6/t3}{
  \node[token] at (\i*\cs, -0.65) {\texttt{\t}};
}

% ── carry arrows (cascade flows right→left: d5→d4→d3→d2→d1→d0) ──────────
\foreach \fr/\to in {5/4, 4/3, 3/2, 2/1, 1/0}{
  \draw[carry] (a\fr.west) -- (a\to.east);
}

% ── cascade bracket + label ──────────────────────────────────────────────
\draw[orange!60!red, thin]
  (a5.north west) -- ++(0, 0.22)
  -- (a0.north east) -- ++(0,-0.22);
\node[font=\scriptsize\itshape, text=orange!60!red]
  at (2.5*\cs, 1.35) {carry cascade};

% ── legend ───────────────────────────────────────────────────────────────
\matrix[
  matrix of nodes,
  nodes={font=\scriptsize, inner sep=2pt, anchor=west},
  row sep=2pt, column sep=6pt,
  anchor=north west,
] at (0, -1.05) {
  \node[digit=cSA, minimum width=0.45cm, minimum height=0.3cm,
        font=\scriptsize] {}; &
  \node {SA — simple add}; &
  \node[digit=cSC, minimum width=0.45cm, minimum height=0.3cm,
        font=\scriptsize] {}; &
  \node {SC — generates carry}; &
  \node[digit=cUC, minimum width=0.45cm, minimum height=0.3cm,
        font=\scriptsize] {}; &
  \node {UC — uses carry}; &
  \node[digit=cUS, minimum width=0.45cm, minimum height=0.3cm,
        font=\scriptsize] {}; &
  \node {US — cascade}; &
  \node[token, minimum width=0.45cm, minimum height=0.3cm,
        font=\scriptsize] {}; &
  \node[text=violet!70!black] {DLR code}; \\
};

\end{tikzpicture}
\caption{%
  Six-digit addition $959{,}271 + 040{,}756 = 1{,}000{,}027$, a four-deep
  carry cascade. At each answer-digit position DLR assigns one
  abstraction code (bottom row). Codes \texttt{t2} and \texttt{t6}
  cluster on cascade positions (UC/US); \texttt{t16} marks the carry
  source (SC); \texttt{t3} marks the trivial position (SA).
  Code assignments from model \texttt{add\_sub\_sorl\_v1\_abs30\_K1\_100K}
  (K=1, 30-code codebook).%
}
\label{fig:arithmetic-example}
\end{figure}
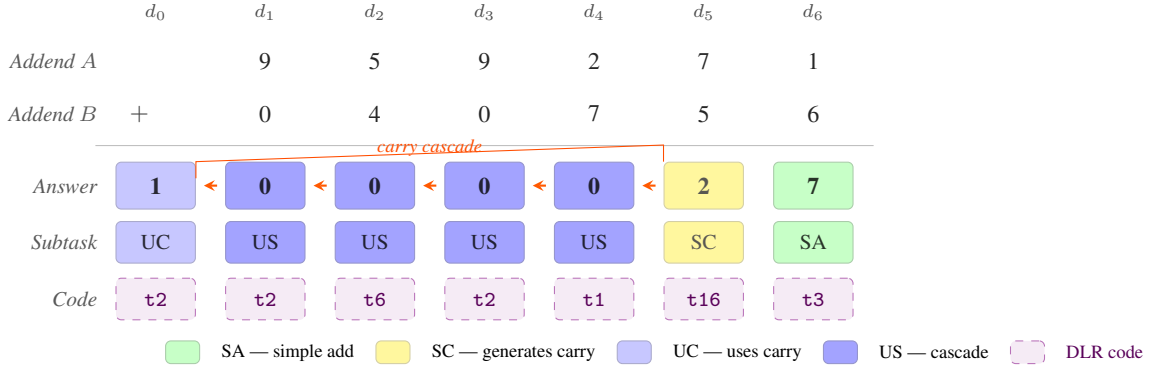

\paragraph{Setup.}
The main interpretability analyses use a 2-layer, 1-head, 128-dimensional transformer trained on 100K examples, evaluated on 2{,}600 held-out problems across 26 splits. This model achieves 95.5\% accuracy with DLR abstraction codes and 0.1\% without; making it the clearest test-bed for causal analysis. All results are reproducible from the released code. We report top level performance on other variants to show the DLR performance gains here are general.

\subsection{Performance: DLR vs.\ \sft{} on undersized architectures}
\label{app:performance}

Full performance results across all architectures and data sizes are in Table~\ref{tab:undersized-wins}. In general, DLR wins consistently, with margins increasing on harder splits of the data.

\begin{table}[ht]
  \centering\small
  \begin{tabular}{llrrrr}
    \toprule
    Architecture & Data & Baseline & DLR & Gap & C6 gap \\
    \midrule
    \texttt{1L/2H/256d} & 10K  & 10\% & \textbf{19\%} & \textcolor{green!50!black}{\textbf{+9\%}}  & \textcolor{green!50!black}{\textbf{+18\%}} \\
                        & 25K  & 32\% & 26\%          & $-7\%$                                    & \textcolor{green!50!black}{\textbf{+10\%}} \\
                        & 50K  & 44\% & \textbf{65\%} & \textcolor{green!50!black}{\textbf{+21\%}} & \textcolor{green!50!black}{\textbf{+34\%}} \\
                        & 100K & 49\% & \textbf{65\%} & \textcolor{green!50!black}{\textbf{+16\%}} & \textcolor{green!50!black}{\textbf{+31\%}} \\
    \midrule
    \texttt{1L/3H/510d} & 10K  & 36\% & \textbf{52\%} & \textcolor{green!50!black}{\textbf{+16\%}} & \textcolor{green!50!black}{\textbf{+30\%}} \\
                        & 25K  & 46\% & \textbf{60\%} & \textcolor{green!50!black}{\textbf{+14\%}} & \textcolor{green!50!black}{\textbf{+22\%}} \\
                        & 50K  & 53\% & \textbf{72\%} & \textcolor{green!50!black}{\textbf{+19\%}} & \textcolor{green!50!black}{\textbf{+38\%}} \\
                        & 100K & 67\% & \textbf{83\%} & \textcolor{green!50!black}{\textbf{+16\%}} & \textcolor{green!50!black}{\textbf{+26\%}} \\
    \midrule
    \texttt{2L/1H/128d} & 10K  & 16\% & \textbf{36\%} & \textcolor{green!50!black}{\textbf{+21\%}} & \textcolor{green!50!black}{\textbf{+39\%}} \\
                        & 25K  & 40\% & \textbf{55\%} & \textcolor{green!50!black}{\textbf{+15\%}} & \textcolor{green!50!black}{\textbf{+23\%}} \\
                        & 50K  & 59\% & \textbf{87\%} & \textcolor{green!50!black}{\textbf{+28\%}} & \textcolor{green!50!black}{\textbf{+50\%}} \\
                        & 75K  & 75\% & \textbf{87\%} & \textcolor{green!50!black}{\textbf{+12\%}} & \textcolor{green!50!black}{\textbf{+5\%}}  \\
                        & 100K & 73\% & \textbf{95\%} & \textcolor{green!50!black}{\textbf{+22\%}} & \textcolor{green!50!black}{\textbf{+33\%}} \\
    \bottomrule
  \end{tabular}
  \caption{DLR ($K{=}1$, $|\mathcal{A}|{=}30$) vs.\ \sft{} baseline on
    undersized architectures across data sizes.
    \textbf{Gap} = overall accuracy gain; \textbf{C6 gap} = gain on
    6-deep carry cascades (the hardest split).
    DLR wins in \textbf{12 of 13} (architecture, data-size) pairs;
    the single exception is \texttt{1L/2H/256d} at 25K, where the model
    is undertrained (accuracy still rising at epoch 20).
    DLR wins on C6 in \textbf{all 13} configurations.}
  \label{tab:undersized-wins}
\end{table}

\subsection{Causal ablations}
\label{app:causal}

To confirm that DLR codes are causally necessary (not merely correlated
with correct outputs), we run three intervention conditions on model
\texttt{2L/1H/128d} (100K), evaluated across 2{,}600 held-out problems:

\begin{itemize}[nosep]
  \item \textbf{Shuffle}: randomly permute all abstraction codes within
    each sequence (code identities preserved; positional assignment destroyed).
  \item \textbf{Random}: replace each code with a draw uniform over the
    30-code codebook (identity and position both destroyed).
  \item \textbf{Knockout}: replace every code with a fixed \texttt{[UNK]}
    embedding (strongest intervention; removes all information).
\end{itemize}

Table~\ref{tab:ablation-splits} shows per-split accuracy under each condition.

\begin{table}[h]
  \centering\small
  \begin{tabular}{llrrrr}
    \toprule
    Family & Split & Baseline & Shuffle & Random & Knockout \\
    \midrule
    \multirow{4}{*}{\textit{Addition (easy)}} & S0 (no carry) & 100\% & 24\% & 28\% & 0\% \\
     & S1 & 100\% & 17\% &  9\% & 0\% \\
     & S2 & 100\% & 22\% & 10\% & 0\% \\
     & random & 100\% & 26\% &  8\% & 0\% \\
    \midrule
    \multirow{4}{*}{\textit{Addition cascade (hard)}} & C3 (3-deep) & 96\% & 28\% & 14\% & 0\% \\
     & C4 (4-deep) & 99\% & 25\% & 13\% & 0\% \\
     & C5 (5-deep) & 99\% & 23\% & 19\% & 0\% \\
     & C6 (6-deep) & 97\% & 27\% & 15\% & 0\% \\
    \midrule
    \multirow{1}{*}{\textit{Subtraction (easy)}} & random & 100\% & 46\% & 12\% & 0\% \\
    \midrule
    \multirow{3}{*}{\textit{Subtraction cascade (hard)}} & M3 (3-deep borrow) & 100\% & 22\% &  1\% & 0\% \\
     & M4 (4-deep borrow) &  85\% &  6\% &  0\% & 0\% \\
     & M5 (5-deep borrow) &  57\% &  3\% &  0\% & 2\% \\
    \midrule
    \multicolumn{2}{l}{\textbf{Overall}} & \textbf{95.5\%} & 26.6\% & 12.3\% & 0.1\% \\
    \bottomrule
  \end{tabular}
  \caption{Per-split causal ablation (\texttt{2L/1H/128d}, 100K training examples).
    \textbf{Shuffle} preserves code identity but destroys positional assignment;
    \textbf{Random} destroys both; \textbf{Knockout} removes all code information.}
  \label{tab:ablation-splits}
\end{table}

\paragraph{Commentary.}
Knockout reduces accuracy to $\leq$2\% on every split, confirming that the model has offloaded computation into the routing codes.
Three patterns are notable:

\begin{itemize}[nosep]
  \item \textbf{Shuffle $>$ Random on easy splits.}
    On addition S0 (no carry), shuffle yields 24\% vs.\ random's 28\%; the gap is small and reflects that no single position is critical - any code in any position is roughly as bad as another.
  \item \textbf{Shuffle $>$ Random on cascade splits (C3-C6).}
    On 4--6-deep carry cascades, shuffle (23--28\%) consistently
    outperforms random (13--19\%).
    When codes are shuffled, a cascade position receives a code
    from another cascade position - a wrong code but from the
    ``right family'', producing a systematic one-off carry error.
    Random codes provide no structural signal at all, making
    cascade resolution impossible.
  \item \textbf{Borrow cascades are uniquely sensitive.}
    Sub-M4 (4-deep borrow) drops from 85\% baseline to 6\% under
    shuffle and 0\% under random - a 79\,pp collapse from shuffle
    alone. Sub-M5 (5-deep borrow) is hardest even at baseline (57\%),
    and all ablations reduce it to $\leq$3\%, showing that deep
    borrow cascades are the single hardest regime and that DLR
    codes are essential for solving them.
\end{itemize}

\begin{tcolorbox}[colback=gray!6, colframe=gray!40,
  fonttitle=\bfseries\small, title={Finding \#2},
  left=5pt, right=5pt, top=4pt, bottom=4pt]
\small
DLR abstraction codes are causally necessary: knockout collapses accuracy from 95.5\% to 0.1\% overall, and to $\leq$3\% on the hardest borrow-cascade splits (M4-M5).
Shuffle (identity-preserving, position-destroying) is more harmful than random on cascade splits - wrong-position codes from the same structural family cause systematic carry errors, while random codes cause broader incoherence.
\end{tcolorbox}

% ─────────────────────────────────────────────────────────────────────────────
\subsection{Code-subtask heatmap}
\label{app:heatmap}

Of the 30 codes in the codebook, 23 appear in the held-out evaluation set.
Each active code concentrates on a narrow slice of the subtask space: the dominant subtask accounts for ${\geq}70\%$ of that code's occurrences in the majority of cases.
Codes are also \emph{position-locked}: each code appears predominantly at one or two answer positions ($d_0$-$d_6$), rarely crossing position boundaries.
Representative examples are shown in Table~\ref{tab:code-profiles}:
code \texttt{t21} fires 93\% of the time on US (sum-9 cascade, addition)
with digit sum $\equiv 9 \pmod{10}$ in 95\% of cases;
code \texttt{t23} is the subtraction mirror (UD, 88\%, position $d_3$). See Figure \ref{fig:code-subtask}.

\begin{figure}[H]
  \centering
  \includegraphics[width=0.95\linewidth]{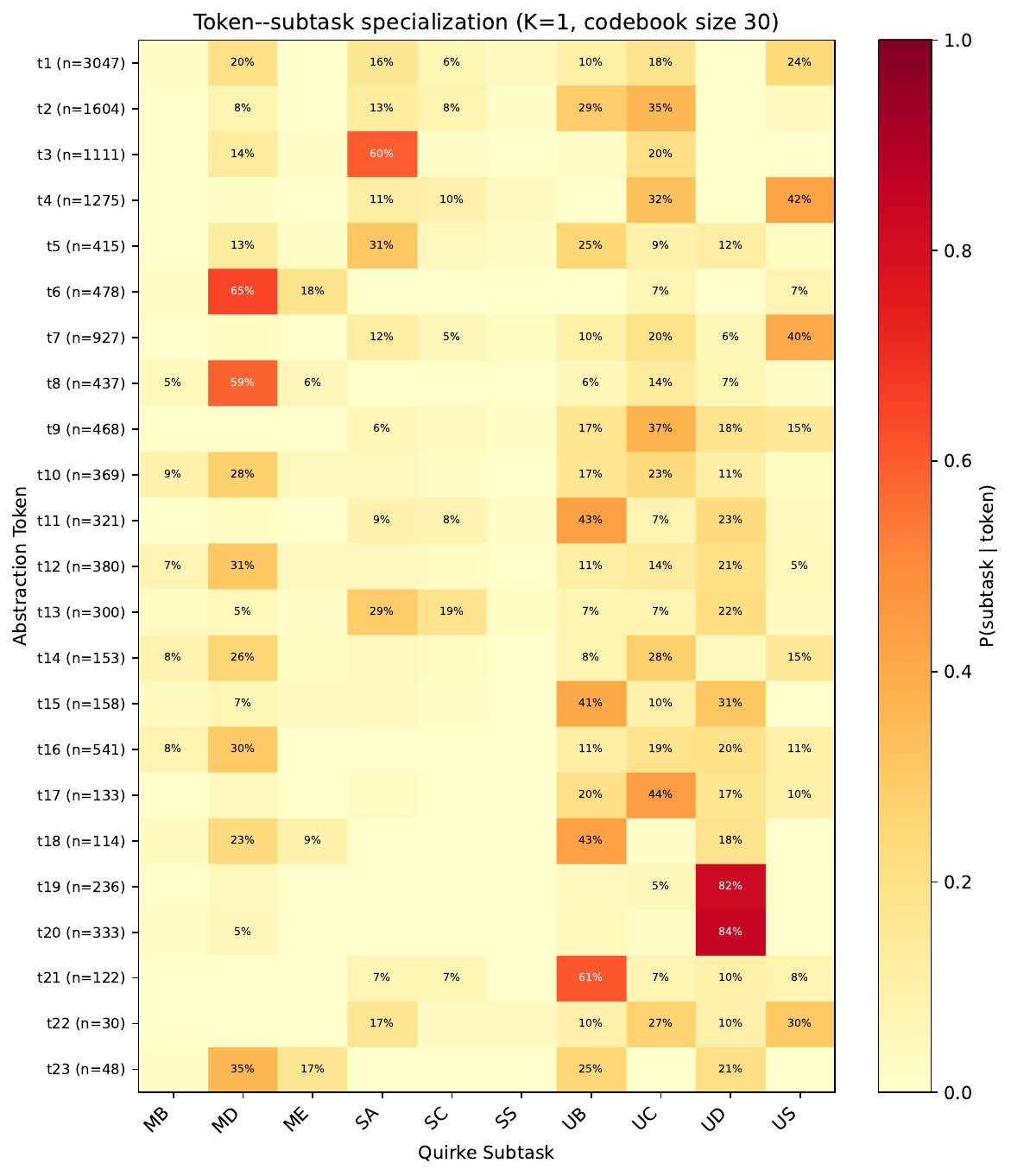}
  \caption{Code-subtask heatmap for \texttt{2L/1H/128d} (100K).
    Each cell shows $P(\text{subtask} \mid \text{code})$ - the fraction of that code's
    occurrences that fall into each Quirke subtask - over 2{,}600 held-out examples across all addition and subtraction splits.
    Rows are the 23 active codes (of 30 in the codebook), sorted by dominant subtask;
    columns are the 10 subtask labels (SA, SC, SS, UC, US for addition; MD, MB, ME, UB, UD for subtraction).
    High values appear in narrow bands: most codes fire predominantly on 1--2 subtasks,
    with the highest-purity codes (e.g.\ \texttt{t19}: UD 82\%, \texttt{t20}: UD 84\%)
    approaching single-subtask specialization.
    Codes at the bottom of the heatmap (e.g.\ \texttt{t1}, \texttt{t2}) are more
    polysemantic, spreading probability mass across multiple subtasks.}
  \label{fig:code-subtask}
\end{figure}

\begin{table}[ht]
  \centering\small
  \begin{tabular}{llrrll}
    \toprule
    Code & Pos & $n$ & Top subtask & Purity & Op \\
    \midrule
    \texttt{t21} & $d_3$ & 719 & US & 93\% & add (94\%) \\
    \texttt{t23} & $d_3$ & 687 & UD & 88\% & sub (93\%) \\
    \texttt{t6}  & $d_2$ & 1438 & US & 52\% & add (79\%), sum${\equiv}9$: 78\% \\
    \texttt{t7}  & $d_2$ & 1026 & UB/UD & 90\% & sub (100\%) \\
    \texttt{t14} & $d_4$ & 1242 & UD & 73\% & sub (81\%) \\
    \bottomrule
  \end{tabular}
  \caption{Selected code profiles. Each code specialises in a specific
    subtask at a specific answer position. ``Purity'' = fraction of
    occurrences where the top subtask applies.}
  \label{tab:code-profiles}
\end{table}

\begin{tcolorbox}[colback=gray!6, colframe=gray!40,
  fonttitle=\bfseries\small, title={Finding \#3},
  left=5pt, right=5pt, top=4pt, bottom=4pt]
\small
DLR spontaneously learns position-locked, subtask-specialised routing:
23 of 30 codebook codes are active; each concentrates on 1-2 of the 10 Quirke subtasks (purity ${\geq}70\%$ for most), and each is tied to one or two answer positions. The codebook partitions the arithmetic computation into an interpretable registry of specialist codes.
\end{tcolorbox}

% ─────────────────────────────────────────────────────────────────────────────
\subsection{Guided computation via code intervention}
\label{app:guided}

If DLR codes encode the \emph{computational route} rather than just the answer, swapping a code at a mispredicted position should \emph{fix} the prediction - without retraining, without accessing internal activations, and without modifying the weights.

We test this with \emph{surgical swap}: for each wrong prediction, we try replacing the abstraction code at each answer position with every other code in the codebook (29 candidates $\times$ 5 positions = 145 interventions per example) and measure how many wrong predictions become correct - and how many previously-correct predictions break.

\paragraph{Results.}
At positions $d_0$-$d_2$ (the carry-heavy positions), a fixing swap exists for 27-31\% of mispredicted examples.
The best single swap is replacing \texttt{t16} with \texttt{t25} at $d_1$: this fixes 10 wrong predictions while breaking only 5 correct ones (a 2:1 fix-to-break ratio), all on carry cascade splits (C4-C6).
Position $d_3$ and $d_4$ are harder to fix surgically (fix rates 8\% and 2\%), consistent with those positions encoding longer-range carry state that a single-position swap cannot resolve.

\begin{tcolorbox}[colback=gray!6, colframe=gray!40,
  fonttitle=\bfseries\small, title={Finding \#4},
  left=5pt, right=5pt, top=4pt, bottom=4pt]
\small
Code interventions enable \emph{guided computation}: replacing a single abstraction code in the sequence fixes a wrong prediction in 27-31\% of mispredicted examples at carry-heavy positions, with no weight updates and no access to internal activations.
This is only possible because the codes are a human-readable interface to the model's routing decisions.
\end{tcolorbox}

% ─────────────────────────────────────────────────────────────────────────────
\subsection{DLR codes recover circuits from Quirke et al.}
\label{app:quirke-analogy}

\citet{quirke_2024_addsub_preprint} identify, via PCA, ablation and activation patching of internal residual-stream activations, a \emph{tri-state carry classifier} at each digit position $n$: the hidden state encodes which of three carry regimes applies -
$\text{ST}_n = 0$ (digit sum $< 9$; carry cannot propagate),
$\text{ST}_n = 1$ (digit sum $> 9$; carry will propagate), or
$\text{ST}_n = U$ (digit sum $= 9$; carry state \emph{uncertain}, depends
on lower positions).
This trichotomy is the core circuit for addition; borrow cascades have an analogous structure for subtraction.

DLR recovers the same trichotomy \emph{without access to activation
data or ground-truth circuit labels}, purely from the info-gain training
signal.
The correspondence is visible directly in the codebook:

\begin{itemize}[nosep]
  \item $\text{ST}_n = U$ (sum-9 uncertain, addition) $\longleftrightarrow$
    \texttt{t21} at $d_3$ (US=93\%, sum${\equiv}9$: 95\%);
    \texttt{t6} at $d_2$ (US=52\%, sum${\equiv}9$: 78\%).
  \item $\text{ST}_n = U$ (borrow-uncertain, subtraction) $\longleftrightarrow$
    \texttt{t23} at $d_3$ (UD=88\%, sub=93\%);
    \texttt{t7} at $d_2$ (UB/UD=90\%, sub=100\%).
  \item $\text{ST}_n = 1$ (carry generated) $\longleftrightarrow$
    codes concentrated on SC/MB subtasks at each position.
  \item $\text{ST}_n = 0$ (simple digit, no carry) $\longleftrightarrow$
    codes concentrated on SA/MD subtasks.
\end{itemize}

Crucially, Quirke et al.\ required full activation-level mechanistic interpretability to discover this classifier; DLR surfaces it as a readable code in the output sequence, accessible without any post-hoc analysis.

\begin{tcolorbox}[colback=gray!6, colframe=gray!40,
  fonttitle=\bfseries\small, title={Finding \#5},
  left=5pt, right=5pt, top=4pt, bottom=4pt]
\small
DLR independently rediscovers the carry-state tri-classifier
($\text{ST}_n \in \{0, U, 1\}$) identified by \citet{quirke_2024_addsub_preprint}
via internal circuit analysis — with no access to ground-truth circuit labels.
The three carry regimes map onto disjoint code clusters (e.g.\
\texttt{t21}/\texttt{t6} for sum-9 uncertain in addition;
\texttt{t23}/\texttt{t7} for borrow-uncertain in subtraction),
and an analogous structure appears for subtraction borrow cascades.
What Quirke et al.\ needed PCA of hidden activations to reveal,
DLR externalises as a readable routing code.
\end{tcolorbox}

% ─────────────────────────────────────────────────────────────────────────────
\subsection{Polysemantic codes}
\label{app:polysemantic}

Not all codebook codes are specialists. Table~\ref{tab:code-polysemanticity}
contrasts the most specialist with the most polysemantic codes.
Code \texttt{t21} fires 94\% of the time on a single subtask (US) at a single position ($d_3$, addition only) - a true specialist.
Code \texttt{t1}, by contrast, is the highest-frequency code
($n{=}6{,}359$, spanning all five answer positions) with no subtask exceeding 24\% - it acts as a general-purpose fallback, handling whichever position and carry regime was not captured by a specialist code.

\begin{table}[H]
  \centering\small
  \begin{tabular}{clrrr}
    \toprule
    Code & Top subtask & Purity & $n$ & Positions \\
    \midrule
    \multicolumn{5}{l}{\textit{Specialist (high purity)}} \\
    \texttt{t21} & US (cascade, add) & 94\% & 719  & 1 \\
    \texttt{t23} & UD (cascade, sub) & 88\% & 687  & 3 \\
    \texttt{t14} & UD               & 74\% & 1242 & 3 \\
    \midrule
    \multicolumn{5}{l}{\textit{Polysemantic (low purity)}} \\
    \texttt{t5}  & UB               & 21\% & 1377 & 4 \\
    \texttt{t1}  & MD               & 24\% & 6359 & 5 \\
    \texttt{t20} & UB               & 24\% &  283 & 3 \\
    \bottomrule
  \end{tabular}
  \caption{Specialist vs.\ polysemantic codes.
    \textbf{Purity} = fraction of occurrences where the top subtask applies.
    \textbf{Positions} = number of distinct answer positions ($d_0$-$d_6$)
    where the code appears. \texttt{t1} is a high-frequency fallback used
    across all positions; \texttt{t21} is a pure sum-9 cascade detector.}
  \label{tab:code-polysemanticity}
\end{table}

Polysemanticity here mirrors the phenomenon described in neural network interpretability~\citep{elhage2022superposition}: a single code encodes multiple distinct roles, likely because the 30-code codebook has spare capacity at the overflow position ($d_0$) where carry state is most
variable. The specialist codes concentrate at mid-sequence positions ($d_2$-$d_4$) where carry propagation is most structured.

\begin{tcolorbox}[colback=gray!6, colframe=gray!40,
  fonttitle=\bfseries\small, title={Finding \#6},
  left=5pt, right=5pt, top=4pt, bottom=4pt]
\small
The 30-code codebook is not uniformly specialist: high-purity codes (e.g.\ \texttt{t21}, 94\% US, single position) coexist with polysemantic fallback codes (e.g.\ \texttt{t1}, top purity 24\%, all five positions).
Polysemanticity concentrates at overflow positions where carry state is most variable; specialist codes dominate the structured mid-sequence carry-propagation positions.
\end{tcolorbox}

\subsection{Automated code interpretation}
\label{app:autointerp}

We implement a light version of the automated interpretation procedure of \citet{bills2023language_models_explain_neurons}.
For each active code, we collect the $N{=}10$ examples from the evaluation set where the model assigned it with highest softmax confidence,
then ask \texttt{claude-haiku} to produce a one-sentence role description.
Table~\ref{tab:auto-interp} shows results for the 8 highest-confidence codes.

\begin{table}[H]
  \centering\small
  \begin{tabular}{clrp{5.5cm}}
    \toprule
    Code & Top subtask & Conf. & Auto-interpretation \\
    \midrule
    \texttt{t0} & UC (47\%) & 1.00 & Code t0 marks the tens digit position in addition problems, regardless of carry state or sum value. \\
    \texttt{t2} & UC (70\%) & 0.99 & Code t2 outputs the ones digit (0) when adding two numbers whose ones digits sum to 10 or more. \\
    \texttt{t1} & UC (30\%) & 0.99 & This code routes to the fourth digit position during addition when a carry from the previous position must be incorporated. \\
    \texttt{t3} & UC (44\%) & 0.94 & Code t3 routes to the hundreds position (d3) when processing carries from the tens column in addition. \\
    \texttt{t5} & MD (65\%) & 0.93 & Code t5 routes cases where the ones digit result is 0, spanning multiple subtasks and operations. \\
    \texttt{t8} & MD (26\%) & 0.91 & Code t8 activates when processing the tens digit (d2) across addition/subtraction with various carry states. \\
    \texttt{t10} & UB (41\%) & 0.88 & Code t10 routes subtraction problems requiring borrow propagation at mid-to-late digit positions. \\
    \texttt{t6} & UC (27\%) & 0.88 & Code t6 routes cases where the ones digit result is 0, regardless of operation or carry state. \\
    \bottomrule
  \end{tabular}
  \caption{Automated interpretation of the 8 highest-confidence DLR abstraction codes
    (\`{a} la \citep{bills2023language_models_explain_neurons}).
    For each code, the 10 examples with highest softmax confidence are shown to an LLM,
    which produces a one-sentence role description.
    \textbf{Conf.} = mean softmax probability of the assigned code.
    High-confidence specialists receive crisp, position- and operation-specific descriptions;
    polysemantic codes (not shown) produce broader descriptions.}
  \label{tab:auto-interp}
\end{table}

\begin{tcolorbox}[colback=gray!6, colframe=gray!40,
  fonttitle=\bfseries\small, title={Finding \#7: Automated interpretation matches ground-truth subtask labels},
  left=5pt, right=5pt, top=4pt, bottom=4pt]
\small High-confidence codes get crisp role descriptions; polysemantic codes get appropriately vague ones — without accessing ground-truth labels.
\end{tcolorbox}

\subsection{Training and evaluation details}
\label{app:training}

We use the following architectures for training our DLR addition models, see Table \ref{tab:undersized_specs}.
\begin{table}[h]
  \centering\small
  \begin{tabular}{lrrrrr}
    \toprule
    Architecture & Layers & Heads & Hidden & FFN & Params \\
    \midrule
    \texttt{1L/2H/256d} & 1 & 2 & 256 & 1024 & ${\sim}$0.3M \\
    \texttt{1L/3H/510d} & 1 & 3 & 510 & 2040 & ${\sim}$2.0M \\
    \texttt{2L/1H/128d} & 2 & 1 & 128 &  512 & ${\sim}$0.1M \\
    \bottomrule
  \end{tabular}
  \caption{Undersized architectures 
    All use pre-norm, GeLU, Qwen3 tokenizer.}
    \label{tab:undersized_specs}
\end{table}

\begin{table}[h]
  \centering\small
  \begin{tabular}{ll}
    \toprule
    Hyperparameter & Value \\
    \midrule
    Optimizer              & AdamW \\
    Learning rate          & $8\times10^{-5}$ \\
    $(\beta_1,\,\beta_2)$  & $(0.9,\;0.999)$ \\
    Weight decay           & $0.01$ \\
    LR schedule            & Linear warmup (3\%) then constant \\
    Batch size             & 64 \\
    Epochs                 & 20 \\
    DLR codebook       & $|\mathcal{A}|{=}30$, $K{=}1$ \\
    $\alpha_{\text{info-gain}},\,\alpha_{\text{abs}},\,\alpha_{\text{zipf}}$
                           & $10.0,\;0.1,\;1.0$ \\
    \bottomrule
  \end{tabular}
  \caption{Shared training hyperparameters.}
\end{table}

Evaluation uses fixed-length autoregressive decoding (no teacher forcing):
the model generates answer digits $d_0 \to d_6$ using its own predictions, with abstraction codes inserted via the DLR search-then-recurse procedure (matching training). Accuracy is measured on 100 held-out examples per split.

\clearpage
\section*{Limitations}
\label{sec:limitations}
On the theoretical side, the GDS optimality theorem assumes a finite
Dynamic MDP with a time-varying reward function determined by the external
environment. Under DLR, however, the reward function $\log p_{\theta}(x\mid a)$
is an ``internal'' reward function that depends on the model parameters
$\theta$, which makes GDS not directly applicable. While DLR borrows the
\textsc{search}--\textsc{select}--\textsc{update} principle from GDS, our
theorem motivates rather than certifies the empirical method.
On the empirical side, our study is limited to LLM post-training in a
low-data regime. The benefit of search may be under-realised at our compute
scale; a continued-pretraining setting at larger data and sequence-length
budgets is likely where DLR's search component pays off most, and we leave
this to future work. While our $6\times 4 = 24$ (model, dataset)
configurations span two open-weight families (Llama-3.2, Qwen3) at
$0.6$B--$8$B parameters, this is still a narrow slice of the model and task
space.

% Flush all pending [h] floats (ablation & hyperparameter-sweep tables)
% before the checklist, so they cannot bleed into the checklist pages.
\clearpage
\section*{NeurIPS Paper Checklist}

\begin{enumerate}

\item {\bf Claims}
    \item[] Question: Do the main claims made in the abstract and introduction accurately reflect the paper's contributions and scope?
    \item[] Answer: \answerYes{}
    \item[] Justification: The abstract and introduction state four claims that map one-to-one to the paper's contents:
    (i) the General Dijkstra Search (GDS) algorithm provably recovers an optimal goal-reaching policy by concatenating sub-policies optimal for intermediate goals (Theorem in \S\ref{sec:gds}, full proofs in App.~\ref{app:dmdp-theory}--\ref{app:policy-composition});
    (ii) Dynamic Latent Routing (DLR) instantiates the ``search--select--update'' principle as a single-stage post-training method that jointly updates the codebook, policy head, and language model (\S\ref{sec:method});
    (iii) in the low-data fine-tuning regime, DLR matches or outperforms supervised fine-tuning across four QA datasets and six backbones, with a mean gain of $+6.6$\,pp, while prior discrete-latent baselines (PauseToken, TokenAssorted, AbstractCoT) lag behind SFT (Table~\ref{tab:main}, \S\ref{sec:experiments});
    (iv) mechanistic analyses verify that the learned routing is structured and causally load-bearing -- codebook diversity, topic-pure $n$-grams, and per-code causal effects (\S\ref{sec:analysis}, App.~\ref{app:case-studies},~\ref{app:arithmetic}).
    Scope is explicitly bounded to small-data fine-tuning of $\leq$8B-parameter open-weight LMs.

\item {\bf Limitations}
    \item[] Question: Does the paper discuss the limitations of the work performed by the authors?
    \item[] Answer: \answerYes{}
    \item[] Justification: A dedicated \emph{Limitations} paragraph at the end of the Conclusion (\S\ref{sec:limitations}) discusses both theoretical and empirical limitations:
    (i) the GDS optimality theorem assumes a finite Dynamic MDP with an externally determined, time-varying reward, whereas DLR's reward $\log p_{\theta}(x\mid a)$ is parameter-dependent, so the theorem motivates but does not certify the empirical method;
    (ii) experiments are restricted to a low-data post-training regime where the benefit of per-step search is likely under-realised relative to a continued-pretraining setting; and
    (iii) the $6\times4=24$ (model, dataset) grid covers only two open-weight families (Llama-3.2, Qwen3) at $0.6$B--$8$B parameters, which is a narrow slice of the model and task space.

\item {\bf Theory assumptions and proofs}
    \item[] Question: For each theoretical result, does the paper provide the full set of assumptions and a complete (and correct) proof?
    \item[] Answer: \answerYes{}
    \item[] Justification: All theoretical results are stated formally with explicit assumptions and accompanied by complete proofs.
    The Dynamic MDP setting (finite state/action spaces, time-varying rewards, finite horizon) is defined in App.~\ref{app:dmdp-theory}, which also contains the existence theorem for the optimal policy and the value-concatenation theorem with full proofs.
    The symmetric goal-covering result and its proof are in App.~\ref{app:gds-coverage}.
    The General Dijkstra Search (GDS) algorithm and its optimality theorem are stated in \S\ref{sec:gds}; the supporting lemmas (policy dominance, dominating/dominated goal sets, queue and pruning invariants -- Lemmas~1--8) and the full optimality proof appear in App.~\ref{app:policy-composition}.
    All theorems and lemmas are numbered and cross-referenced from the main text; assumptions used by each result are restated in its statement.

    \item {\bf Experimental result reproducibility}
    \item[] Question: Does the paper fully disclose all the information needed to reproduce the main experimental results of the paper to the extent that it affects the main claims and/or conclusions of the paper (regardless of whether the code and data are provided or not)?
    \item[] Answer: \answerYes{}
    \item[] Justification: The information needed to reproduce the main results is fully disclosed.
    The DLR objective and per-step \textsc{search}--\textsc{select}--\textsc{update} procedure are specified in \S\ref{sec:method} and Algorithm~\ref{alg:sorl-step}, including the four loss components (Eq.~\eqref{eq:total-loss}) and the stop-gradient structure on the policy head.
    Hyperparameters are reported in full: shared training settings (optimizer, learning rate, schedule, batch size, epochs, codebook size $C$, abstraction ratio $K$, loss weights $\alpha_{\mathrm{policy}}, \alpha_{\mathrm{reg}}$; the information-gain term enters at unit weight) appear in App.~\ref{app:arithmetic} (Table~\ref{tab:undersized_specs} and the shared-training-hyperparameters table); per-row choices (steering layer, codebook size, $\alpha_{\mathrm{policy}}$) are listed alongside each result, with full sweeps in App.~\ref{app:hp-sweeps}.
    The four QA datasets (CommonsenseQA, GSM8K, ScienceQA, StrategyQA) are public and used at their official splits with sizes given in Table~\ref{tab:main}.
    All six backbone models are open-weight (Qwen3-0.6B/1.7B/4B/8B, Llama-3.2-1B/3B) and cited.
    Evaluation uses fixed-length autoregressive decoding with bootstrap-resampled $\pm$pp error bars (\S\ref{sec:experiments}).

\item {\bf Open access to data and code}
    \item[] Question: Does the paper provide open access to the data and code, with sufficient instructions to faithfully reproduce the main experimental results, as described in supplemental material?
    \item[] Answer: \answerNo{}
    \item[] Justification: Code is not released with the submission to preserve anonymity, but will be made publicly available (with training and evaluation scripts, hyperparameter configs, and instructions for reproducing each table) upon publication.
    All datasets used (CommonsenseQA, GSM8K, ScienceQA, StrategyQA) are public and accessed at their official splits, and all six backbone models (Qwen3-0.6B/1.7B/4B/8B, Llama-3.2-1B/3B) are open-weight checkpoints.
    The paper itself provides full algorithmic detail (\S\ref{sec:method}, Algorithm~\ref{alg:sorl-step}, Eq.~\eqref{eq:total-loss}) and complete hyperparameter listings (App.~\ref{app:hp-sweeps},~\ref{app:arithmetic}) sufficient to re-implement DLR from scratch.

\item {\bf Experimental setting/details}
    \item[] Question: Does the paper specify all the training and test details (e.g., data splits, hyperparameters, how they were chosen, type of optimizer) necessary to understand the results?
    \item[] Answer: \answerYes{}
    \item[] Justification: All training and test details are specified.
    Optimizer (AdamW), learning rate ($8\times 10^{-5}$), $(\beta_1,\beta_2)$, weight decay, LR schedule (3\% linear warmup then constant), batch size, number of epochs, codebook size $C$, abstraction ratio $K$, and loss weights $\alpha_{\mathrm{policy}}, \alpha_{\mathrm{reg}}$ (the information-gain term enters at unit weight) are listed in the shared-training-hyperparameters table in App.~\ref{app:arithmetic}.
    Per-row choices that vary by model (steering layer, $\alpha_{\mathrm{policy}}$) are reported alongside each result; how these were chosen is documented by the full layer sweep and hyperparameter sweep in App.~\ref{app:hp-sweeps}.
    Data splits and test sizes are given in Table~\ref{tab:main} (CSQA: 1221, GSM8K: 1319, ScienceQA: 2224, StrategyQA: 687) and use the official splits of each benchmark.
    Evaluation uses fixed-length autoregressive decoding without teacher forcing (\S\ref{sec:experiments}).

\item {\bf Experiment statistical significance}
    \item[] Question: Does the paper report error bars suitably and correctly defined or other appropriate information about the statistical significance of the experiments?
    \item[] Answer: \answerYes{}
    \item[] Justification: All accuracy numbers in the main results, ablation, and sweep tables are reported with $\pm$pp error bars (e.g., $49.1_{\pm 2.7}$). The error bars are obtained by bootstrap resampling over the held-out test set, as stated in \S\ref{sec:experiments} (``$\pm$pp from bootstrap resampling on the test sets''); the factor of variability captured is finite-test-set sampling rather than training-time randomness, and the protocol is consistent across all tables. Where deltas are shown (\texttt{\textbackslash dlt} entries in the ablation tables), they are point differences against the non-ablated DLR baseline in the same row of Table~\ref{tab:main}, so the reader can compare the magnitude of each delta against the bootstrap error bars on the corresponding cells.

\item {\bf Experiments compute resources}
    \item[] Question: For each experiment, does the paper provide sufficient information on the computer resources (type of compute workers, memory, time of execution) needed to reproduce the experiments?
    \item[] Answer: \answerYes{}
    \item[] Justification: The compute setup is stated in \S\ref{sec:experiments}: all DLR experiments are run on $2{\times}$H100 GPUs, with abstraction ratio $K{=}4$ and $N{=}4$ rollouts per step. Combined with the optimizer, batch size, epoch, and model-size hyperparameters listed in App.~\ref{app:arithmetic}, this is sufficient to estimate the wall-clock cost of any individual run.
    
\item {\bf Code of ethics}
    \item[] Question: Does the research conducted in the paper conform, in every respect, with the NeurIPS Code of Ethics \url{https://neurips.cc/public/EthicsGuidelines}?
    \item[] Answer: \answerYes{}
    \item[] Justification: The authors have reviewed the NeurIPS Code of Ethics, and the research conforms with its guidelines. The work uses only public benchmark datasets and open-weight model checkpoints, involves no human-subject data, no personally identifiable information, and no scraped or sensitive data; anonymity has been preserved in the submission.

\item {\bf Broader impacts}
    \item[] Question: Does the paper discuss both potential positive societal impacts and negative societal impacts of the work performed?
    \item[] Answer: \answerNA{}
    \item[] Justification: This is foundational research on a fine-tuning method for language models, with no direct path to a specific deployment or application. Any societal impact is mediated through downstream uses of the underlying open-weight LMs that DLR fine-tunes, and is not specific to this work; we therefore do not discuss broader impacts beyond the inherited considerations of the base models.
    
\item {\bf Safeguards}
    \item[] Question: Does the paper describe safeguards that have been put in place for responsible release of data or models that have a high risk for misuse (e.g., pre-trained language models, image generators, or scraped datasets)?
    \item[] Answer: \answerNA{}
    \item[] Justification: The paper does not release new datasets, scraped data, or new pre-trained models. DLR is a fine-tuning method applied to existing open-weight LMs (Qwen3, Llama-3.2) on public QA benchmarks; the resulting fine-tuned checkpoints inherit the safety profile of their base models and pose no incremental misuse risk beyond that already present in the open-weight ecosystem.

\item {\bf Licenses for existing assets}
    \item[] Question: Are the creators or original owners of assets (e.g., code, data, models), used in the paper, properly credited and are the license and terms of use explicitly mentioned and properly respected?
    \item[] Answer: \answerYes{}
    \item[] Justification: All external assets are credited via citation in the main text and references.
    Datasets: CommonsenseQA, GSM8K, ScienceQA, and StrategyQA are public benchmarks cited at their original papers and used at their official splits under their respective standard research-use licenses.
    Backbone models: Qwen3-0.6B/1.7B/4B/8B and Llama-3.2-1B/3B are open-weight checkpoints used in accordance with the Qwen and Llama community licenses, respectively, and cited at their model cards.
    Baselines (PauseToken, TokenAssorted, Abstract-CoT) are reproduced according to the descriptions in their original papers, which are cited.

\item {\bf New assets}
    \item[] Question: Are new assets introduced in the paper well documented and is the documentation provided alongside the assets?
    \item[] Answer: \answerNA{}
    \item[] Justification: No new assets (datasets, model checkpoints, or other artefacts) are released with the submission. Code implementing the DLR training and evaluation pipeline will be released upon publication, with documentation (training/eval scripts, hyperparameter configs, instructions for reproducing each table).

\item {\bf Crowdsourcing and research with human subjects}
    \item[] Question: For crowdsourcing experiments and research with human subjects, does the paper include the full text of instructions given to participants and screenshots, if applicable, as well as details about compensation (if any)? 
    \item[] Answer: \answerNA{}
    \item[] Justification: The paper does not involve crowdsourcing or any research with human subjects. All experiments use existing public benchmarks and open-weight models.

\item {\bf Institutional review board (IRB) approvals or equivalent for research with human subjects}
    \item[] Question: Does the paper describe potential risks incurred by study participants, whether such risks were disclosed to the subjects, and whether Institutional Review Board (IRB) approvals (or an equivalent approval/review based on the requirements of your country or institution) were obtained?
    \item[] Answer: \answerNA{}
    \item[] Justification: The paper does not involve research with human subjects, so IRB review (or equivalent) is not applicable.

\item {\bf Declaration of LLM usage}
    \item[] Question: Does the paper describe the usage of LLMs if it is an important, original, or non-standard component of the core methods in this research? Note that if the LLM is used only for writing, editing, or formatting purposes and does \emph{not} impact the core methodology, scientific rigor, or originality of the research, declaration is not required.
    %this research? 
    \item[] Answer: \answerNA{}
    \item[] Justification: LLMs are not an important, original, or non-standard component of the core methodology. The DLR method itself does not call any external LLM. For completeness, App.~\ref{app:arithmetic} (\emph{Automated code interpretation}) uses an off-the-shelf model (\texttt{claude-haiku}) as a standard labeller to produce one-sentence descriptions of learned codes, mirroring established SAE auto-interpretation pipelines~\citep{paulo_2025_autointerp}; this is an analysis convenience, not a part of DLR's core method, and the paper's claims do not depend on it.

\end{enumerate}

\end{document}